\documentclass[11pt,english]{article}
\usepackage[utf8]{inputenc}
\usepackage{amsmath}
\usepackage{amsthm}
\usepackage{amssymb}
\usepackage{amsfonts}
\usepackage{comment}
\usepackage{mathtools}
\usepackage{algorithm}
\usepackage{graphicx}
\usepackage[font=footnotesize]{caption}
\usepackage{xcolor}
\usepackage{enumitem}
\usepackage{nicefrac}
\DeclareMathOperator*{\argmin}{arg\,min}
\usepackage{bbding}

\usepackage{a4wide}
\usepackage{chngpage}
\usepackage{authblk}

\newtheorem{corollary}{Corollary}
\newtheorem{fact}{Fact}

\usepackage[T1]{fontenc}
\usepackage{babel}
\usepackage{threeparttable}
\usepackage{booktabs}
\usepackage{etoolbox}
\appto\TPTnoteSettings{\footnotesize}
\usepackage[font=footnotesize]{caption}
\usepackage{pifont}

\usepackage{mathtools, nccmath, textcomp}
\usepackage[round]{natbib}

\usepackage[backref=page]{hyperref}
\renewcommand*{\backref}[1]{}
\renewcommand*{\backrefalt}[4]{%
    \ifcase #1 %
    \or        (Cited on page~#2.)%
    \else      (Cited on pages~#2.)%
    \fi}

\usepackage{xcolor} 
\usepackage{amsfonts, amsthm, amssymb} 
\usepackage{mathtools}

\usepackage{subcaption}
\usepackage{makecell}
\usepackage{enumitem}
\usepackage{nicefrac}

\usepackage[normalem]{ulem}
\usepackage[normalem]{ulem}
\usepackage{algorithm}
\usepackage[noend]{algcompatible}
\usepackage{hyperref, tabularx}
\usepackage[capitalise]{cleveref}
\makeatletter
\newcommand{\multiline}[1]{%
  \begin{tabularx}{\dimexpr\linewidth-\ALG@thistlm}[t]{@{}X@{}}
    #1
  \end{tabularx}
}
\makeatother
\algnewcommand{\Initialize}[1]{%
  \State \textbf{Initialize:} #1}
\theoremstyle{plain}
\newtheorem{theorem}{Theorem}
\newtheorem*{theorem*}{Theorem}

\newtheorem{lemma}{Lemma}[section]

\newtheorem*{cor*}{Corollary}
\theoremstyle{definition}
\newtheorem{definition}{Definition}
\newtheorem{assump}{Assumption}
\theoremstyle{remark}
\newtheorem{remark}{Remark}

\newcommand{\mbb}{\mathbb}

\newcommand{\mbe}{\mathbb E}

\newcommand{\lp}{\left(}
\newcommand{\rp}{\right)}

\newcommand{\lcb}{\left\{}
\newcommand{\rcb}{\right\}}
\newcommand{\lbr}{\left[}
\newcommand{\rbr}{\right]}
\newcommand{\lnr}{\left\|}
\newcommand{\rnr}{\right\|}

\newcommand\norm[1]{\lnr#1\rnr}

\newcommand{\bx}{{\mathbf x}}
\newcommand{\bxt}{{\mathbf x^{(t)}}}
\newcommand{\bxk}{{\mathbf x^{(k)}}}
\newcommand{\bxtp}{{\mathbf x^{(t+1)}}}

\newcommand{\be}{{\mathbf e}}
\newcommand{\bet}{{\mathbf e^{(t)}}}
\newcommand{\bek}{{\mathbf e^{(k)}}}

\newcommand{\bz}{{\mathbf z}}
\newcommand{\bzt}{{\mathbf z^{(t)}}}

\newcommand{\bPsi}{{\boldsymbol{\Psi}}}

\newcommand{\bPhi}{{\boldsymbol{\Phi}}}
\newcommand{\bPhit}{{\bPhi^{(t)}}}

\newcommand{\by}{{\mathbf y}}

\newcommand{\G}{\nabla}

\newcommand{\nn}{\nonumber}

\usepackage{xcolor}

\usepackage{tikz}
\usetikzlibrary{calc,trees,positioning,arrows,chains,shapes.geometric,%
	decorations.pathreplacing,decorations.pathmorphing,shapes,%
	matrix,shapes.symbols}

\tikzstyle{startstop} = [rectangle, draw, rounded corners, align=center, minimum width=3cm, minimum height=1cm,text centered]
\tikzstyle{decision} = [diamond, draw, fill=blue!20, 
text width=4.5em, text badly centered, node distance=3cm, inner sep=0pt]
\tikzstyle{block} = [rectangle, draw, fill=blue!10, align=center, rounded corners, minimum width=3cm, minimum height=1cm]
\tikzstyle{blockcast} = [rectangle, draw, fill=red!10, align=center, rounded corners, minimum width=3cm, minimum height=0.45cm]
\tikzstyle{line} = [draw, -latex']
\tikzstyle{cloud} = [draw, ellipse,fill=red!20, node distance=3cm,
minimum height=2em]

\newcommand{\R}{\mathbb{R}}
\newcommand{\E}{\mathbb{E}}
\newcommand{\N}{\mathbb{N}}

\newcommand{\calN}{\mathcal{N}}

\newtheorem{example}{Example}

\makeatletter
\renewcommand\@fnsymbol[1]{}
\makeatother

\usepackage{multirow}
\usepackage{tabularx}

\author[1]{Aleksandar Armacki
}
\author[1]{Shuhua Yu}
\author[1]{Pranay Sharma}
\author[1]{Gauri Joshi}
\author[2]{Dragana Bajovi\'{c}}
\author[3]{Du\v{s}an Jakoveti\'{c}}
\author[1]{Soummya Kar}
\affil[1]{Carnegie Mellon University, Pittsburgh, PA, USA\\ \texttt{\{aarmacki,shuhuay,pranaysh,gaurij,soummyak\}@andrew.cmu.edu }}
\affil[2]{Faculty of Technical Sciences, University of Novi Sad, Novi Sad, Serbia\\ \texttt{dbajovic@uns.ac.rs}}
\affil[3]{Faculty of Sciences, University of Novi Sad, Novi Sad, Serbia\\ \texttt{dusan.jakovetic@dmi.uns.ac.rs}}

\title{Nonlinear Stochastic Gradient Descent and Heavy-tailed Noise: A Unified Framework and High-probability Guarantees}
\date{}

\begin{document}

\maketitle

\begin{abstract}
  We study high-probability convergence in online learning, in the presence of heavy-tailed noise. To combat the heavy tails, a general framework of nonlinear SGD methods is considered, subsuming several popular nonlinearities like sign, quantization, component-wise and joint clipping. In our work the nonlinearity is treated in a black-box manner, allowing us to establish unified guarantees for a broad range of nonlinear methods. For symmetric noise and non-convex costs we establish convergence of gradient norm-squared, at a rate $\widetilde{\mathcal{O}}(t^{-\nicefrac{1}{4}})$, while for the last iterate of strongly convex costs we establish convergence to the population optima, at a rate $\mathcal{O}(t^{-\zeta})$, where $\zeta \in (0,1)$ depends on noise and problem parameters. Further, if the noise is a (biased) mixture of symmetric and non-symmetric components, we show convergence to a neighbourhood of stationarity, whose size depends on the mixture coefficient, nonlinearity and noise. Compared to state-of-the-art, who only consider clipping and require unbiased noise with bounded $p$-th moments, $p \in (1,2]$, we provide guarantees for a broad class of nonlinearities, without any assumptions on noise moments. While the rate exponents in state-of-the-art depend on noise moments and vanish as $p \rightarrow 1$, our exponents are constant and strictly better whenever $p < 6/5$ for non-convex and $p < 8/7$ for strongly convex costs. Experiments validate our theory, showing that clipping is not always the optimal nonlinearity, further underlining the value of a general framework.
\end{abstract}

\section{Introduction}

Stochastic optimization is a well-studied problem, e.g., \citet{robbins1951stochastic,nemirovski2009robust}, where the goal is to minimize an expected cost, without knowing the underlying probability distribution. Formally, the problem is cast as
\begin{equation}\label{eq:problem}
    \argmin_{\bx \in \R^d} \lcb f(\bx) \triangleq \mbe_{\upsilon \sim \Upsilon} [\ell(\bx; \upsilon)] \rcb,
\end{equation} where $\bx \in \R^d$ represents model parameters, $\ell: \R^d \times \mathcal{V} \mapsto \R$ is a loss function, $\upsilon \in \mathcal{V}$ is a random sample distributed according to the unknown distribution $\Upsilon$, while $f: \mbb R^d \mapsto \mbb R$ is commonly known as the \emph{population cost}. Many modern machine learning applications, such as classification and regression, are modeled using \eqref{eq:problem}.

Perhaps the most popular method to solve \eqref{eq:problem} is stochastic gradient descent (SGD) \citet{robbins1951stochastic}, whose popularity stems from low computation cost and incredible empirical success \citet{bottou2010large, hardt2016train}. Convergence guarantees of SGD have been studied extensively \citet{bach-sgd,rakhlin2012making,curtis-large-scale_ml}. Classical convergence results are mostly concerned with \emph{mean-squared error} (MSE) convergence, characterizing the average performance across many runs of the algorithm. However, due to significant computational cost of a single run of an algorithm in many modern machine learning applications, it is often infeasible to perform multiple runs \citet{harvey2019tight, davis2021low}. As such, many applications require more fine-grained results, such as \emph{high-probability convergence}, which characterize the behaviour of an algorithm with respect to a single run.

Another striking feature of existing works is the assumption that the gradient noise has \textit{light-tails} or \textit{uniformly bounded variance} \citet{rakhlin2012making, ghadimi2012optimal, ghadimi2013stochastic}, which represents a major limitation in many modern applications, see \cite{csimcsekli2019heavy,simsekli2019tail}. For example, \citet{zhang2020adaptive} show that the gradient noise distribution during training of large attention models resembles a Levy $\alpha$-stable distribution, with $\alpha < 2$, which has unbounded variance. To better model this phenomena, the authors propose the \emph{bounded $p$-th moment} assumption, i.e.,
\begin{equation}\label{eq:bounded-moment}
    \E_{\upsilon \sim \Upsilon}\|\nabla \ell(\bx,\upsilon) - \nabla f(\bx) \|^p \leq \sigma^p, \tag{BM}
\end{equation} for every $\bx \in \R^d$ and some $p \in (1,2]$, $\sigma > 0$, subsuming the bounded variance case for $p = 2$. Under this assumption, \citet{zhang2020adaptive} show that SGD fails to converge for any fixed step-size. The clipped variant of SGD solves this problem and achieves \textit{optimal} MSE convergence rate for smooth non-convex losses. Along with addressing heavy-tailed noise, clipped SGD also addresses non-smoothness of the cost \citet{zhang2019gradient}, ensures differential privacy \citet{zhang2022clip_FL_icml} and robustness to malicious nodes in distributed learning \citet{shuhua-clipping}. While popular, clipping is not the only nonlinearity employed in practice. Sign and quantized variants of SGD improve communication efficiency in distributed learning \citet{alistarh2017qsgd, bernstein2018signsgd, gandikota2021vqsgd}. Sign SGD achieves performance on par with state-of-the-art adaptive methods \citet{crawshaw2022general_signSGD}, and is robust to faulty and malicious users \citet{bernstein2018signsgd_iclr}. Normalized SGD is empirically observed to accelerate neural network training \citet{hazan2015beyond, you2019reducing, cutkosky20normalized_SGD} and facilitates privacy \citet{das2021DP_normFedAvg, yang2022normalized_sgd}, while \citet{zhang2020adaptive} empirically observe that component-wise clipping converges faster than the joint one, showing better dependence on problem dimension. Although assumption \eqref{eq:bounded-moment} helps bridge the gap between theory and practice, the downside is that the resulting convergence rates have exponents which explicitly depend on the noise moment and vanish as $p \rightarrow 1$. This seems to contradict the strong performance of nonlinear SGD methods observed in practice and fails to explain the empirical success of nonlinear SGD, e.g., during training of models such as neural networks, in the presence of heavy-tailed noise. A growing body of works recently provided strong evidence that the stochastic noise during training of neural networks is \emph{symmetric}, by studying the empirical distribution of gradient noise during training. For example, \citet{bernstein2018signsgd,bernstein2018signsgd_iclr} show that histograms of gradient noise during training of different Resnet architectures on CIFAR-10 and Imagenet data exhibit strong symmetry under various batch sizes, see their Figures 2 (in both works). Similarly, \citet{chen2020understanding} demonstrate strong symmetry of gradient distributions during training of convolutional neural networks (CNN) on CIFAR-10 and MNIST data, see their Figures 1-3. \citet{barsbey-heavy_tails_and_compressibility} show that the histograms of weights of a CNN layer trained on MNIST data almost identically match samples simulated from a symmetric $\alpha$-stable distribution, see their Figure 2. Finally, \citet{pmlr-v238-battash24a} show that a heavy-tailed symmetric $\alpha$-stable distribution is a much better fit for the stochastic gradient noise than a Gaussian, for a myriad of deep learning architectures and datasets, see their Tables 1-3. Relying on a generalization of the central limit theorem (CLT), \citet{simsekli2019tail,pmlr-v108-peluchetti20b,heavy-tail-phenomena,barsbey-heavy_tails_and_compressibility} theoretically show that symmetric heavy-tailed noises are appropriate models in many practical settings, e.g., when training neural networks with mini-batch SGD using a large batch size. In contrast, works using assumption~\eqref{eq:bounded-moment} are inherently oblivious to this widely observed phenomena. The goal of this paper is to study high-probability guarantees of nonlinear SGD methods in the presence of symmetric heavy-tailed noise and the benefits symmetry brings.

\begin{table*}[!htp]
\caption{High-probability guarantees of SGD methods under heavy-tailed noise. Online indicates whether a method uses a time-varying step-size and is applicable in the online setting (indicated by lower-case $t$), or if it uses a fixed step-size and requires a preset time horizon which is optimized to achieve the best rate and works only in the offline setting (indicated by upper-case $T$). The value $\beta \in (0,1)$ represents the failure probability, while $\widetilde{\mathcal{O}}(\cdot)$ hides factors poly-logarithmic in time $t$. All the works achieve a poly-logarithmic dependence on the failure probability $\beta$ (i.e., contain a multiplicative factor of $\log(\nicefrac{1}{\beta})$ in the bound), which is hidden under the big O notation, for ease of presentation.}
\label{tab:comp}
\begin{adjustwidth}{-1in}{-1in} 
\begin{center}
\begin{threeparttable}
\begin{small}
\begin{sc}
\begin{tabular}{cccccc}
\toprule
\multicolumn{1}{c}{\rule{0pt}{2.5ex}\scriptsize Cost} & \multicolumn{1}{c}{\scriptsize Work} & \multicolumn{1}{c}{\scriptsize Nonlinearity} & \multicolumn{1}{c}{\scriptsize Noise} & \multicolumn{1}{c}{\scriptsize Online} & \multicolumn{1}{c}{\scriptsize Rate}\\
\midrule
\multirow{4}{*}{\scriptsize Non-convex} & \scriptsize\citet{nguyen2023improved} & \multirow{2.5}{*}{\scriptsize Clipping only} & \multirow{2.4}{*}{$\substack{\text{\scriptsize{unbiased, bounded}} \\ \text{\scriptsize{moment of order }} p \in (1,2]}$} & \ding{52} & \scriptsize $\widetilde{\mathcal{O}}\left(t^{\nicefrac{2(1 - p)}{(3p - 2)}} \right)$ \\
& \scriptsize \citet{sadiev2023highprobability} & & & \ding{55} & \scriptsize $\mathcal{O}\left(T^{\nicefrac{(1 - p)}{p}} \right)$ \\
\cmidrule{2-6}
& \scriptsize This paper & $\substack{\text{\scriptsize Component-wise} \\ \text{\scriptsize and joint}}$ & $\substack{\text{\scriptsize{symmetric pdf,}} \\ \text{\scriptsize{positive around zero}}}$ & \ding{52} & \scriptsize $\widetilde{\mathcal{O}}\left(t^{-\nicefrac{1}{4}}\right)^\dagger$ \\
\midrule
\multicolumn{1}{c}{\multirow{4}{*}{$\substack{\text{\scriptsize Strongly} \\ \text{\scriptsize convex}}$}} & \scriptsize \citet{sadiev2023highprobability} & \scriptsize Clipping only & $\substack{\text{\scriptsize{unbiased, bounded}} \\ \text{\scriptsize{moment of order }} p \in (1,2]}$ & \ding{55} & \scriptsize $\mathcal{O}\left(T^{\nicefrac{2(1-p)}{p}} \right)$ \\
\cmidrule{2-6}
\multicolumn{1}{c}{} & $\substack{\text{\scriptsize This paper - weighted} \\ \text{\scriptsize average of iterates}}$ & \multirow{2}{*}{$\substack{\text{\scriptsize Component-wise} \\ \text{\scriptsize and joint}}$} & \multirow{2}{*}{$\substack{\text{\scriptsize{symmetric pdf,}} \\ \text{\scriptsize{positive around zero}}}$} & \ding{52} & \scriptsize $\widetilde{\mathcal{O}}\left(t^{-\nicefrac{1}{4}}\right)^\dagger$ \\
\multicolumn{1}{c}{} &  \scriptsize This paper - last iterate  & & & \ding{52} & \scriptsize $\mathcal{O}\left(t^{-\zeta}\right)^\S$ \\
\bottomrule
\end{tabular}
\end{sc}
\end{small}
\begin{tablenotes}\scriptsize
    \item[$\dagger$] We derive convergence guarantees for a wide range of step-sizes of the form $\alpha_t = \nicefrac{a}{(t+1)^\delta}$, where $a > 0$, $\delta \in (\nicefrac{2}{3},1)$, with the resulting convergence rate depending on $\delta$. The best rate, shown in the table, is achieved for the choice $\delta = \nicefrac{3}{4}$. 
    \item[$\S$] The rate $\zeta \in (0,1)$ depends on the choice of nonlinearity, noise and problem related parameters, see Section~\ref{sec:main} and Appendix~\ref{app:rate}. We provide examples of noise for which $\zeta > \nicefrac{2(p - 1)}{p}$, see Examples~\ref{example:1}-\ref{example:4} ahead.
\end{tablenotes}
\end{threeparttable}
\end{center}
\vskip -0.1in
\end{adjustwidth}
\end{table*}

\paragraph{Literature Review.} We now review the literature on high-probability convergence of SGD and its variants. Initial works on high-probability convergence of stochastic gradient methods considered light-tailed noise and include \citet{nemirovski2009robust, lan2012optimal,hazan2014beyond, harvey2019tight,ghadimi2013stochastic,li2020high}. Subsequent works \citet{gorbunov2020stochastic,gorbunov2021near,parletta2022high} generalized these results to noise with bounded variance. \citet{pmlr-v151-tsai22a} study clipped SGD, assuming the variance is bounded by iterate distance, while \citet{li2022high,eldowa2023general,madden2020high} consider sub-Weibull noise. Recent works~\citet{liu2023high,eldowa2023general} remove restrictive assumptions, like bounded stochastic gradients and domain. \citet{sadiev2023highprobability} show that even with bounded variance and smooth, strongly-convex functions, vanilla SGD cannot achieve an exponential tail decay, implying that the complexity of achieving a high-probability bound for SGD can be much worse than that of the corresponding MSE bound. As such, nonlinear SGD is used to handle tails heavier than sub-Gaussian. Recent works consider a class of heavy-tailed noises satisfying \eqref{eq:bounded-moment}, e.g., \citet{nguyen2023improved, nguyen2023high, sadiev2023highprobability,liu2023breaking}. \citet{nguyen2023improved,nguyen2023high} study high-probability convergence of clipped SGD for convex and non-convex minimization, \citet{sadiev2023highprobability} study clipped SGD for optimization and variational inequality problems, while \citet{liu2023breaking} study accelerated variants of clipped SGD for smooth losses. It is worth mentioning \citet{gorbunov2023breaking}, who show that clipped SGD achieves the optimal $\mathcal{O}\left(T^{-1}\right)$\footnote{We use lower-case $t$ to indicate an online method, using a time-varying step-size, whereas upper-case $T$ indicates an offline method, which uses a fixed-step size and a predefined time horizon $T$. While an online method can clearly be used in the offline setting, the converse is not true.} rate for smooth, strongly convex costs, under a class of heavy-tailed noises with possibly unbounded first moments. However, their noise assumption is difficult to verify, as it requires computing convolutions of order $k$, for all $k \in \N$. Additionally, they use a median-of-means gradient estimator, which requires evaluating multiple stochastic gradients per iteration and is not applicable in the online setting considered in this paper.

The works closest to ours are \citet{nguyen2023improved} for online non-convex and~\citet{sadiev2023highprobability} for offline strongly convex problems. We present a detailed comparison in Table \ref{tab:comp}. Both works study only the clipping operator and use assumption \eqref{eq:bounded-moment}. For non-convex costs, \citet{nguyen2023improved} achieve the optimal rate $\widetilde{\mathcal{O}}\left(t^{\nicefrac{2(1-p)}{(3p-2)}} \right)$, while \cite{sadiev2023highprobability} achieve the optimal rate $\mathcal{O}\left(T^{\nicefrac{2(1-p)}{p}}\right)$ for strongly convex costs. Compared to them, we consider a much broader class of nonlinearities in the presence of noise with symmetric density with no moment requirements, achieving the near-optimal rate $\widetilde{\mathcal{O}}\left(t^{-\nicefrac{1}{4}}\right)$ for non-convex costs and extending it to the weighted average of iterates for strongly convex costs. Crucially, our rate exponent \emph{is independent of noise and problem parameters}, which is not the case with \citet{nguyen2023improved,sadiev2023highprobability}. Our rates are strictly better whenever $p < 6/5$ for non-convex and $p < 8/7$ for strongly convex costs.\footnote{This does not contradict the optimality of the rates in \citet{nguyen2023improved,sadiev2023highprobability}, as their assumptions differ from ours. While \citet{nguyen2023improved,sadiev2023highprobability} require bounded noise moment of order $p \in (1,2]$, we study noise with symmetric density, without making any moment requirements. As such, we show that symmetry leads to improved results and allows for relaxed moment conditions and heavier tails (see Examples \ref{example:1}-\ref{example:3}).} Additionally, we establish convergence of the \emph{last iterate} for strongly convex costs, with rate $\mathcal{O}\left(t^{-\zeta} \right)$, where $\zeta \in (0,1)$ depends on noise, nonlinearity and other problem parameters. We give examples of noise regimes where our rate is better than the one in \citet{sadiev2023highprobability} (see Examples \ref{example:1}-\ref{example:4}) and demonstrate numerically that \emph{clipping is not always the best nonlinearity} (see Section \ref{sec:an-num}), further highlighting the importance and usefulness of our general framework. Finally, it is worth mentioning \citet{jakovetic2023nonlinear}, who provide MSE, asymptotic normality and almost sure guarantees of the same nonlinear framework for strongly convex costs and noises with symmetric PDF, positive around zero and bounded first moments. Our work differs in that we study high-probability convergence, relax the moment conditions and allow for non-convex costs. The latter is achieved by providing a novel characterization of the interplay of the ``denoised'' nonlinear gradient and the true gradient (see Lemma \ref{lm:huber}).  

\paragraph{Contributions.} Our contributions are as follows. 
\begin{enumerate}
    \item We study convergence in high probability of a unified framework of nonlinear SGD, in the presence of heavy-tailed noise and widely observed noise symmetry, making no assumptions on noise moments. The nonlinear map is treated in a black-box manner, subsuming many popular nonlinearities, like sign, normalization, clipping and quantization. \emph{To the best of our knowledge, we provide the first high-probability results under heavy-tailed noise for methods such as sign, quantized and component-wise clipped SGD}.
    
    \item For non-convex costs, we show convergence of gradient norm-squared, at a near-optimal rate $\widetilde{\mathcal{O}}\left(t^{-\nicefrac{1}{4}}\right)$. The exponent in our rate is constant, independent of noise and problem parameters, which is not the case with state-of-the-art \citet{nguyen2023improved}. Our rate is strictly better than state-of-the-art whenever the noise has bounded moments of order $p < \frac{6}{5}$.
    
    \item For strongly convex costs we show convergence of the weighted average of iterates, at the same rate $\widetilde{\mathcal{O}}\left(t^{-\nicefrac{1}{4}}\right)$. Our rate dominates the state-of-the-art \citet{sadiev2023highprobability} whenever the noise has bounded moments of order $p < \frac{8}{7}$, while being applicable in the online setting, which is not the case for \citet{sadiev2023highprobability}. For the last iterate we show convergence at a rate $\mathcal{O}\left(t^{-\zeta} \right)$, where $\zeta \in (0,1)$ depends on noise, nonlinearity and problem parameters, but remains bounded away from zero even for unbounded noise moments.
    
    \item We extend our results beyond symmetric noise, by considering a mixture of symmetric and non-symmetric components. For non-convex costs we show convergence to a neighbourhood of stationarity, at a rate $\widetilde{\mathcal{O}}(t^{-\nicefrac{1}{4}})$, where the size of the neighbourhood depends on the mixture coefficient, nonlinearity and noise. While \cite{nguyen2023improved} achieve convergence under condition \eqref{eq:bounded-moment}, which does not require symmetry, they explicitly require \emph{unbiased noise}, which is not the case for our mixture noise, allowing it to be \emph{biased}.
    
    \item Compared to state-of-the-art \citet{nguyen2023improved,sadiev2023highprobability}, who only consider clipping, require bounded noise moments of order $p \in (1,2]$ and whose rates vanish as $p \rightarrow 1$, we consider a much broader class of nonlinearities, relax the moment condition and provide convergence rates with constant exponents. Finally, we provide numerical results that show \emph{clipping is not always the optimal choice of nonlinearity}, further reinforcing the importance of our general framework.
\end{enumerate}

\paragraph{Paper Organization.} The rest of the paper is organized as follows. Section \ref{sec:framework} outlines the proposed framework. Section \ref{sec:main} presents the main results. Section \ref{sec:an-num} provides numerical results. Section \ref{sec:conclusion} concludes the paper. Appendix contains additional experiments and proofs omitted from the main body. The remainder of this section introduces the notation.   

\paragraph{Notation.} The set of positive integers is denoted by $\N$. For $a \in \N$, the set of integers up to and including $a$ is denoted by $[a] = \{1,\ldots,a \}$. The sets of real numbers and $d$-dimensional vectors are denoted by $\R$ and $\R^d$. Regular and bold symbols denote scalars and vectors, i.e., $x \in \R$ and $\bx \in \R^d$. The Euclidean inner product and induced norm are denoted by $\langle \cdot,\cdot\rangle$ and $\|\cdot\|$.

\section{Proposed Framework}\label{sec:framework}

To solve~\eqref{eq:problem} in the online setting, under the presence of heavy-tailed noise, we use the \textit{nonlinear SGD} framework. The algorithm starts by choosing a deterministic initial model $\bx^{(1)} \in \mbb R^d$,\footnote{While the initial model is deterministically chosen, it can be any vector in $\R^d$. This distinction is required for the theoretical analysis in the next section.} a step-size schedule $\{\alpha_t\}_{t \in \N}$ and a nonlinear map $\boldsymbol{\Psi}:\mbb R^d \mapsto \mbb R^d$. In iteration $t = 1,2,\ldots$, the method performs as follows: a first-order oracle is queried, which returns the gradient of the loss $\ell$ evaluated at the current model $\bxt$ and a random sample $\upsilon^{(t)}$.\footnote{Equivalently, the oracle directly sends the random sample $\upsilon^{(t)}$, which we use to compute the gradient of $\ell$.} Then, the model is updated as
\begin{equation}\label{eq:update}
    \bxtp = \bxt - \alpha_t\mathbf{\Psi}\left(\nabla \ell(\bxt;\upsilon^{(t)})\right),
\end{equation} where $\alpha_t > 0$ is the step-size at iteration $t$. The method is summed up in Algorithm~\ref{alg:nonlin-sgd}. We make the following assumption on the nonlinear map $\bPsi$.

\begin{algorithm}[!tb]
\caption{Online Nonlinear SGD}
\label{alg:nonlin-sgd}
\begin{algorithmic}[1]
   \REQUIRE{Choice of nonlinearity $\bPsi: \R^d \mapsto \R^d$, model initialization $\bx^{(1)} \in \R^{d}$, step-size schedule $\{\alpha_t\}_{t \in \N}$;}
   \FOR{t = 1,2,\ldots}:
        \STATE Query the oracle and receive $\nabla \ell(\bxt;\upsilon^{(t)})$;  
        \STATE Update $\bxtp \leftarrow \bxt - \alpha_t\mathbf{\Psi}\left(\nabla \ell(\bxt;\upsilon^{(t)})\right)$;
    \ENDFOR
\end{algorithmic}
\end{algorithm}

\begin{assump}\label{asmpt:nonlin}
The nonlinear map $\bPsi: \mbb R^d \mapsto \mbb R^d$ is either of the form $\bPsi(\bx) = \bPsi(x_1,\dots,x_d) = \lbr \calN_1(x_1), \dots, \calN_1(x_d) \rbr^\top$ or $\bPsi(\bx) = \bx\calN_2(\|\bx\|)$, where $\calN_1,\: \calN_2: \R \mapsto \R$ satisfy: 
\begin{enumerate}
    \item $\calN_1,\calN_2$ are continuous almost everywhere,\footnote{With respect to the Lebesgue measure.} $\calN_1$ is piece-wise differentiable and the map $a \mapsto a\calN_2(a)$ is non-decreasing.

    \item $\calN_1$ is monotonically non-decreasing and odd, while $\calN_2$ is non-increasing.

    \item $\calN_1$ is either discontinuous at zero, or strictly increasing on $(-c_1,c_1)$, for some $c_1 > 0$, with $\calN_2(a) > 0$, for any $a > 0$.

    \item $\calN_1$ and $\bx\calN_2(\|\bx\|)$ are uniformly bounded, i.e., $|\calN_1(x)| \leq C_1$ and $\|\bx\calN_2(\|\bx\|)\|\leq C_2$, for some $C_1,\: C_2 > 0$, and all $x \in \R$, $\bx \in \R^d$.
\end{enumerate}
\end{assump}

Note that the fourth property implies $\|\bPsi(\bx)\| \leq C$, where $C = C_1\sqrt{d}$ or $C = C_2$, depending on the form of nonlinearity. We will use the general bound $\|\bPsi(\bx) \| \leq C$ for ease of presentation, and specialize where appropriate. Assumption~\ref{asmpt:nonlin} is satisfied by a wide class of nonlinearities, including:
\begin{enumerate}
    \item \emph{Sign}: $[\bPsi(\bx)]_i = \text{sign}(x_i), \: i \in [d]$.
    
    \item \emph{Component-wise clipping}: $[\bPsi(\bx)]_i = x_i$, for $|x_i| \leq m$, and $[\bPsi(\bx)]_i = m\cdot\text{sign}(x_i)$, for $|x_i| > m$, $i \in [d]$, for user-specified $m>0$.
    
    \item \emph{Component-wise quantization}: for each $i \in [d]$, let $[\bPsi(\bx)]_i = r_j$, for $x_i \in (q_j,q_{j+1}]$, with $j = 0,\ldots,J-1$ and $-\infty = q_0 < q_1 <\ldots < q_J = +\infty$, where $r_j$,$q_j$ are chosen such that each component of $\bPsi$ is odd, and we have $\max_{j \in \{0,\ldots,J-1\}}|r_j| < R$, for user-specified $R > 0$.
    
    \item \emph{Normalization}: $\bPsi(\bx) = \frac{\bx}{\norm{\bx}}$ and $\bPsi(\bx) = \mathbf{0}$, if $\bx = \mathbf{0}$.
    
    \item \emph{Clipping}: $\bPsi(\bx) = \min\{1,\nicefrac{M}{\|\bx\|}\}\bx$, for user-specified $M>0$.
\end{enumerate}

\section{Theoretical Guarantees}\label{sec:main}

In this section we present the main results of the paper. Subsection \ref{subsec:prelim} presents the preliminaries, Subsection \ref{subsec:theory-nonconv} presents the results for symmetric noises, while Subsection \ref{subsec:non-sym} presents the results for non-symmetric noises. The proofs can be found in the Appendix.

\subsection{Preliminaries}\label{subsec:prelim}
In this section we provide the preliminaries and assumptions used in the analysis. To begin, we state the assumptions on the behaviour of the cost $f$.

\begin{assump}\label{asmpt:L-smooth}
The cost $f$ is bounded from below, has at least one stationary point and Lipschitz continuous gradients, i.e., $\inf_{\bx \in \R^d}f(\bx) > -\infty$, there exists a $\bx^\star \in \R^d$, such that $\nabla f(\bx^\star) = 0$, and $\|\nabla f(\bx) - \nabla f(\by) \| \leq L\|\bx - \by\|$, for some $L > 0$ and every $\bx,\by \in \R^d$.
\end{assump}

\begin{remark}
    Boundedness from below and Lipschitz continuous gradients are standard for non-convex losses, e.g., \cite{ghadimi2013stochastic}. Since the goal in non-convex optimization is to reach a stationary point, it is natural to assume at least one such point exists, see \citet{liu2023high,madden2020high}. 
\end{remark}

\begin{remark}
    It can be shown that Lipschitz continuous gradients imply the $L$-smothness inequality, i.e., $f(\by) \leq f(\bx) + \langle\nabla f(\bx),\by - \bx\rangle + \frac{L}{2} \|\bx - \by \|^2$, for any $\bx, \by \in \R^d$, see \citet{nesterov-lectures_on_cvxopt,Wright_Recht_2022}.
\end{remark}

In addition to Assumption \ref{asmpt:L-smooth}, we will sometimes use the following assumption.

\begin{assump}\label{asmpt:cvx}
The cost $f$ is strongly convex, i.e., $f(\by) \geq f(\bx) + \langle\nabla f(\bx),\by - \bx\rangle + \frac{\mu}{2} \|\bx - \by \|^2$, for some $\mu > 0$ and every $\bx,\by \in \R^d$.
\end{assump}

Denote the infimum of $f$ by $f^\star \triangleq \inf_{\bx \in \R^d}f(\bx)$. Denote the set of stationary points of $f$ by $\mathcal{X} \triangleq \left\{\bx^\star \in \R^d: \: \nabla f(\bx^\star) = 0 \right\}$. By Assumption \ref{asmpt:L-smooth}, it follows that $\mathcal{X} \neq \emptyset$. If in addition Assumption \ref{asmpt:cvx} holds, we have $\mathcal{X} = \{\bx^\star\}$ and $f^\star = f(\bx^\star)$, for some $\bx^\star \in \R^d$. Denote the distance of the initial model from the set of stationary points by $D_{\mathcal{X}} \triangleq \inf_{\bx \in \mathcal{X}}\|\bx^{(1)} - \bx\|^2$. Next, rewrite the update \eqref{eq:update} as
\begin{align}\label{eq:nonlin-sgd}
    \bxtp = \bxt - \alpha_t \boldsymbol{\Psi}(\G f(\bxt) + \bzt), 
\end{align} where $\bzt \triangleq \nabla \ell(\bxt;\upsilon^{(t)}) - \G f(\bxt)$ is the stochastic noise at iteration $t$. To simplify the notation, we use the shorthand $\bPsi^{(t)} \triangleq \boldsymbol{\Psi}(\nabla f(\bxt) + \bzt)$. We make the following assumption on the noise vectors $\{\bzt \}_{t \in \N}$.

\begin{assump}\label{asmpt:noise}
    The noise vectors $\{\bzt \}_{t \in \N}$ are independent, identically distributed, with symmetric probability density function (PDF) $P: \R^d \mapsto \R$, positive around zero, i.e., $P(-\bz) = P(\bz)$, for all $\bz \in \R^d$ and $P(\bz) > 0$, for all $\|\bz\| \leq B_0$ and some $B_0 > 0$.
\end{assump}

\begin{remark}
    Assumption \ref{asmpt:noise} imposes no moment conditions, at the expense of requiring a symmetric PDF, positive in a neighborhood of zero. Symmetry and positivity around zero are mild assumptions, satisfied by many noise distributions, such as Gaussian, the ones in Examples \ref{example:1}-\ref{example:3} below, and a broad class of heavy-tailed symmetric $\alpha$-stable distributions, e.g., \citet{stable-distributions,heavy-tail-book}. 
\end{remark}

\begin{remark}
    As discussed in the introduction, heavy-tailed symmetric noise has been widely observed during training of deep learning models, across different architectures, datasets and batch sizes, e.g., \citet{bernstein2018signsgd,bernstein2018signsgd_iclr,chen2020understanding,barsbey-heavy_tails_and_compressibility,pmlr-v238-battash24a}.Building on the generalized CLT, \citet{simsekli2019tail,pmlr-v108-peluchetti20b,heavy-tail-phenomena,barsbey-heavy_tails_and_compressibility} provide theoretical justification for this phenomena, e.g., when training neural nets with a large batch size.
\end{remark}

\begin{remark}\label{rmk:iid}
    The independent, identically distributed requirement in Assumption \ref{asmpt:noise} can be significantly relaxed, to allow for noises which are not identically distributed, and in the case of joint nonlinearities, potentially depend on the current model. The reader is referred to the Appendix for a detailed discussion.
\end{remark}

\begin{remark}
    Positivity around zero of the PDF is a technical condition, ensuring that the ``denoised nonlinearity'' (see Section \ref{sec:main} ahead) is well-behaved. As such, the magnitude of the neighborhood $B_0$ does not directly affect the bounds established in Section \ref{sec:main}.  
\end{remark}

\begin{remark}\label{rmk:noise-comparison}
    While the noise assumption used in our work and the $p$-th bounded moment assumption \eqref{eq:bounded-moment} are different, neither set of assumptions is uniformly weaker, with both having some advantages and disadvantages. For a detailed comparison between the two sets of assumptions, the reader is referred to the Appendix. 
\end{remark}

We now give some examples of noise PDFs satisfying Assumption~\ref{asmpt:noise}.

\begin{example}\label{example:1}
    The noise PDF $P(\bz) = \rho(z_1)\times \ldots \times \rho(z_d)$, where $\rho(z) = \frac{\alpha - 1}{2(1 + |z|)^\alpha}$, for some $\alpha > 2$. It can be shown that the noise only has finite $p$-th moments for $p < \alpha - 1$.
\end{example}

\begin{example}\label{example:2}
    The noise PDF $P(\bz) = \rho(z_1) \times \ldots \times \rho(z_d)$, where $\rho(z) = \frac{c}{(z^2 + 1)\log^2(|z| + 2)}$, with $c = \int \nicefrac{1}{[(z^2 + 1)\log^2(|z| + 2)]}dz$ being the normalizing constant. It can be shown that the noise has a finite first moment, but for any $p \in (1,2]$, the $p$-th moments do not exist.
\end{example}

\begin{example}\label{example:Cauchy}
    The noise PDF $P(\bz) = \rho(z_1) \times \ldots \times \rho(z_d)$, where $\rho(z) = \frac{\gamma}{\pi\gamma^2 + \pi(x - x_0)^2}$, for some $x_0 \in \R$ and $\gamma > 0$, i.e., each component is distributed according to the Cauchy distribution. In this case, even the mean of the noise does not exist. 
\end{example}

\begin{example}\label{example:3}
    The PDF $P: \R^d \mapsto \R$ with ``radial symmetry'', i.e., $P(\bz) = \rho(\|\bz\|)$, where $\rho: \R \mapsto \R$ is itself a PDF. If $\rho$ is the PDF from Example \ref{example:2}, then the noise does not have finite $p$-th moments, for any $p > 1$, while if $\rho$ is the PDF of the Cauchy distribution, then the noise does even not have the first moment. 
\end{example}

\begin{remark} 
    While noise in Example \ref{example:1} satisfies moment condition \eqref{eq:bounded-moment}, noises in Examples \ref{example:2} and \ref{example:Cauchy} do not.
\end{remark}

Next, define the function $\bPhi:\mbb R^d \mapsto \mbb R^d$, given by $\bPhi(\bx) \triangleq \mbe_{\bz} [\bPsi (\bx + \bz)] = \int \bPsi(\bx+\bz) P(\bz) d\bz$,\footnote{If $\bPsi$ is a component-wise nonlinearity, then $\bPhi$ is a vector with components $\phi_i(x_i) = \mbe_{z_i}[\calN_1(x_i + z_i)]$, where $\E_{z_i}$ is the marginal expectation with respect to the $i$-th noise component, $i \in [d]$ (see Lemma~\ref{lm:polyak-tsypkin} ahead).} where the expectation is taken with respect to the gradient noise at a random sample. We use the shorthand $\bPhit \triangleq \mbe_{\bzt} [\bPsi (\nabla f(\bxt) + \bzt) \: \vert \: \mathcal{F}_t]$,\footnote{Conditioning on $\mathcal{F}_t$ ensures that the quantity $\nabla f(\bxt)$ is deterministic and $\bPhit$ is well defined.} where $\mathcal{F}_t$ is the natural filtration, i.e., $\mathcal{F}_1 \triangleq \sigma(\{\emptyset,\Omega\})$ and $\mathcal{F}_t \triangleq \sigma\left(\{\bx^{(2)},\ldots,\bx^{(t)}\}\right)$, for $t \geq 2$.\footnote{Recall that in our setup, the initialization $\bx^{(1)} \in \R^d$ is an arbitrary, but deterministic quantity.} The vector $\bPhit$ can be seen as the ``denoised'' version of $\bPsi^{(t)}$. Using $\bPhit$, we can rewrite the update rule~\eqref{eq:nonlin-sgd} as
\begin{equation}\label{eq:nonlin-sgd2}
    \bxtp = \bxt - \alpha_t\bPhit + \alpha_t\bet,    
\end{equation} where $\bet \triangleq \boldsymbol{\Phi}^{(t)} - \boldsymbol{\Psi}^{(t)}$ represents the \emph{effective noise} term. As we show next, the effective noise is light-tailed, even though the original noise may not be.

\begin{lemma}\label{lm:error_component}
    Let Assumptions~\ref{asmpt:nonlin} and~\ref{asmpt:noise} hold. Then, the effective noise vectors $\{\bet\}_{t \in \N}$ satisfy:
    \begin{enumerate}
        \item $\E[\bet\vert \: \mathcal{F}_t] = 0$  and  $\|\bet\| \leq 2C$.

        \item The effective noise is sub-Gaussian, i.e., for any $\bx \in \R^d$, we have $\E\left[\exp\left(\langle \bx, \bet \rangle \right) \: \vert \: \mathcal{F}_t \right] \leq \exp\left(4C^2\|\bx\|^2\right)$.
    \end{enumerate}
\end{lemma}

\subsection{Main Results}\label{subsec:theory-nonconv}

In this section we establish convergence in high probability of the proposed framework. Our results are facilitated by a novel result on the interplay of $\bPhi(\bx)$ and the original vector $\bx$, which is presented next. 

\begin{lemma}\label{lm:huber}
    Let Assumptions \ref{asmpt:nonlin} and \ref{asmpt:noise} hold. Then $\langle \bPhi(\bx),\bx\rangle \geq \min\left\{\eta_1\|\bx\|,\eta_2\|\bx\|^2 \right\}$, for any $\bx \in \R^d$, where $\eta_1,\eta_2 > 0$, are noise, nonlinearity and problem dependent constants.
\end{lemma}

Lemma~\ref{lm:huber} provides a novel characterization of the inner product between the ``denoised'' nonlinearity $\bPhi$ at vector $\bx$ and the vector $\bx$ itself. We specialize the value of constants $\eta_1,\eta_2$ for different nonlinearities in the Appendix. We are now ready to state our high-probability convergence bounds for non-convex costs.

\begin{theorem}\label{thm:non-conv}
    Let Assumptions \ref{asmpt:nonlin}, \ref{asmpt:L-smooth} and \ref{asmpt:noise} hold. Let $\{\bxt\}_{t \in \N}$ be the sequence generated by Algorithm \ref{alg:nonlin-sgd}, with step-size $\alpha_t = \frac{a}{(t + 1)^\delta}$, for any $\delta \in (\nicefrac{2}{3},1)$ and $a > 0$. Then, for any $t \geq 1$ and $\beta \in (0,1)$, with probability at least $1 - \beta$, the following hold.
    \begin{enumerate}
        \item For the choice $\delta \in (\nicefrac{2}{3},\nicefrac{3}{4})$, we have 
        \begin{align*}
            & \quad \min_{k \in [t]}\|\nabla f(\bxk)\|^2 = \frac{2R_1(\beta)/\eta_2}{(t+2)^{1-\delta}-2^{1-\delta}} + \frac{2R_2/\eta_2}{(t+2)^{3\delta-2}-2^{3\delta-2}}
            \\ & \quad \quad\quad\quad+ \left(\frac{2R_1(\beta)/\eta_1}{(t+2)^{1-\delta}-2^{1-\delta}}\right)^2 + \left(\frac{2R_2/\eta_1}{(t+2)^{3\delta-2}-2^{3\delta-2}}\right)^2,    
        \end{align*} where $R_1(\beta) \triangleq (1-\delta)\big[\left(f(\bx^{(1)}) - f^\star + \log(\nicefrac{1}{\beta})\right)/a + aLC^2(\nicefrac{1}{2} + 8LD_{\mathcal{X}})/(2\delta-1)\big]$ and $R_2 \triangleq \frac{8a^3C^4L^2}{(1-\delta)(3-4\delta)}$.

        \item For the choice $\delta = \nicefrac{3}{4}$, we have
        \begin{align*}
            &\min_{k \in [t]}\|\nabla f(\bxk)\|^2 = \frac{2R_3(t,\beta)/\eta_2}{(t+2)^{1/4} - 2^{1/4}} + \left(\frac{\sqrt{2}R_3(t,\beta)/\eta_1}{(t+2)^{1/4} - 2^{1/4}}\right)^2,
        \end{align*} where $R_3(t,\beta) \triangleq \left(f(\bx^{(1)}) - f^\star + \log(\nicefrac{1}{\beta})\right)/4a + aLC^2(\nicefrac{1}{4} + 4LD_{\mathcal{X}}) + 32a^3C^4L^2\log(t+1)$.

        \item For the choice $\delta \in (\nicefrac{3}{4},1)$, we have 
        \begin{align*}
            &\min_{k \in [t]}\|\nabla f(\bxk)\|^2 = \frac{2R_4(\beta)/\eta_2}{(t+2)^{1-\delta} - 2^{1-\delta}} + \left(\frac{\sqrt{2}R_4(\beta)/\eta_1}{(t+2)^{1-\delta} - 2^{1-\delta}}\right)^2,
        \end{align*} where $R_4(\beta) \triangleq (1-\delta)\big[\left(f(\bx^{(1)}) - f^\star + \log(\nicefrac{1}{\beta})\right)/a + aLC^2(\nicefrac{1}{2} + 8LD_{\mathcal{X}})/(2\delta-1)\big] + \nicefrac{8a^3C^4L^2}{(1-\delta)(4\delta-3)}$.
    \end{enumerate}
\end{theorem}

\begin{remark}
    Theorem \ref{thm:non-conv} provides convergence in high-probability of nonlinear SGD in the online setting, for a broad range of nonlinearities and step-sizes, with the best rate achieved for the choice $\delta = \nicefrac{3}{4}$. Compared to \cite{nguyen2023improved}, who achieve the rate $\mathcal{O}(t^{\nicefrac{2(1-p)}{(3p-2)}}\log(\nicefrac{t}{\beta}))$ for clipped SGD, with the exponent explicitly depending on $p$ and vanishing as $p \rightarrow 1$, our results apply to a broad range of nonlinearities and are strictly better whenever $p < 6/5$.  
\end{remark}

\begin{remark}
    Note that for both step-size choices $\delta_1 \in (\nicefrac{2}{3},\nicefrac{3}{4})$ and $\delta_2 \in (\nicefrac{3}{4},1)$, we can get arbitrarily close to the rate $t^{-\nicefrac{1}{4}}$, by choosing $\delta_1 = \nicefrac{3}{4} - \epsilon_1$, for $\epsilon_1 \in (0,\nicefrac{1}{12})$ and $\delta_2 = \nicefrac{3}{4} + \epsilon_2$, for $\epsilon_2 \in (0,\nicefrac{1}{4})$. In both cases, the rate incurs a constant multiplicative factor $1/\epsilon_i$, $i \in [2]$.
\end{remark}

\begin{remark}
    For the choice of $\delta = 3/4$, our rate incurs an additional $\log(t+1)$ factor. This additional factor is unavoidable in online learning, where the time horizon is unknown and a time-varying step-size is required. The rate in \cite{nguyen2023improved} incurs the same logarithmic factor for the ``unknown $T$'' regime (see Theorem B.2 in their work). The logarithmic factor can be removed if the time horizon $T \in \N$ is preset, by using a fixed step-size inversely proportional to the time horizon, i.e., $\alpha_t \equiv \alpha \propto T^{-3/4}$. Our analysis readily goes through when a fixed step-size is used.     
\end{remark}

\begin{remark}\label{rmk:metric}
    The guarantees in Theorem \ref{thm:non-conv} (and Theorem \ref{thm:non-sym} ahead) are stated in terms of the metric $\min_{k \in [t]}\|\nabla f(\bxk)\|^2$, widely used in non-convex optimization. However, in our proof, we provide guarantees of the same order for $\sum_{k = 1}^t\widetilde{\alpha}_k\min\{\|\nabla f(\bxk)\|,\|\nabla f(\bxk)\|^2\}$, which is more general, in the sense that the high-probability bounds on this metric imply the  bounds on the metric $\min_{k \in [t]}\|\nabla f(\bxk)\|^2$. For details, as well as comparisons with the metric used in \cite{nguyen2023improved}, see the Appendix.
\end{remark}

\begin{remark}
    The convergence bounds in Theorem \ref{thm:non-conv} depend on standard quantity, such as the initialization gap (through $f(\bx^{(1)})-f^\star$ and $D_{\mathcal{X}}$) and smoothness parameter $L$, as well as nonlinearity and noise dependent quantities, such as $C$, $\eta_1$ and $\eta_2$. These constants can be specialized for specific nonlinearities and noises. For example, consider the noise from Example \ref{example:1}, sign nonlinearity and step-size parameter $\delta = 3/4$. It can then be shown that $C = \sqrt{d}$, $\eta_1 = \nicefrac{(\alpha-1)}{2\alpha\sqrt{d}}$, $\eta_2 = \nicefrac{(\alpha-1)}{d}$ (see the Appendix), resulting in the following problem related constant (up to a logarithmic factor) in the leading term $\frac{ad^2L(\nicefrac{1}{4} + 4LD_{\mathcal{X}})}{\alpha-1} + \frac{32a^3d^3L^2\log(t+1)}{\alpha-1} +\frac{d\left(f(\bx^{(1)}) - f^\star + \log(\nicefrac{1}{\beta})\right)}{4a(\alpha-1)}$, where we recall that $\alpha > 2$. Choosing $\alpha = d^{-1/2}$, reduces the overall dependence on problem dimension to $d^{3/2}$. 
\end{remark}

We specialize the rates from Theorem \ref{thm:non-conv} for specific choices of nonlinearities and noise in the Appendix, showing that our rates predict that \emph{clipping is not always the optimal choice of nonlinearity} and confirm the findings of \cite{zhang2020adaptive}, namely that component-wise clipping demonstrates better dimension dependence that joint clipping, for some noise instances. This is further validated in our numerical experiments in Section \ref{sec:an-num}.

Next, if the cost is strongly convex, results of Theorem \ref{thm:non-conv} can be improved. Define $\widetilde{\alpha}_k \triangleq \nicefrac{\alpha_k}{\sum_{j = 1}^{t}\alpha_j}$, $k \in [t]$, so that $\sum_{k = 1}^{t}\widetilde{\alpha}_k = 1$ and define a weighted average of iterates as $\widehat{\bx}^{(t)} \triangleq \sum_{k = 1}^{t}\widetilde{\alpha}_k\bxk$. The estimator $\widehat{\bx}^{(t)}$ can be seen as generalized Polyak-Ruppert averaging, e.g., \citet{ruppert,polyak,polyak-ruppert}. We then have the following result.

\begin{corollary}\label{cor:cvx}
    Let Assumptions \ref{asmpt:nonlin}-\ref{asmpt:noise} hold. Let $\{\bxt\}_{t \in \N}$ be the sequence generated by Algorithm \ref{alg:nonlin-sgd}, with step-size $\alpha_t = \frac{a}{(t + 1)^\delta}$, for any $\delta \in (\nicefrac{2}{3},1)$ and $a > 0$. Then, for any $t \geq 1$ and any $\beta \in (0,1)$, with probability at least $1 - \beta$, the following hold.
    \begin{enumerate}
        \item For the choice $\delta \in (\nicefrac{2}{3},\nicefrac{3}{4})$, we have $\|\widehat{\bx}^{(t)} - \bx^\star\|^2 = \mathcal{O}\left((t^{\delta - 1} + t^{2-3\delta})\log(\nicefrac{1}{\beta}))\right)$.

        \item For the choice $\delta = \nicefrac{3}{4}$, we have $\|\widehat{\bx}^{(t)} - \bx^\star\|^2 = \mathcal{O}\left(t^{-\nicefrac{1}{4}}\log(\nicefrac{t}{\beta}) \right)$.

        \item For the choice $\delta \in (\nicefrac{3}{4},1)$, we have $\|\widehat{\bx}^{(t)} - \bx^\star\|^2 = \mathcal{O}\left(t^{\delta - 1}\log(\nicefrac{1}{\beta}) \right)$.
    \end{enumerate}
\end{corollary}

\begin{remark}
    Corollary \ref{cor:cvx} maintains the rates from Theorem \ref{thm:non-conv}, while improving on the metric of interest, providing guarantees for the generalized Polyak-Ruppert average $\widehat{\bx}^{(t)}$. Compared to \cite{sadiev2023highprobability}, who show convergence of the last iterate for clipped SGD in the offline setting, with a rate $\mathcal{O}(T^{\nicefrac{2(1-p)}{p}})$, our results again apply to a much broader range of nonlinearities and the online setting, beating the rate from \cite{sadiev2023highprobability} whenever $p < \nicefrac{8}{7}$.
\end{remark}

For strongly convex functions it is of interest to characterize the convergence guarantees of the last iterate \citet{harvey2019tight,pmlr-v151-tsai22a,sadiev2023highprobability}. To that end, we first characterize the interplay between $\bPhit$ and $\nabla f(\bxt)$.  

\begin{lemma}\label{lm:key}
    Let Assumptions~\ref{asmpt:nonlin}-\ref{asmpt:noise} hold and $\{\bxt\}_{t \in \N}$ be the sequence generated by Algorithm \ref{alg:nonlin-sgd}, with step-size $\alpha_t = \frac{a}{(t + 1)^\delta}$, for any $\delta \in (\nicefrac{1}{2},1)$ and $a > 0$. Then $\langle \bPhit, \nabla f(\bxt) \rangle \geq \gamma(t+2)^{\delta - 1}\|\nabla f(\bxt)\|^2$, for some $\gamma = \gamma(a) > 0$ and any $t \geq 1$, almost surely.
\end{lemma}

We then have the following result.

\begin{theorem}
\label{theorem:main}
    Suppose Assumptions~\ref{asmpt:nonlin}-\ref{asmpt:noise} hold and $\{\bxt\}_{t \in \N}$ is the sequence generated by Algorithm \ref{alg:nonlin-sgd}, with step-size $\alpha_t = \frac{a}{(t + 1)^\delta}$, for any $\delta \in (\nicefrac{1}{2},1)$ and $a > 0$. Then, for any $t \geq 1$ and $\beta \in (0,1)$, with probability at least $1 - \beta$, it holds that
    \begin{equation*}
        \|\bxtp - \bx^\star\|^2 = \mathcal{O}\left(\log(\nicefrac{1}{\beta})(t+1)^{-\zeta}\right),
    \end{equation*} 
    where $\zeta = \min\left\{2\delta - 1, a\mu\gamma/2 \right\}$.
\end{theorem}

We specialize $\gamma = \gamma(a)$ for different nonlinearities and discuss its impact on the rate in the Appendix. The value of $\zeta$ can be explicitly calculated for specific choices of nonlinearities and noise, as we show next.

\begin{example}\label{example:4}
    For the noise from Example \ref{example:1} and sign nonlinearity, it can be shown that $\zeta \approx \min\big\{2\delta - 1,\frac{\mu}{L}\frac{1-\delta}{\sqrt{d}}\frac{\alpha-1}{\alpha} \big\}$, while for component-wise clipping it can be shown that $\zeta \approx \min\big\{2\delta-1,\frac{\mu}{L\sqrt{d}}\frac{(1-\delta)(m-1)(1-(m+1)^{-\alpha})}{m} \big\}$, see \citet{jakovetic2023nonlinear}. On the other hand, the rate from \cite{sadiev2023highprobability} for joint clipping, adapted to the same noise, is $\nicefrac{2(r - 1)}{r}$, where $r \leq \min\{\alpha - 1,2 \}$. If $\alpha = 2 + \epsilon$, where $\epsilon \in (0,1]$, i.e., very heavy tails, then moments of order $1 < p < 1 + \epsilon \leq 2$ exist, satisfying the bounded $p$-th moment condition from \cite{sadiev2023highprobability}, with their best-case rate given by $2(r-1)/r = 2\epsilon/(1+\epsilon) < 2\epsilon$. On the other hand, consider the sign nonlinearity with step-size parameter $\delta = 3/4$. In this case, our rate is given by $\zeta = \min\big\{\frac{1}{2},\frac{\mu(1+\epsilon)}{4L\sqrt{d}(2+\epsilon)}\big\} = \frac{\mu(1+\epsilon)}{4L\sqrt{d}(2+\epsilon)} > \frac{\mu}{8L\sqrt{d}}$, where the last inequality follows from $\epsilon \in (0,1]$. Using the corresponding lower and upper bounds, it follows that our rate is strictly better than the one from \cite{sadiev2023highprobability}, i.e., $\zeta > 2(r-1)/r$, whenever $\epsilon < \frac{\mu}{16L\sqrt{d}}$. Therefore, for any heavy-tailed noise of the form given in Example 1, such that $\alpha = 2 + \epsilon$, with $0 < \epsilon < \frac{\mu}{16L\sqrt{d}}$, the noise condition in both our work and \cite{sadiev2023highprobability} is satisfied, with our rate being strictly better. Additionally, we can see that our rate for noises of this form is uniformly bounded below by a quantity constant with respect to $\alpha$, i.e., $\zeta > \frac{\mu}{8L\sqrt{d}}$. On the other hand, the rate from \cite{sadiev2023highprobability}, specialized to noises from Example 1 with $\alpha = 2+\epsilon$ and $\epsilon \in (0,1]$, is upper-bounded by a quantity strictly depending on the noise, i.e., $2(r-1)/r < 2\epsilon$, with $2(r-1)/r \rightarrow 0$ as $\alpha \rightarrow 2$ (i.e., as $\epsilon \rightarrow 0$). Similar results can be shown to hold for component-wise clipping.
\end{example}

\subsection{Beyond Symmetric Noise}\label{subsec:non-sym}

In this section we extend our results to the case when the noise is not necessarily symmetric. In particular, we make the following assumption on the noise vectors. 

\begin{assump}\label{asmpt:non-sym}
    The noise vectors $\{\bzt\}_{t \in \N}$ are independent, identically distributed, drawn from a mixture distribution with PDF $P(\bz) = (1 - \lambda)P_1(\bz) + \lambda P_2(\bz)$, where $P_1(\bz)$ is symmetric and $\lambda \in (0,1)$ is the mixture coefficient. Additionally, $P_1$ is positive around zero, i.e., $P_1(\bz) > 0$, for all $\|\bz\| \leq B_0$ and some $B_0 > 0$. 
\end{assump}

\begin{remark}
    Assumption \ref{asmpt:non-sym} relaxes Assumption \ref{asmpt:noise}, by allowing for a mixture of symmetric and non-symmetric noises. The resulting noise is non-symmetric and in general does not have to be zero mean, i.e., we allow for the oracle to send \emph{biased} gradient estimators. We again make no assumptions on noise moments.
\end{remark}

\begin{remark}\label{rmk:mixture}
    Assumption \ref{asmpt:non-sym} arises naturally in scenarios like training with a large batch size, in which the generalized CLT implies that the noise becomes more symmetric as the batch size grows \citet{simsekli2019tail,pmlr-v108-peluchetti20b,heavy-tail-phenomena,barsbey-heavy_tails_and_compressibility}. In such scenarios, the effect of the non-symmetric part decreases with batch size, resulting in small $\lambda$ for a large enough batch size.   
\end{remark}

We then have the following result.

\begin{theorem}\label{thm:non-sym}
    Let Assumptions \ref{asmpt:nonlin}, \ref{asmpt:L-smooth} and \ref{asmpt:non-sym} hold. Let $\{\bxt\}_{t \in \N}$ be the sequence generated by Algorithm \ref{alg:nonlin-sgd}, with step-size $\alpha_t = \frac{a}{(t + 1)^\delta}$, for any $\delta \in (\nicefrac{2}{3},1)$ and $a > 0$. If $\lambda < \frac{\eta_1}{C+\eta_1}$, then for any $t \geq 1$ and $\beta \in (0,1)$, with probability at least $1 - \beta$, the following hold.
    \begin{enumerate}
        \item For the choice $\delta \in (\nicefrac{2}{3},\nicefrac{3}{4})$, we have $\min_{k \in [t]}\|\nabla f(\bxk)\|^2 = \mathcal{O}\left((t^{\delta - 1} + t^{2-3\delta})\right) + \frac{\eta_1\lambda C }{\eta_2^2(1-\lambda)}$.

        \item For the choice $\delta = \nicefrac{3}{4}$, we have $\min_{k \in [t]}\|\nabla f(\bxk)\|^2 = \mathcal{O}\left(\log(t)/t^{\nicefrac{1}{4}} \right) + \frac{\eta_1 \lambda C }{\eta_2^2(1-\lambda)}$.

        \item For the choice $\delta \in (\nicefrac{3}{4},1)$, we have $\min_{k \in [t]}\|\nabla f(\bxk)\|^2 = \mathcal{O}\left(t^{\delta - 1} \right) + \frac{\eta_1\lambda C }{\eta_2^2(1-\lambda)}$.
    \end{enumerate}
\end{theorem}

\begin{remark}
    All three step-size regimes in Theorem \ref{thm:non-sym} again achieve exponential tail decay, i.e., a $\log(\nicefrac{1}{\beta})$ dependence on the failure probability $\beta$, which is hidden under the big O notation, for ease of exposition. 
\end{remark}

\begin{remark}
    Theorem \ref{thm:non-sym} provides convergence guarantees to a neighbourhood of stationarity, for mixtures of symmetric and non-symmetric components, if the contribution of the non-symmetric component is sufficiently small. As discussed in Remark \ref{rmk:mixture}, this can be guaranteed, e.g., by using a sufficiently large batch size. As the neighborhood size is determined by the mixture component via $\nicefrac{\lambda}{(1-\lambda)}$, it follows that, as the noise becomes more symmetric (i.e., $\lambda \rightarrow 0$), we recover exact convergence from Theorem \ref{thm:non-conv}.    
\end{remark}

\begin{remark}
    While \cite{nguyen2023improved} guarantee convergence of gradient norm-squared to zero under condition \eqref{eq:bounded-moment}, which allows for non-symmetric noises, it is important to note that they explicitly require that the oracle sends \emph{unbiased} gradient estimators. On the other hand, we allow for the oracle to send \emph{biased} gradient estimators. Without incorporating a correction mechanism (e.g., momentum or error-feedback), in general, it is not possible to guarantee exact convergence with a biased oracle.  
\end{remark}

The size of the neighbourhood and the condition on the mixture coefficient provided in Theorem \ref{thm:non-sym} are both determined by the noise and choice of nonlinearity. We can specialize the constants $\eta_1,\eta_2$ and $C$ for specific choices of nonlinearities and noises. We now give some examples. For full derivations, see the Appendix.

\begin{example}\label{example:5}
    Consider the noise from Example \ref{example:1}. For the sign nonlinearity it can be shown that $\eta_1 = \nicefrac{(\alpha-1)}{2\alpha\sqrt{d}}$, $\eta_2 = \nicefrac{(\alpha - 1)}{2d}$ and $C = \sqrt{d}$, resulting in convergence to a neighborhood of size $\nicefrac{2d^2\lambda}{[\alpha(\alpha-1)(1-\lambda)]}$ and $\lambda < \nicefrac{(\alpha-1)}{[\alpha(2d + 1)- 1] }$. As $\alpha > 2$, we can see that $\lambda < \nicefrac{1}{(2d+1)}$, at best and $\lambda < \nicefrac{1}{(4d+1)}$, at worst.
\end{example} 

\begin{example}
    For component-wise clipping with $m > 1$, we have $\eta_1 = \nicefrac{[1 - (m+1)^{-\alpha}](m-1)}{2\sqrt{d}}$, $\eta_2 = \nicefrac{[1 - (m+1)^{-\alpha}]}{2d}$ and $C = m\sqrt{d}$, resulting in convergence to a neighborhood of size $\nicefrac{2d^2m(m-1)\lambda}{[1-(m+1)^{-\alpha}](1-\lambda)}$ and $\lambda < \nicefrac{(m-1)[1-(m+1)^{-\alpha}]}{[(m-1)[1-(m+1)^{-\alpha}] + 2md]}$. While taking $m \rightarrow 1$ results in full convergence, we can see that this simultaneously implies $\lambda \rightarrow 0$, i.e., requiring the noise to be symmetric.
\end{example} 

\begin{example}\label{example:7}
    For joint clipping with threshold $M > 0$, we have $\eta_1 = [\nicefrac{(\alpha - 1)}{2}]^d\min\{1,M\}/2$, $\eta_2 = [\nicefrac{(\alpha-1)}{2}]^d\min\{1,M\}$ and $C = M$, resulting in convergence to a neighborhood of size $\nicefrac{2^{d-1}\lambda M}{[(\alpha-1)^d(1-\lambda)\min\{1,M\}]}$ and $\lambda < \frac{(\alpha - 1)^d\min\{1,M\}}{2(\alpha - 1)^d\min\{1,M\} + 2^{d+1}M}$. Choosing $M \leq 1$, results in converging to a neighborhood of size $\nicefrac{2^{d-1}\lambda}{(\alpha - 1)^d(1 - \lambda)}$ and condition $\lambda < \nicefrac{(\alpha - 1)^d}{2(\alpha - 1)^d + 2^{d+1}}$. Similar observations hold for $M > 1$.
\end{example}

\section{Numerical Results}\label{sec:an-num}

In this section we present numerical results. The first set of experiments demonstrates the noise symmetry phenomena on a deep learning model with real data. The second set of experiments compares the behaviour of different nonlinearities on a toy example. Additional experiments can be found in the Appendix.

\paragraph{Noise Symmetry - Setup.} We train a Convolutional Neural Network (CNN) \citet{lecun2015deep} on the MNIST dataset \citet{lecun1998mnist}, using PyTorch\footnote{https://github.com/pytorch/examples/tree/main/mnist}. The CNN we use consists of two convolutional layers, followed by $2 \times 2$ max pooling with a stride of 2, and two fully connected layers, with all layers using ReLU activations. The network is trained using the Adadelta optimizer \citet{zeiler2012adadelta} with $\ell_2$ gradient clipping threshold $M = 1$. For full details on the network and hyperparameter tuning, see the Appendix.

\paragraph{Noise Symmetry - Visualization.} Similar to the visualization method used in \citet{chen2020understanding}, we evaluate the symmetry of gradient distribution by projecting all per-sample gradients into a 2-D space using random Gaussian matrices. For any symmetric distribution, its 2-D projection remains symmetric under any projection matrix. Conversely, if the projected gradient distribution is symmetric for every projection matrix, the original gradient distribution is also symmetric. In Figure \ref{fig:proj-epochs}, we present a 2-D plot of the random projections of all per-sample gradients after training for several epochs, with epoch 0 showing the gradient distribution at the initialization. We can observe that all the random projections exhibit a high degree of symmetry over the duration of the entire training process. 

\begin{figure*}[htbp]
    \centering
    \begin{subfigure}[b]{0.3\textwidth}
        \centering
        \includegraphics[width=\textwidth]{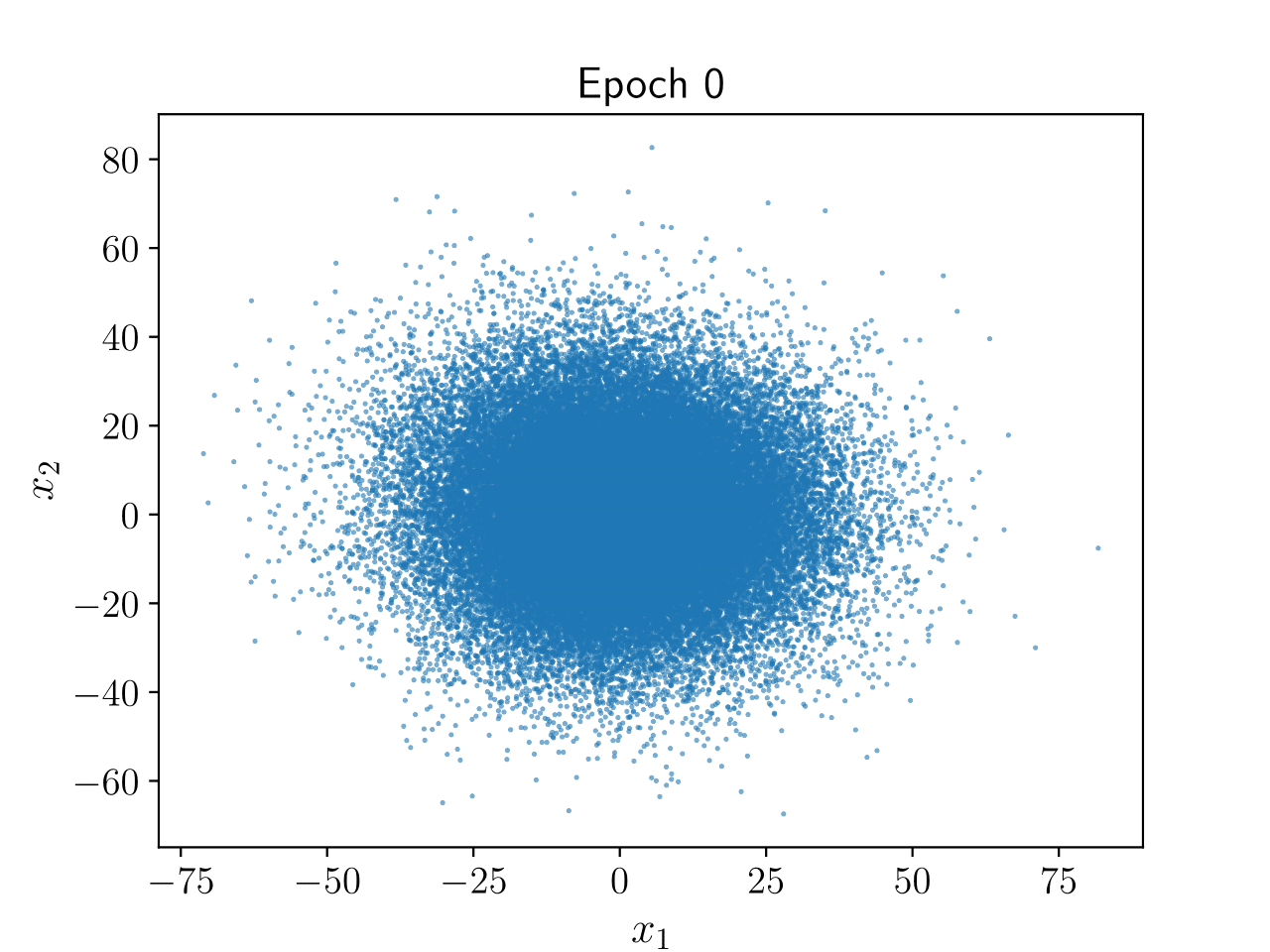} 
    \end{subfigure}
    \hfill
    \begin{subfigure}[b]{0.3\textwidth}
        \centering
        \includegraphics[width=\textwidth]{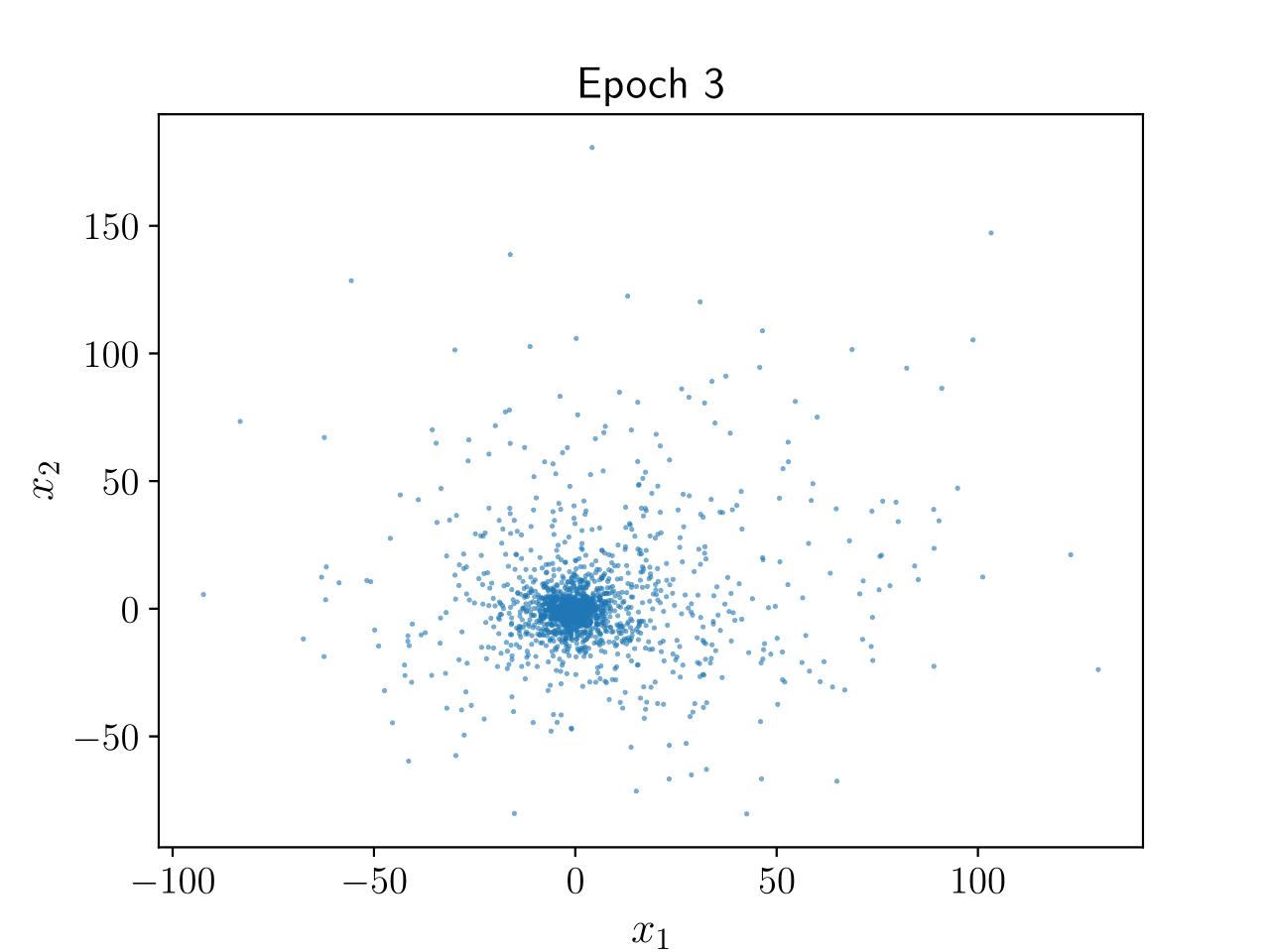} 
    \end{subfigure}
    \hfill
    \begin{subfigure}[b]{0.3\textwidth}
        \centering
        \includegraphics[width=\textwidth]{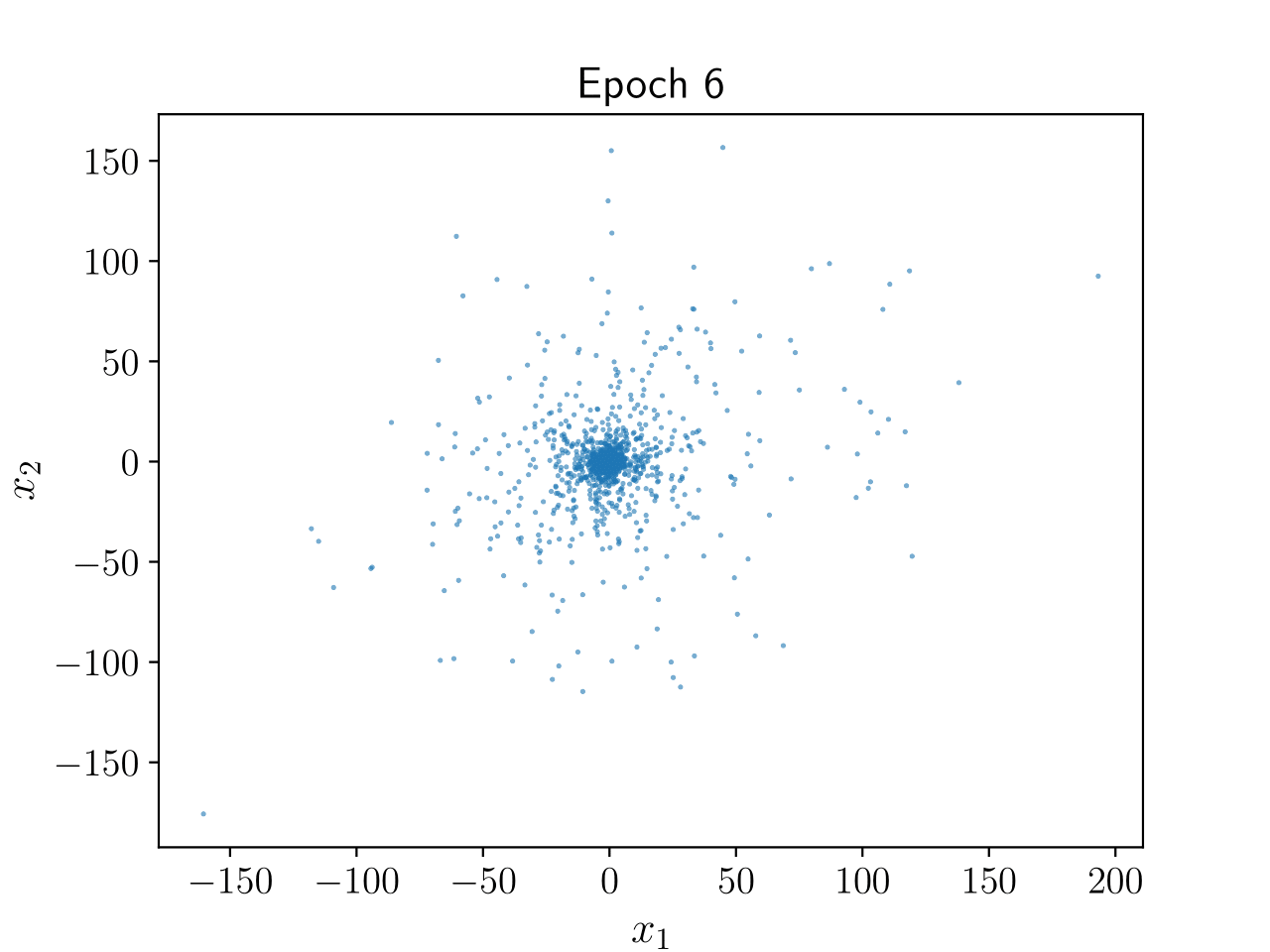} 
    \end{subfigure}
    \vskip\baselineskip
    \begin{subfigure}[b]{0.3\textwidth}
        \centering
        \includegraphics[width=\textwidth]{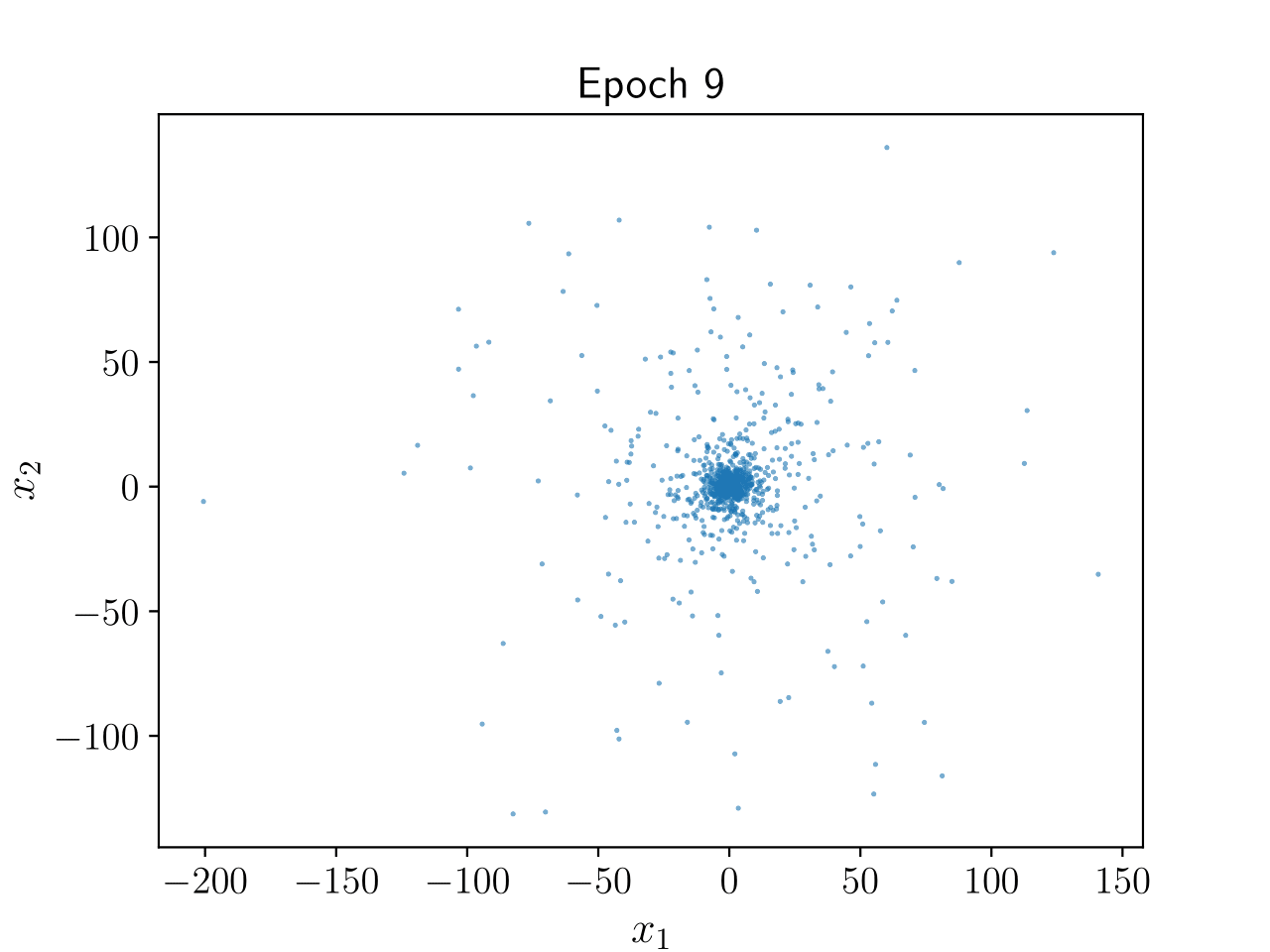} 
    \end{subfigure}
    \hfill
    \begin{subfigure}[b]{0.3\textwidth}
        \centering
        \includegraphics[width=\textwidth]{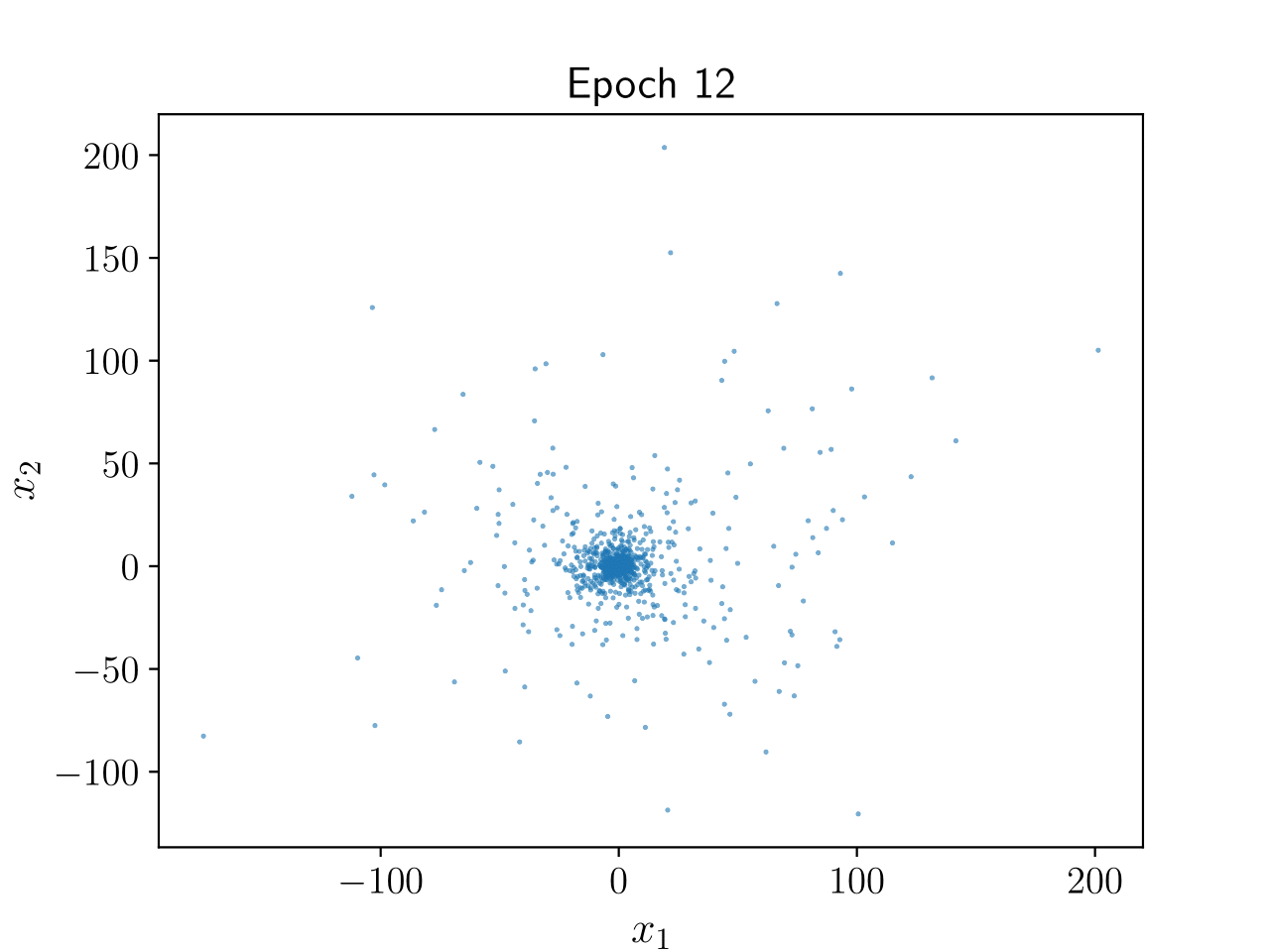}
    \end{subfigure}
    \hfill
    \begin{subfigure}[b]{0.3\textwidth}
        \centering
        \includegraphics[width=\textwidth]{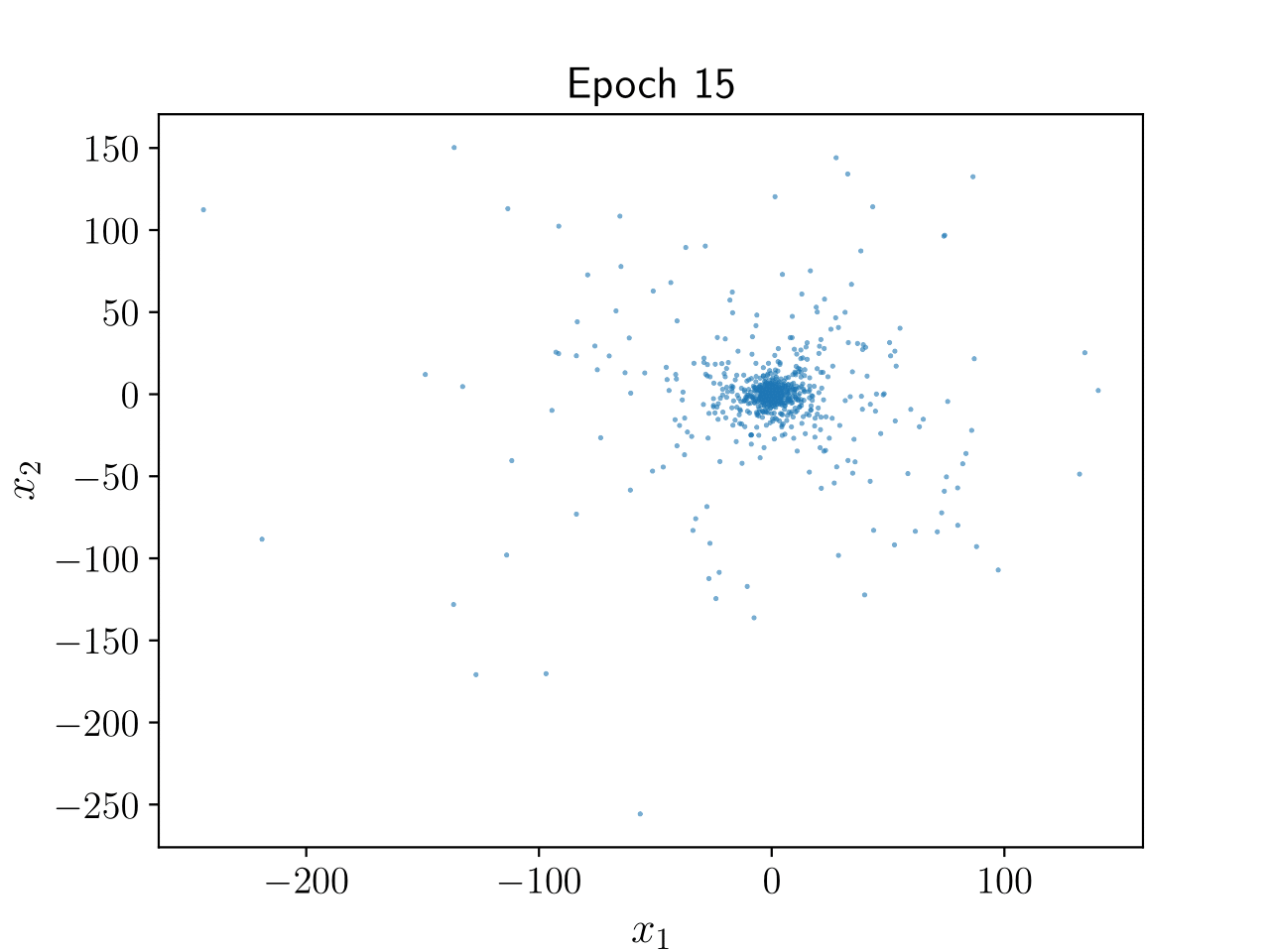} 
    \end{subfigure}
    \caption{Random projections of per-sample gradients across epochs.}
    \label{fig:proj-epochs}
\end{figure*}

\paragraph{Nonlinearity Comparison.} In this set of experiments, we compare the performance of multiple nonlinear SGD methods across different metrics, using a toy example. We consider a quadratic problem $f(\bx) = \frac{1}{2}\bx^\top A\bx + \mathbf{b}^\top\bx$, where $A \in \R^{d \times d}$ is positive definite and $\mathbf{b} \in \R^d$ is fixed. We set $d = 100$. The stochastic noise is generated according to the PDF from Example~\ref{example:1}, with $\alpha = 2.05$. We compare the performance of sign, component-wise and joint clipped SGD, with all three using the step-size schedule $\alpha_t = \frac{1}{t+1}$. We choose the clipping thresholds $m$ and $M$ for which component-wise and joint clipping performed the best, those being $m = 1$ and $M = 100$. All three algorithms are initialized at the zero vector and perform $T = 25000$ iterations, across $R = 5000$ runs. To evaluate the performance of the methods, we use the following criteria:\\
1. \emph{Mean-squared error}: we present the MSE of the algorithms, by evaluating the gap $\|\bx^{(t)} - \bx^\star\|^2$ in each iteration, averaged across all runs, i.e., the final estimator at iteration $t = 1,\ldots,T$, is given by $MSE^t = \frac{1}{R}\sum_{r = 1}^R\|\bx^{(t)}_r - \bx^\star\|^2$, where $\bx^{(t)}_r$ is the $t$-th iterate in the $r$-th run, generated by the algorithms.\\
2. \emph{High-probability estimate}: we evaluate the high-probability behaviour of the methods, as follows. We consider the events $A^t = \{\|\bx^{(t)} - \bx^\star\|^2 > \varepsilon \}$, for a fixed $\varepsilon > 0$. To estimate the probability of $A^t$, for each $t = 1,\ldots, T$, we construct a Monte Carlo estimator of the empirical probability, by sampling $n = 3000$ instances from the $R = 5000$ runs, uniformly with replacement. We then obtain the empirical probability estimator as $\mathbb{P}_n(A^t) = \frac{1}{n}\sum_{i = 1}^n\mathbb{I}_i(A^{t}) = \frac{1}{n}\sum_{i = 1}^n\mathbb{I}\big(\{\|\bx_i^{(t)} - \bx^\star\|^2 > \varepsilon\}\big)$, where $\mathbb{I}(\cdot)$ is the indicator function and $\bx^{(t)}_i$ is the $i$-th Monte Carlo sample. 

The results are presented in Figure~\ref{fig:fig1}. We can see that component-wise nonlinearities outperform joint clipping, both in terms of MSE and high-probability performance and demonstrate exponential tail decay, validating our theoretical results and further underlining the usefulness of considering a general framework beyond only clipping. Additional experiments can be found in the Appendix. 

\begin{figure}[!ht]
\centering
\begin{tabular}{lll}
\includegraphics[scale=0.3]{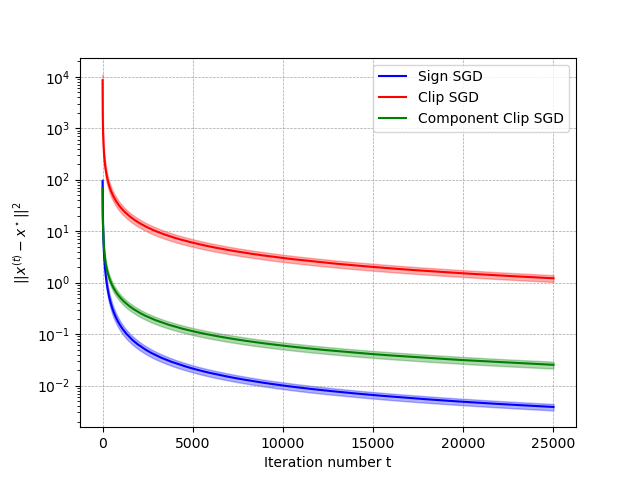}
&
\includegraphics[scale=0.3]{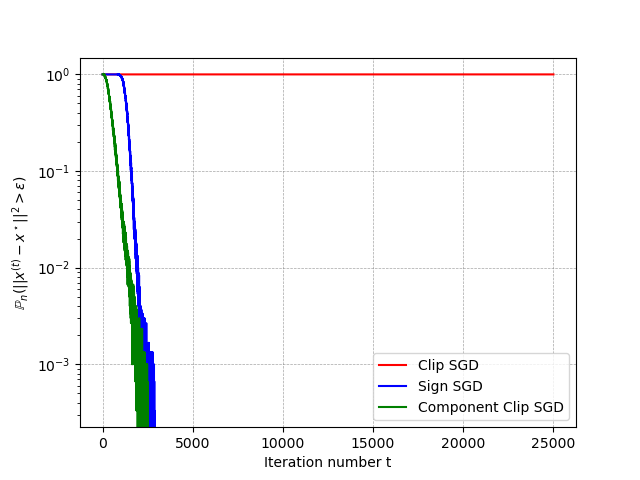}
&
\includegraphics[scale=0.3]{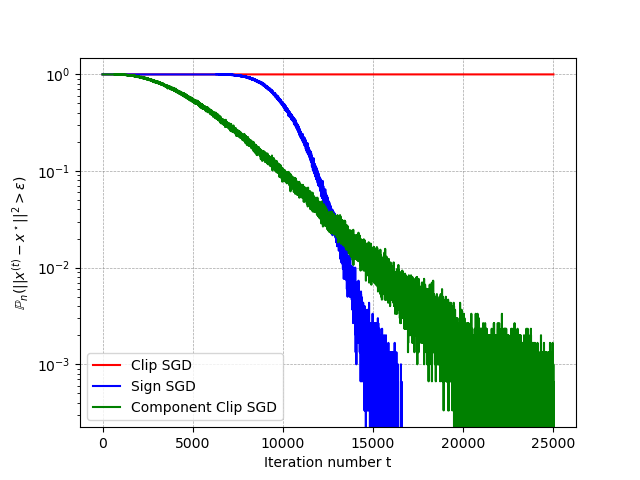}
\end{tabular}
\caption{Performance of sign, component-wise clipping and joint clipping. Left to right: MSE performance and high-probability performance for $\varepsilon = \{0.1, 0.01 \}$, respectively. We can see that both component-wise nonlinearities converge faster in the MSE sense and achieve exponential tail decay. Note that clipping does not achieve exponentially decaying tails in second and third figures, as it does not reach the required accuracy in the allocated number of iterations.}
\label{fig:fig1}
\end{figure}

\section{Conclusion}\label{sec:conclusion}

We present high-probability guarantees for a broad class of nonlinear SGD algorithms in the online setting and the presence of heavy-tailed noise with symmetric PDF. We establish near-optimal $\widetilde{\mathcal{O}}(t^{-\nicefrac{1}{4}})$ convergence rates for non-convex costs, and similar rates for the weighted average of iterates for strongly convex costs. Additionally, for the last iterate of strongly convex costs we establish convergence at a rate $\mathcal{O}(t^{-\zeta})$, where $\zeta \in (0,1)$ depends on noise and other problem parameters. We extend our analysis to noises that are mixtures of symmetric and non-symmetric components, showing convergence to a neighbourhood of stationarity, where the size of the neighborhood depends on choice of nonlinearity, noise and mixture coefficient. Compared to state-of-the-art works~\citet{nguyen2023improved,sadiev2023highprobability}, we extend the high-probability convergence guarantees to a broad class of nonlinearities, relax the noise moment condition, and demonstrate regimes in which our convergence rates are strictly better. This is made possible by a novel result on the interplay between the ``denoised'' nonlinearity and the gradient. Numerical results confirm the theory and demonstrate that clipping, exclusively considered in prior works, is not always the optimal choice of nonlinearity, further highlighting the importance and usefulness of our general framework.

\bibliography{bibliography}

\begin{thebibliography}{}

\bibitem[Alistarh et~al., 2017]{alistarh2017qsgd}
Alistarh, D., Grubic, D., Li, J., Tomioka, R., and Vojnovic, M. (2017).
\newblock Qsgd: Communication-efficient sgd via gradient quantization and
  encoding.
\newblock {\em Advances in neural information processing systems}, 30.

\bibitem[Barsbey et~al., 2021]{barsbey-heavy_tails_and_compressibility}
Barsbey, M., Sefidgaran, M., Erdogdu, M.~A., Richard, G., and Simsekli, U.
  (2021).
\newblock Heavy tails in sgd and compressibility of overparametrized neural
  networks.
\newblock In Ranzato, M., Beygelzimer, A., Dauphin, Y., Liang, P., and Vaughan,
  J.~W., editors, {\em Advances in Neural Information Processing Systems},
  volume~34, pages 29364--29378. Curran Associates, Inc.

\bibitem[Battash et~al., 2024]{pmlr-v238-battash24a}
Battash, B., Wolf, L., and Lindenbaum, O. (2024).
\newblock Revisiting the noise model of stochastic gradient descent.
\newblock In Dasgupta, S., Mandt, S., and Li, Y., editors, {\em Proceedings of
  The 27th International Conference on Artificial Intelligence and Statistics},
  volume 238 of {\em Proceedings of Machine Learning Research}, pages
  4780--4788. PMLR.

\bibitem[Bercovici et~al., 1999]{stable-distributions}
Bercovici, H., Pata, V., and Biane, P. (1999).
\newblock Stable laws and domains of attraction in free probability theory.
\newblock {\em Annals of Mathematics}, 149(3):1023--1060.

\bibitem[Bernstein et~al., 2018a]{bernstein2018signsgd}
Bernstein, J., Wang, Y.-X., Azizzadenesheli, K., and Anandkumar, A. (2018a).
\newblock signsgd: Compressed optimisation for non-convex problems.
\newblock In {\em International Conference on Machine Learning}, pages
  560--569. PMLR.

\bibitem[Bernstein et~al., 2018b]{bernstein2018signsgd_iclr}
Bernstein, J., Zhao, J., Azizzadenesheli, K., and Anandkumar, A. (2018b).
\newblock signsgd with majority vote is communication efficient and fault
  tolerant.
\newblock {\em arXiv preprint arXiv:1810.05291}.

\bibitem[Bottou, 2010]{bottou2010large}
Bottou, L. (2010).
\newblock Large-scale machine learning with stochastic gradient descent.
\newblock In {\em Proceedings of COMPSTAT'2010: 19th International Conference
  on Computational StatisticsParis France, August 22-27, 2010 Keynote, Invited
  and Contributed Papers}, pages 177--186. Springer.

\bibitem[Bottou et~al., 2018]{curtis-large-scale_ml}
Bottou, L., Curtis, F.~E., and Nocedal, J. (2018).
\newblock Optimization methods for large-scale machine learning.
\newblock {\em SIAM Review}, 60(2):223--311.

\bibitem[Chen et~al., 2020]{chen2020understanding}
Chen, X., Wu, S.~Z., and Hong, M. (2020).
\newblock Understanding gradient clipping in private sgd: A geometric
  perspective.
\newblock {\em Advances in Neural Information Processing Systems},
  33:13773--13782.

\bibitem[Crawshaw et~al., 2022]{crawshaw2022general_signSGD}
Crawshaw, M., Liu, M., Orabona, F., Zhang, W., and Zhuang, Z. (2022).
\newblock Robustness to unbounded smoothness of generalized signsgd.
\newblock {\em Advances in Neural Information Processing Systems},
  35:9955--9968.

\bibitem[Cutkosky and Mehta, 2020]{cutkosky20normalized_SGD}
Cutkosky, A. and Mehta, H. (2020).
\newblock Momentum improves normalized sgd.
\newblock In {\em International conference on machine learning}, pages
  2260--2268. PMLR.

\bibitem[Das et~al., 2021]{das2021DP_normFedAvg}
Das, R., Hashemi, A., Sanghavi, S., and Dhillon, I.~S. (2021).
\newblock Dp-normfedavg: Normalizing client updates for privacy-preserving
  federated learning.
\newblock {\em arXiv preprint arXiv:2106.07094}.

\bibitem[Davis et~al., 2021]{davis2021low}
Davis, D., Drusvyatskiy, D., Xiao, L., and Zhang, J. (2021).
\newblock From low probability to high confidence in stochastic convex
  optimization.
\newblock {\em The Journal of Machine Learning Research}, 22(1):2237--2274.

\bibitem[Eldowa and Paudice, 2023]{eldowa2023general}
Eldowa, K. and Paudice, A. (2023).
\newblock General tail bounds for non-smooth stochastic mirror descent.
\newblock {\em arXiv preprint arXiv:2312.07142}.

\bibitem[Gandikota et~al., 2021]{gandikota2021vqsgd}
Gandikota, V., Kane, D., Maity, R.~K., and Mazumdar, A. (2021).
\newblock vqsgd: Vector quantized stochastic gradient descent.
\newblock In {\em International Conference on Artificial Intelligence and
  Statistics}, pages 2197--2205. PMLR.

\bibitem[Ghadimi and Lan, 2012]{ghadimi2012optimal}
Ghadimi, S. and Lan, G. (2012).
\newblock Optimal stochastic approximation algorithms for strongly convex
  stochastic composite optimization i: A generic algorithmic framework.
\newblock {\em SIAM Journal on Optimization}, 22(4):1469--1492.

\bibitem[Ghadimi and Lan, 2013]{ghadimi2013stochastic}
Ghadimi, S. and Lan, G. (2013).
\newblock Stochastic first-and zeroth-order methods for nonconvex stochastic
  programming.
\newblock {\em SIAM Journal on Optimization}, 23(4):2341--2368.

\bibitem[Gorbunov et~al., 2020]{gorbunov2020stochastic}
Gorbunov, E., Danilova, M., and Gasnikov, A. (2020).
\newblock Stochastic optimization with heavy-tailed noise via accelerated
  gradient clipping.
\newblock {\em Advances in Neural Information Processing Systems},
  33:15042--15053.

\bibitem[Gorbunov et~al., 2021]{gorbunov2021near}
Gorbunov, E., Danilova, M., Shibaev, I., Dvurechensky, P., and Gasnikov, A.
  (2021).
\newblock Near-optimal high probability complexity bounds for non-smooth
  stochastic optimization with heavy-tailed noise.
\newblock {\em arXiv preprint arXiv:2106.05958}.

\bibitem[Gurbuzbalaban et~al., 2021]{heavy-tail-phenomena}
Gurbuzbalaban, M., Simsekli, U., and Zhu, L. (2021).
\newblock The heavy-tail phenomenon in sgd.
\newblock In Meila, M. and Zhang, T., editors, {\em Proceedings of the 38th
  International Conference on Machine Learning}, volume 139 of {\em Proceedings
  of Machine Learning Research}, pages 3964--3975. PMLR.

\bibitem[Hardt et~al., 2016]{hardt2016train}
Hardt, M., Recht, B., and Singer, Y. (2016).
\newblock Train faster, generalize better: Stability of stochastic gradient
  descent.
\newblock In {\em International conference on machine learning}, pages
  1225--1234. PMLR.

\bibitem[Harvey et~al., 2019]{harvey2019tight}
Harvey, N.~J., Liaw, C., Plan, Y., and Randhawa, S. (2019).
\newblock Tight analyses for non-smooth stochastic gradient descent.
\newblock In {\em Conference on Learning Theory}, pages 1579--1613. PMLR.

\bibitem[Hazan and Kale, 2014]{hazan2014beyond}
Hazan, E. and Kale, S. (2014).
\newblock Beyond the regret minimization barrier: optimal algorithms for
  stochastic strongly-convex optimization.
\newblock {\em The Journal of Machine Learning Research}, 15(1):2489--2512.

\bibitem[Hazan et~al., 2015]{hazan2015beyond}
Hazan, E., Levy, K., and Shalev-Shwartz, S. (2015).
\newblock Beyond convexity: Stochastic quasi-convex optimization.
\newblock {\em Advances in neural information processing systems}, 28.

\bibitem[Huber, 1964]{huber_loss}
Huber, P.~J. (1964).
\newblock {Robust Estimation of a Location Parameter}.
\newblock {\em The Annals of Mathematical Statistics}, 35(1):73 -- 101.

\bibitem[H{\"u}bler et~al., 2024]{hubler2024gradient}
H{\"u}bler, F., Fatkhullin, I., and He, N. (2024).
\newblock From gradient clipping to normalization for heavy tailed sgd.
\newblock {\em arXiv preprint arXiv:2410.13849}.

\bibitem[Jakoveti{\'c} et~al., 2023]{jakovetic2023nonlinear}
Jakoveti{\'c}, D., Bajovi{\'c}, D., Sahu, A.~K., Kar, S., Milo{\v s}evi{\'c},
  N., and Stamenkovi{\'c}, D. (2023).
\newblock Nonlinear gradient mappings and stochastic optimization: A general
  framework with applications to heavy-tail noise.
\newblock {\em SIAM Journal on Optimization}, 33(2):394--423.

\bibitem[Lan, 2012]{lan2012optimal}
Lan, G. (2012).
\newblock An optimal method for stochastic composite optimization.
\newblock {\em Mathematical Programming}, 133(1-2):365--397.

\bibitem[LeCun et~al., 2015]{lecun2015deep}
LeCun, Y., Bengio, Y., and Hinton, G. (2015).
\newblock Deep learning.
\newblock {\em nature}, 521(7553):436--444.

\bibitem[LeCun et~al., 1998]{lecun1998mnist}
LeCun, Y., Cortes, C., and Burges, C.~J. (1998).
\newblock The mnist database of handwritten digits.
\newblock {\em http://yann. lecun. com/exdb/mnist/}.

\bibitem[Li and Liu, 2022]{li2022high}
Li, S. and Liu, Y. (2022).
\newblock High probability guarantees for nonconvex stochastic gradient descent
  with heavy tails.
\newblock In {\em International Conference on Machine Learning}, pages
  12931--12963. PMLR.

\bibitem[Li and Orabona, 2020]{li2020high}
Li, X. and Orabona, F. (2020).
\newblock A high probability analysis of adaptive sgd with momentum.
\newblock {\em arXiv preprint arXiv:2007.14294}.

\bibitem[Liu et~al., 2023a]{liu2023high}
Liu, Z., Nguyen, T.~D., Nguyen, T.~H., Ene, A., and Nguyen, H. (2023a).
\newblock High probability convergence of stochastic gradient methods.
\newblock In {\em International Conference on Machine Learning}, pages
  21884--21914. PMLR.

\bibitem[Liu et~al., 2023b]{liu2023breaking}
Liu, Z., Zhang, J., and Zhou, Z. (2023b).
\newblock Breaking the lower bound with (little) structure: Acceleration in
  non-convex stochastic optimization with heavy-tailed noise.
\newblock In Neu, G. and Rosasco, L., editors, {\em Proceedings of Thirty Sixth
  Conference on Learning Theory}, volume 195 of {\em Proceedings of Machine
  Learning Research}, pages 2266--2290. PMLR.

\bibitem[Madden et~al., 2024]{madden2020high}
Madden, L., Dall'Anese, E., and Becker, S. (2024).
\newblock High probability convergence bounds for non-convex stochastic
  gradient descent with sub-weibull noise.
\newblock {\em Journal of Machine Learning Research}, 25(241):1--36.

\bibitem[Moulines and Bach, 2011]{bach-sgd}
Moulines, E. and Bach, F. (2011).
\newblock Non-asymptotic analysis of stochastic approximation algorithms for
  machine learning.
\newblock In Shawe-Taylor, J., Zemel, R., Bartlett, P., Pereira, F., and
  Weinberger, K., editors, {\em Advances in Neural Information Processing
  Systems}, volume~24. Curran Associates, Inc.

\bibitem[Nair et~al., 2022]{heavy-tail-book}
Nair, J., Wierman, A., and Zwart, B. (2022).
\newblock {\em The Fundamentals of Heavy Tails: Properties, Emergence, and
  Estimation}.
\newblock Cambridge Series in Statistical and Probabilistic Mathematics.
  Cambridge University Press.

\bibitem[Nemirovski et~al., 2009]{nemirovski2009robust}
Nemirovski, A., Juditsky, A., Lan, G., and Shapiro, A. (2009).
\newblock Robust stochastic approximation approach to stochastic programming.
\newblock {\em SIAM Journal on optimization}, 19(4):1574--1609.

\bibitem[Nesterov, 2018]{nesterov-lectures_on_cvxopt}
Nesterov, Y. (2018).
\newblock {\em Lectures on Convex Optimization}.
\newblock Springer Publishing Company, Incorporated, 2nd edition.

\bibitem[Nguyen et~al., 2023a]{nguyen2023improved}
Nguyen, T.~D., Nguyen, T.~H., Ene, A., and Nguyen, H. (2023a).
\newblock Improved convergence in high probability of clipped gradient methods
  with heavy tailed noise.
\newblock In Oh, A., Neumann, T., Globerson, A., Saenko, K., Hardt, M., and
  Levine, S., editors, {\em Advances in Neural Information Processing Systems},
  volume~36, pages 24191--24222. Curran Associates, Inc.

\bibitem[Nguyen et~al., 2023b]{nguyen2023high}
Nguyen, T.~D., Nguyen, T.~H., Ene, A., and Nguyen, H.~L. (2023b).
\newblock High probability convergence of clipped-sgd under heavy-tailed noise.
\newblock {\em arXiv preprint arXiv:2302.05437}.

\bibitem[Parletta et~al., 2022]{parletta2022high}
Parletta, D.~A., Paudice, A., Pontil, M., and Salzo, S. (2022).
\newblock High probability bounds for stochastic subgradient schemes with heavy
  tailed noise.
\newblock {\em arXiv preprint arXiv:2208.08567}.

\bibitem[Peluchetti et~al., 2020]{pmlr-v108-peluchetti20b}
Peluchetti, S., Favaro, S., and Fortini, S. (2020).
\newblock Stable behaviour of infinitely wide deep neural networks.
\newblock In Chiappa, S. and Calandra, R., editors, {\em Proceedings of the
  Twenty Third International Conference on Artificial Intelligence and
  Statistics}, volume 108 of {\em Proceedings of Machine Learning Research},
  pages 1137--1146. PMLR.

\bibitem[Polyak, 1990]{polyak}
Polyak, B. (1990).
\newblock New stochastic approximation type procedures.
\newblock {\em Avtomatica i Telemekhanika}, 7:98--107.

\bibitem[Polyak and Juditsky, 1992]{polyak-ruppert}
Polyak, B. and Juditsky, A. (1992).
\newblock {Acceleration of Stochastic Approximation by Averaging}.
\newblock {\em SIAM Journal on Control and Optimization}, 30:838--855.

\bibitem[Polyak and Tsypkin, 1979]{polyak-adaptive-estimation}
Polyak, B. and Tsypkin, Y. (1979).
\newblock Adaptive estimation algorithms: Convergence, optimality, stability.
\newblock {\em Automation and Remote Control}, 1979.

\bibitem[Puchkin et~al., 2023]{gorbunov2023breaking}
Puchkin, N., Gorbunov, E., Kutuzov, N., and Gasnikov, A. (2023).
\newblock Breaking the heavy-tailed noise barrier in stochastic optimization
  problems.
\newblock {\em arXiv preprint arXiv:2311.04161}.

\bibitem[Rakhlin et~al., 2012]{rakhlin2012making}
Rakhlin, A., Shamir, O., and Sridharan, K. (2012).
\newblock Making gradient descent optimal for strongly convex stochastic
  optimization.
\newblock In {\em Proceedings of the 29th International Coference on
  International Conference on Machine Learning}, pages 1571--1578.

\bibitem[Robbins and Monro, 1951]{robbins1951stochastic}
Robbins, H. and Monro, S. (1951).
\newblock A stochastic approximation method.
\newblock {\em The annals of mathematical statistics}, pages 400--407.

\bibitem[Ruppert, 1988]{ruppert}
Ruppert, D. (1988).
\newblock {Efficient Estimations from a Slowly Convergent Robbins-Monro
  Process}.
\newblock Technical Report 781, Cornell University Operations Research and
  Industrial Engineering.

\bibitem[Sadiev et~al., 2023]{sadiev2023highprobability}
Sadiev, A., Danilova, M., Gorbunov, E., Horváth, S., Gidel, G., Dvurechensky,
  P., Gasnikov, A., and Richtárik, P. (2023).
\newblock High-probability bounds for stochastic optimization and variational
  inequalities: the case of unbounded variance.
\newblock In {\em International Conference on Machine Learning}, pages
  29563--29648. PMLR.

\bibitem[Simsekli et~al., 2019a]{csimcsekli2019heavy}
Simsekli, U., Gurbuzbalaban, M., Nguyen, T.~H., Richard, G., and Sagun, L.
  (2019a).
\newblock On the heavy-tailed theory of stochastic gradient descent for deep
  neural networks.
\newblock {\em arXiv preprint arXiv:1912.00018}.

\bibitem[Simsekli et~al., 2019b]{simsekli2019tail}
Simsekli, U., Sagun, L., and Gurbuzbalaban, M. (2019b).
\newblock A tail-index analysis of stochastic gradient noise in deep neural
  networks.
\newblock In {\em International Conference on Machine Learning}, pages
  5827--5837. PMLR.

\bibitem[Tsai et~al., 2022]{pmlr-v151-tsai22a}
Tsai, C.-P., Prasad, A., Balakrishnan, S., and Ravikumar, P. (2022).
\newblock Heavy-tailed streaming statistical estimation.
\newblock In Camps-Valls, G., Ruiz, F. J.~R., and Valera, I., editors, {\em
  Proceedings of The 25th International Conference on Artificial Intelligence
  and Statistics}, volume 151 of {\em Proceedings of Machine Learning
  Research}, pages 1251--1282. PMLR.

\bibitem[Vershynin, 2018]{vershynin_2018}
Vershynin, R. (2018).
\newblock {\em High-Dimensional Probability: An Introduction with Applications
  in Data Science}.
\newblock Cambridge Series in Statistical and Probabilistic Mathematics.
  Cambridge University Press.

\bibitem[Wright and Recht, 2022]{Wright_Recht_2022}
Wright, S.~J. and Recht, B. (2022).
\newblock {\em Optimization for Data Analysis}.
\newblock Cambridge University Press.

\bibitem[Yang et~al., 2022]{yang2022normalized_sgd}
Yang, X., Zhang, H., Chen, W., and Liu, T.-Y. (2022).
\newblock Normalized/clipped sgd with perturbation for differentially private
  non-convex optimization.
\newblock {\em arXiv preprint arXiv:2206.13033}.

\bibitem[You et~al., 2019]{you2019reducing}
You, Y., Li, J., Hseu, J., Song, X., Demmel, J., and Hsieh, C.-J. (2019).
\newblock Reducing bert pre-training time from 3 days to 76 minutes.
\newblock {\em arXiv preprint arXiv:1904.00962}, 12.

\bibitem[Yu and Kar, 2023]{shuhua-clipping}
Yu, S. and Kar, S. (2023).
\newblock Secure distributed optimization under gradient attacks.
\newblock {\em IEEE Transactions on Signal Processing}, 71:1802--1816.

\bibitem[Zeiler, 2012]{zeiler2012adadelta}
Zeiler, M.~D. (2012).
\newblock Adadelta: an adaptive learning rate method.
\newblock {\em arXiv preprint arXiv:1212.5701}.

\bibitem[Zhang et~al., 2019]{zhang2019gradient}
Zhang, J., He, T., Sra, S., and Jadbabaie, A. (2019).
\newblock Why gradient clipping accelerates training: A theoretical
  justification for adaptivity.
\newblock In {\em International Conference on Learning Representations}.

\bibitem[Zhang et~al., 2020]{zhang2020adaptive}
Zhang, J., Karimireddy, S.~P., Veit, A., Kim, S., Reddi, S., Kumar, S., and
  Sra, S. (2020).
\newblock Why are adaptive methods good for attention models?
\newblock {\em Advances in Neural Information Processing Systems},
  33:15383--15393.

\bibitem[Zhang et~al., 2022]{zhang2022clip_FL_icml}
Zhang, X., Chen, X., Hong, M., Wu, Z.~S., and Yi, J. (2022).
\newblock Understanding clipping for federated learning: Convergence and
  client-level differential privacy.
\newblock In {\em International Conference on Machine Learning, ICML 2022}.

\end{thebibliography}

\appendix

\section{Introduction}

Appendix contains results omitted from the main body. Section \ref{app:facts} provides some useful facts and results used in the proofs, Section \ref{app:proofs} provides the proofs omitted from the main body, Section \ref{app:rate} specializes the rate exponent $\zeta$ from Theorem \ref{theorem:main} for component-wise and joint nonlinearities, Section \ref{app:derivations} details the derivations for Examples \ref{example:5}-\ref{example:7}, Section \ref{app:analytical} provides an analytical example for which our rates predict that joint clipping is not the optimal choice of nonlinearity, Section \ref{app:num} provides additional experiments, Section \ref{app:noise} provides a detailed discussion on the noise assumption used in our work, while Section \ref{app:metric} provides a discussion on the metric used in our work.

\section{Useful Facts}\label{app:facts}

In this section we present some useful facts and results, concerning $L$-smooth, $\mu$-strongly convex functions, bounded random vectors and the behaviour of nonlinearities.

\begin{fact}
    Let $f: \R^d \mapsto \R$ be $L$-smooth, $\mu$-strongly convex, and let $\bx^\star = \argmin_{\bx \in \R^d}f(\bx)$. Then, for any $\bx \in \R^d$, we have
    \begin{equation*}
        2\mu\left(f(\bx) - f(\bx^\star)\right)\leq \|\nabla f(\bx)\|^2 \leq 2L\left(f(\bx) - f(\bx^\star)\right).
    \end{equation*}
\end{fact}
\begin{proof}
    The proof of the upper bound follows by plugging $y = \bx$, $x = \bx^\star$ in equation (2.1.10) of Theorem~2.1.5 from~\citet{nesterov-lectures_on_cvxopt}. The proof of the lower bound similarly follows by plugging $y = \bx$, $x = \bx^\star$ in equation (2.1.24) of Theorem~2.1.10 from~\citet{nesterov-lectures_on_cvxopt}.
\end{proof}

\begin{fact}\label{fact:subgauss}
    Let $X \in \R^d$ be a zero-mean, bounded random vector, i.e., $\E X = 0$ and $\|X\| \leq \sigma$, for some $\sigma > 0$. Then, $X$ is $\sigma\sqrt{2}$-sub-Gaussian, i.e., for any $v \in \R^d$, we have
    \begin{equation*}
        \E e^{\langle X,v \rangle} \leq e^{\sigma^2\|v\|^2}.
    \end{equation*}
\end{fact}
\begin{proof}
    The proof follows a similar idea to the one of proving sub-Gaussian properties in, e.g., \cite{vershynin_2018}. Using the general inequality $e^x \leq x + e^{x^2}$, which holds for any $x \in \R$, setting $x = \langle X, v \rangle$, we get
    \begin{equation*}
        \E\left[\exp\left(\langle X,v \rangle\right)\right] \leq \E\left[\exp\left(\langle X,v \rangle^2\right)\right] \leq e^{\sigma^2 \|v\|^2}, 
    \end{equation*} where the first inequality follows from the fact that $X$ is zero mean, while the second follows from the Cauchy-Schwartz inequality and $\|X\| \leq \sigma$. 
\end{proof}

\section{Missing Proofs}\label{app:proofs}

In this section we provide proofs omitted from the main body. Subsection \ref{sec:proof-thm1} proves results pertinent to Theorem \ref{thm:non-conv}, Subsection \ref{sec:proof-thm2} proves results relating to Theorem \ref{theorem:main}, while Subsection \ref{sec:proof-thm3} proves Theorem \ref{thm:non-sym}.

\subsection{Proof of Theorem~\ref{thm:non-conv}}\label{sec:proof-thm1}

In this section we prove Lemmas \ref{lm:error_component}, \ref{lm:huber}, Theorem \ref{thm:non-conv} and Corollary \ref{cor:cvx}. We begin by proving Lemma~\ref{lm:error_component}.

\begin{proof}[Proof of Lemma~\ref{lm:error_component}]
    Recall the definition of the error vector $\bet \triangleq \boldsymbol{\Phi}^{(t)} - \boldsymbol{\Psi}^{(t)}$, where $\bPhi^{(t)} \triangleq \mbe_\bzt \left[\bPsi (\nabla f(\bxt)+\bzt) \: \vert \: \mathcal{F}_t \right]$ is the denoised version of $\bPsi^{(t)}$. By definition, it then follows that
    \begin{equation*}
        \E\left[\bet \vert \: \mathcal{F}_t\right] = \E\left[\bPhi^{(t)} - \bPsi^{(t)} \vert \: \mathcal{F}_t\right] = \bPhi^{(t)} - \E\left[ \bPsi^{(t)} \vert \: \mathcal{F}_t \right] = 0.
    \end{equation*} Moreover, by Assumption~\ref{asmpt:nonlin}, we have $\|\bet \| = \|\bPhi^{(t)} - \bPsi^{(t)}\| \leq \|\bPhi^{(t)}\| + \|\bPsi^{(t)}\| \leq \E\|\bPsi^{(t)}\| +  C \leq 2C$, which proves the first claim. The second claim readily follows by using the fact that $\bet$ is a bounded random variable and applying Fact~\ref{fact:subgauss}.
\end{proof}

Prior to proving Lemma \ref{lm:huber}, we state two results used in the proof. The first result, due to \cite{polyak-adaptive-estimation}, provides some properties of the mapping $\bPhi$ for component-wise nonlinearities under symmetric noise.

\begin{lemma}\label{lm:polyak-tsypkin}
    Let Assumptions~\ref{asmpt:nonlin} and~\ref{asmpt:noise} hold, with the nonlinearity $\bPsi: \R^d \mapsto \R^d$ being component wise, i.e., of the form $\bPsi(\bx) = \begin{bmatrix} \calN_1(x_1),\ldots,\calN_1(x_d)\end{bmatrix}^\top$. Then, the function $\bPhi: \R^d \mapsto \R^d$ is of the form $\bPhi(\bx) = \begin{bmatrix} \phi_1(x_1),\ldots,\phi_d(x_d) \end{bmatrix}^\top$, where $\phi_i(x_i) = \E_{z_i}\left[\calN_1(x_i + z_i)\right]$ is the marginal expectation of the $i$-th noise component, $i \in [d]$, with the following properties: 
    \begin{enumerate}
        \item $\phi_i$ is non-decreasing and odd, with $\phi_i(0) = 0$;
        
        \item $\phi_i$ is differentiable in zero, with $\phi_i^\prime(0) > 0$.
    \end{enumerate}
\end{lemma}

The second result, due to \cite{jakovetic2023nonlinear}, gives a useful property of $\bPhi$ for joint nonlinearities.

\begin{lemma}\label{lm:jakovetic-joint}
    Let Assumption \ref{asmpt:nonlin} hold, with the nonlinearity $\bPsi: \R^d \mapsto \R^d$ being joint, i.e., of the form $\bPsi(\bx) = \bx\calN_2(\|\bx\|)$. Then for any $\bx, \bz \in \R^d$ such that $\|\bz\| > \|\bx\|$
    \begin{equation*}
        \left|\calN_2(\|\bx + \bz\|) - \calN_2(\|\bx - \bz\|)\right| \leq \nicefrac{\|\bx\|}{\|\bz\|}\left[\calN_2(\|\bx + \bz\|) + \calN_2(\|\bx - \bz\|) \right].
    \end{equation*} 
\end{lemma}

Define $\phi^\prime(0) \triangleq \min_{i \in [d]}\phi_i^\prime(0)$ and $p_0 \triangleq P(\mathbf{0})$. We are now ready to prove Lemma \ref{lm:huber}. For convenience, we restate the full lemma below.

\begin{lemma}\label{lm:huber-app}
    Let Assumptions \ref{asmpt:nonlin} and \ref{asmpt:noise} hold. Then, for any $\bx \in \R^d$, we have $\langle \bPhi(\bx),\bx\rangle \geq \min\left\{\eta_1\|\bx\|,\eta_2\|\bx\|^2 \right\}$, where $\eta_1,\eta_2 > 0$, are noise, nonlinearity and problem dependent constants. In particular, if the nonlinearity $\bPsi$ is component-wise, we have $\eta_1 = \nicefrac{\phi^\prime(0)\xi}{2\sqrt{d}}$ and $\eta_2 = \nicefrac{\phi^\prime(0)}{2d}$, where $\xi > 0$ is a constant that depends only on the noise and choice of nonlinearity. If $\bPsi$ is a joint nonlinearity, then $\eta_1 = p_0\calN_2(1)/2$ and $\eta_2 = p_0\calN_2(1)$.
\end{lemma}

\begin{proof}[Proof of Lemma~\ref{lm:huber-app}]
    First, consider the case when $\bPhi(\bx) = [\calN_1(x_1), \ldots, \calN_15(x_d)]^\top$ is component-wise. From Lemma \ref{lm:polyak-tsypkin} it follows that, for any $x \in \R$, and any $i \in [d]$, we have
    \begin{equation*}
        \phi_i(x) = \phi_i(0) + \phi_i^\prime(0)x + h_i(x)x = \phi_i^\prime(0)x + h_i(x)x,
    \end{equation*} where $h_i: \R \mapsto \R$ is such that $\lim_{x \rightarrow 0}h_i(x) = 0$. Recalling that $\phi^\prime(0) = \min_{i \in [d]}\phi_i^\prime(0) > 0$, it follows that there exists a $\xi > 0$ (depending only on the nonlinearity $\mathcal{N}_1$) such that, for each $x \in \R$ and all $i \in [d]$, we have $|h_i(x)| \leq \phi^\prime(0) / 2$, if $|x| \leq \xi$. Therefore, for any $0 \leq x \leq \xi$, we have $\phi_i(x) \geq \frac{\phi^\prime(0)x}{2}$. On the other hand, for $x > \xi$, since $\phi_i$ is non-decreasing, we have from the previous relation that $\phi_i(x) \geq \phi_i(\xi) \geq \frac{\phi^\prime(0)\xi}{2}$. Therefore, it follows that $\phi_i(x) \geq \frac{\phi^\prime(0)}{2}\min\{x,\xi \}$, for any $x \geq 0$. Combined with the oddity of $\phi_i$, we get $x\phi_i(x) = |x|\phi_i(|x|) \geq \frac{\phi^\prime(0)}{2}\min\{\xi|x|,x^2\}$, for any $x \in \R$. Using the previously established relations, we then have, for any vector $\bx \in \R^d$
    \begin{align*}
        \langle \bx, \bPhi(\bx) \rangle &= \sum_{i = 1}^d x_i\phi_i(x_i)  = \sum_{i = 1}^d |x_i|\phi(|x_i|) \geq \max_{i \in [d]}|x_i|\phi_i(|x_i|) \geq \frac{\phi^\prime(0)}{2}\max_{i \in [d]}\min\{\xi|x_i|,|x_i|^2\} \nn \\ &= \frac{\phi^\prime(0)}{2}\min\{\xi\|\bx\|_{\infty},\|\bx\|_{\infty}^2\} \geq \frac{\phi^\prime(0)}{2}\min\{\nicefrac{\xi\|\bx\|}{ \sqrt{d}},\nicefrac{\|\bx\|^2}{d}\} ,
    \end{align*} where the last inequality follows from the fact that $\|\bx\|_{\infty} \geq \|\bx\| / \sqrt{d}$. Next, consider the case when $\bPhi(\bx) = \bx\calN_2(\|\bx\|)$ is joint. The proof follows a similar idea to the one in \cite[Lemma 6.2]{jakovetic2023nonlinear}, with some important differences due to the different noise assumption. Fix an arbitrary $\bx \in \R^d \setminus \{\mathbf{0}\}$. By the definition of $\bPsi$, we have
    \begin{align*}
        \langle \bPhi(\bx),\bx \rangle = \int_{\bz \in \R^d}\underbrace{(\bx + \bz)^\top \bx \calN_2(\|\bx + \bz\|)}_{\triangleq M(\bx,\bz)}P(\bz)d\bz = \int_{\{\bz \in \R^d: \langle\bz,\bx\rangle \geq 0\} \cup \{\bz \in \R^d: \langle\bz,\bx\rangle < 0\}}\hspace{-10em}M(\bx,\bz)P(\bz)d\bz. 
    \end{align*} Next, by symmetry of $P$, it readily follows that $\langle \bPhi(\bx),\bx \rangle = \int_{J_1(\bx)}M_2(\bx,\bz)P(\bz)d\bz,$ where $J_1(\bx) \triangleq \{\bz \in \R^d: \langle \bz,\bx\rangle \geq 0 \}$ and $M_2(\bx,\bz) \triangleq (\|\bx\|^2 + \langle \bz,\bx\rangle)\calN_2(\|\bx + \bz\|) + (\|\bx\|^2 - \langle \bz,\bx\rangle)\calN_2(\|\bx - \bz\|)$. Consider the set $J_2(\bx) \triangleq \left\{\bz \in \R^d: \frac{\langle \bz,\bx\rangle}{\|\bz\|\|\bx\|} \in [0,0.5] \right\} \cup \{\mathbf{0}\}$. Clearly $J_2(\bx) \subset J_1(\bx)$. Note that on $J_1(\bx)$ we have $\|\bx + \bz \| \geq \|\bx - \bz\|$, which, together with the fact that $\calN_2$ is non-increasing, implies
    \begin{equation}\label{eq:identity}
        \calN_2(\|\bx - \bz\|) - \calN_2(\|\bx + \bz\|) = \left|\calN_2(\|\bx - \bz\|) - \calN_2(\|\bx + \bz\|) \right|, 
    \end{equation} for any $\bz \in J_1(\bx)$. For any $\bz \in J_2(\bx)$ such that $\|\bz\| > \|\bx\|$, we then have
    \begin{align*}
        M_2&(\bx,\bz) = \|\bx\|^2[\calN_2(\|\bx - \bz\|) + \calN_2(\|\bx + \bz\|)] - \langle \bz,\bx\rangle[\calN_2(\|\bx - \bz\|) - \calN_2(\|\bx + \bz\|)] \\ &\stackrel{(a)}{=} \|\bx\|^2[\calN_2(\|\bx - \bz\|) + \calN_2(\|\bx + \bz\|)] - \langle \bz,\bx\rangle\left|\calN_2(\|\bx - \bz\|) - \calN_2(\|\bx + \bz\|)\right| \\ &\stackrel{(b)}{\geq} \|\bx\|^2[\calN_2(\|\bx - \bz\|) + \calN_2(\|\bx + \bz\|)] - \langle \bz,\bx\rangle\nicefrac{\|\bx\|}{\|\bz\|}[\calN_2(\|\bx - \bz\|) + \calN_2(\|\bx + \bz\|)] \\ &\stackrel{(c)}{\geq} 0.5\|\bx\|^2[\calN_2(\|\bx - \bz\|) + \calN_2(\|\bx + \bz\|)],
    \end{align*} where $(a)$ follows from \eqref{eq:identity}, $(b)$ follows from Lemma \ref{lm:jakovetic-joint}, while $(c)$ follows from the definition of $J_2(\bx)$. Next, consider any $\bz \in J_2(\bx)$, such that $0 < \|\bz\| \leq \|\bx\|$. We have
    \begin{align*}
        M_2&(\bx,\bz) = \|\bx\|^2[\calN_2(\|\bx - \bz\|) + \calN_2(\|\bx + \bz\|)] - \langle \bz,\bx\rangle[\calN_2(\|\bx - \bz\|) - \calN_2(\|\bx + \bz\|)] \\ &\stackrel{(a)}{=} \|\bx\|^2[\calN_2(\|\bx - \bz\|) + \calN_2(\|\bx + \bz\|)] - \langle \bz,\bx\rangle\left|\calN_2(\|\bx - \bz\|) - \calN_2(\|\bx + \bz\|)\right| \\ &\stackrel{(b)}{\geq} \|\bx\|^2[\calN_2(\|\bx - \bz\|) + \calN_2(\|\bx + \bz\|)] - 0.5\|\bx\|^2\left|\calN_2(\|\bx - \bz\|) - \calN_2(\|\bx + \bz\|)\right| \\ &\stackrel{(c)}{\geq} 0.5\|\bx\|^2[\calN_2(\|\bx - \bz\|) + \calN_2(\|\bx + \bz\|)],
    \end{align*} where $(a)$ again follows from \eqref{eq:identity}, $(b)$ follows from the definition of $J_2(\bx)$ and the fact that $0 < \|\bz\| \leq \|\bx\|$, while $(c)$ follows from $\calN_2$ being non-negative and $\left|\calN_2(\|\bx - \bz\|) - \calN_2(\|\bx + \bz\|)\right| \leq \calN_2(\|\bx - \bz\|) + \calN_2(\|\bx + \bz\|)$. Finally, if $\bz = \mathbf{0}$, we have $M_2(\bx,\mathbf{0}) = 2\|\bx\|^2\calN_2(\|\bx\|) > 0.5\|\bx\|^2[\calN_2(\|\bx + \mathbf{0}\|) + \calN_2(\|\bx - \mathbf{0}\|)]$. Therefore, for any $\bz \in J_2(\bx)$, we have $M_2(\bx,\bz) \geq 0.5\|\bx\|^2[\calN_2(\|\bx - \bz\|) + \calN_2(\|\bx + \bz\|)] \geq \|\bx\|^2\calN_2(\|\bx\| + \|\bz\|)$, where the second inequality follows from the fact that $\calN_2$ is non-increasing and $\|\bx \pm \bz\| \leq \|\bx\| + \|\bz\|$. Note that following a similar argument as above, it can be shown that $M_2(\bx,\bz) \geq 0$, for any $\bx \in \R^d$ and $\bz \in J_1(\bx)$. Combining everything, it readily follows that
    \begin{align}\label{eq:semi-done}
        \langle \bPhi(\bx),\bx \rangle \geq \int_{J_2(\bx)}M_2(\bx,\bz)P(\bz)d\bz \geq \|\bx\|^2 \int_{J_2(\bx)}\calN_2(\|\bx\| + \|\bz\|)P(\bz)d\bz,
    \end{align} where the first inequality follows from the fact that $M_2(\bx,\bz) \geq 0$ on $J_1(\bx)$. Define $C_0 \triangleq \min\left\{B_0,0.5 \right\}$ and consider the set $J_3(\bx) \subset J_2(\bx)$, defined as 
    \begin{equation*}
        J_3(\bx) \triangleq \left\{\bz \in \R^d: \frac{\langle \bz,\bx\rangle}{\|\bz\|\|\bx\|} \in [0,0.5], \: \|\bz\| \leq C_0 \right\} \cup \{\mathbf{0}\}.
    \end{equation*} Since $a\calN_2(a)$ is non-decreasing, it follows that $\calN_2(a) \geq \calN_2(1)\min\left\{a^{-1},1 \right\}$, for any $a > 0$. For any $\bz \in J_3(\bx)$, it then holds that $\calN_2(\|\bz\| + \|\bx\|) \geq \calN_2(1)\min\left\{1/(\|\bx\|+C_0),1 \right\}$. Plugging in \eqref{eq:semi-done}, we then have
    \begin{align}
        &\langle \bPhi(\bx),\bx \rangle \geq \|\bx\|^2 \int_{J_3(\bx)}\calN_2(\|\bx\| + \|\bz\|)P(\bz)d\bz \nonumber \\ &\geq \|\bx\|^2\calN_2(1)\min\left\{(\|\bx\| + C_0)^{-1},1 \right\}\hspace{-0.45em}\int_{J_3(\bx)}\hspace{-1.5em}P(\bz)d\bz \geq \|\bx\|^2\calN_2(1)\min\left\{(\|\bx\| + C_0)^{-1},1 \right\}p_0 \label{eq:intermed}.
    \end{align} If $\|\bx\| \leq C_0$, it follows that $\|\bx\| + C_0 \leq 2C_0$, therefore $\min\left\{1/(\|\bx\|+C_0),1\right\} \geq \min\left\{1/(2C_0),1 \right\}.$ Define $\kappa \triangleq \min\left\{1/(2C_0),1 \right\}$. If $\|\bx\| \geq C_0$, it follows that $\|\bx\| + C_0 \leq 2\|\bx\|$, therefore $\min\left\{1/(\|\bx\|+C_0),1 \right\} \geq \min\left\{1/(2\|\bx\|),1 \right\} \geq \min\left\{1/(2\|\bx\|),\kappa \right\}.$ Combining the observations, we get $\langle \bPhi(\bx),\bx \rangle \geq p_0\calN_2(1)\min\left\{\|\bx\|/2,\kappa\|\bx\|^2 \right\}$. Consider $\kappa = \min\left\{\nicefrac{1}{(2C_0)},1 \right\}$. If $B_0 \geq 0.5$, it follows that $C_0 = 0.5$ and therefore $\kappa = 1$. On the other hand, if $B_0 < 0.5$, it follows that $C_0 = B_0$ and therefore $\kappa = \min\left\{\nicefrac{1}{(2B_0)},1 \right\} = 1$, as $2B_0 < 1$.
\end{proof}

We are now ready to prove Theorem~\ref{thm:non-conv}.

\begin{proof}[Proof of Theorem~\ref{thm:non-conv}]
    For ease of notation, let $Z_t \triangleq \min\{\eta_1\|\nabla f(\bxt) \|,\eta_2\|\nabla f(\bxt)\|^2\}$. Applying the $L$-smoothness property of $f$ and the update rule~\eqref{eq:nonlin-sgd2}, to get
    \begin{align*}
        f(\bxtp) &\leq f(\bxt) - \alpha_t\langle \nabla f(\bxt),\bPhit - \bet \rangle + \frac{\alpha_t^2L}{2}\|\bPsi^{(t)}\|^2 \\ &\leq f(\bxt) - \alpha_tZ_t + \alpha_t\langle \nabla f(\bxt),\bet\rangle + \frac{\alpha_t^2LC^2}{2}, 
    \end{align*} where the second inequality follows from Lemma~\ref{lm:huber} and Assumption~\ref{asmpt:nonlin}. Rearranging and summing up the first $t$ terms, we get
    \begin{equation}\label{eq:5}
        \sum_{k = 1}^{t}\alpha_kZ_k \leq \underbrace{f(\bx^{(1)}) - f^\star + \frac{LC^2}{2}\sum_{k = 1}^{t}\alpha_k^2}_{\triangleq B_1} + \underbrace{\sum_{k = 1}^{t}\alpha_k\langle \nabla f(\bxk), \bek \rangle}_{\triangleq B_2}. 
    \end{equation} Denote the left-hand side of \eqref{eq:5} by $G_t$, i.e., $G_t \triangleq \sum_{k = 1}^{t}\alpha_kZ_k$ and note that $B_1$ is independent of the noise, i.e., is a deterministic quantity. We then have
    \begin{equation*}
        \E\left[\exp(G_t) \right] \stackrel{\eqref{eq:5}}{\leq} \E\left[\exp\left(B_1 + B_2\right) \right] = \exp\left(B_1 \right)\E\left[\exp(B_2)\right]. 
    \end{equation*} We now bound $\E[\exp(B_2)]$. Denote by $\E_t[\cdot] \triangleq \E[ \cdot \: \vert \: \mathcal{F}_t]$ the expectation conditioned on history up to time $t$. We then have
    \begin{align}
        \E[\exp(B_2)] &= \E\left[\exp\left(\sum_{k = 1}^{t}\alpha_k\langle \nabla f(\bxk), \bek \rangle\right)\right] \nonumber \\ &= \E\left[\exp\left(\sum_{k = 1}^{t-1}\alpha_k\langle \nabla f(\bxk), \bek \rangle \right)\E_{t}\left[\exp(\alpha_{t}\langle \nabla f(\bx^{(t)}),\be^{(t)} \rangle) \right] \right] \nonumber \\ &\leq \E\left[\exp\left(\sum_{k = 1}^{t-1}\alpha_k\langle \nabla f(\bxk), \bek \rangle \right) \exp\left(4C^2\alpha_{t}^2\|\nabla f(\bx^{(t)})\|^2 \right) \right], \label{eq:b2}
    \end{align} where the last inequality follows from Lemma~\ref{lm:error_component}. Next, consider $\|\nabla f(\bxk)\|$, for any $k \geq 0$. Define $A_t \triangleq \sum_{k = 1}^{t}\alpha_k$ and use $L$-smoothness, to get
    \begin{align}
        \|\nabla f(\bxk)\| &\leq L\|\bxk - \bx^\star\| = L\|\bx^{(k-1)} - \alpha_{k-1}\bPsi^{(k-1)} - \bx^\star\| \leq L\left(\|\bx^{(k-1)} - \bx^\star\| +  \alpha_{k-1}C\right) \nonumber \\ &\leq \ldots \leq L\left(\|\bx^{(1)} - \bx^\star\| + C\sum_{s = 1}^{k-1}\alpha_s\right) \leq L\left(\|\bx^{(1)} - \bx^\star\| + CA_k\right), \label{eq:grad-norm}
    \end{align} where we recall that $\bx^\star \in \mathcal{X}$ is any stationary point of $f$. Combining~\eqref{eq:b2} and~\eqref{eq:grad-norm}, we get
    \begin{align*}
        \E[\exp(B_2)] &\leq \exp\left(8C^2L^2D_{\mathcal{X}}\alpha_{t}^2 + 8C^4L^2\alpha_{t}^2A_{t}^2 \right)\E\left[\exp\left(\sum_{k = 1}^{t-1}\alpha_k\langle \nabla f(\bxk), \bek \rangle \right) \right],
    \end{align*} where $D_{\mathcal{X}} = \inf_{\bx^\star \in \mathcal{X}}\|\bx^{(1)} - \bx^\star\|^2$ is the distance of the initial model estimate from the set of stationary points. Repeating the same arguments recursively, we then get
    \begin{equation*}
        \E\left[\exp(B_2) \right] \leq \exp\left(8C^2L^2D_{\mathcal{X}}\sum_{k = 1}^{t}\alpha_k^2 + 8C^4L^2\sum_{k = 1}^{t}\alpha_k^2A_{k}^2 \right),
    \end{equation*}  Combining everything, we get
    \begin{equation*}
        \E\left[\exp(G_t) \right] \leq \exp\left(f(\bx^{(1)}) - f^\star + LC^2\left(\nicefrac{1}{2} + 8LD_{\mathcal{X}}\right)\sum_{k = 1}^{t}\alpha_k^2 + 8C^4L^2\sum_{k = 1}^{t}\alpha_k^2A_{k}^2 \right).
    \end{equation*} Define $N_t \triangleq f(\bx^{(1)}) - f^\star + LC^2\left(\nicefrac{1}{2} + 8LD_{\mathcal{X}}\right)\sum_{k = 1}^{t}\alpha_k^2 + 8C^4L^2\sum_{k = 1}^{t}\alpha_k^2A_{k}^2$. Using Markov's inequality, it then follows that, for any $\epsilon > 0$
    \begin{equation*}
        \mathbb{P}(G_t > \epsilon) \leq \exp(-\epsilon)\E[\exp(G_t)] \leq \exp(-\epsilon + N_t) \iff \mathbb{P}(G_t > \epsilon + N_t) \leq \exp(-\epsilon).
    \end{equation*} Finally, for any $\beta \in (0,1)$, with probability at least $1 - \beta$, we have
    \begin{equation}\label{eq:6}
        G_t \leq \log(\nicefrac{1}{\beta}) + N_t \iff A_t^{-1}G_t \leq A_t^{-1}\left(\log(\nicefrac{1}{\beta}) + N_t \right).
    \end{equation} Note that for the step-size schedule $\alpha_t = \frac{a}{(t + 1)^\delta}$ and any $\delta \in (\nicefrac{2}{3},1)$, using lower and upper Darboux sums, we have
    \begin{equation}\label{eq:Darboux}
    \begin{aligned}
        \frac{a}{1-\delta}((t+2)^{1-\delta} - 2^{1-\delta})&\leq A_t \leq \frac{a}{1-\delta}((t+1)^{1-\delta} - 1), \\
        \frac{a^2}{2\delta - 1}(2^{1-2\delta} - (t+2)^{1-2\delta})&\leq \sum_{k = 1}^{t}\alpha_k^2 \leq \frac{a^2}{2\delta - 1}(1 - (t+1)^{1-2\delta}).
    \end{aligned}
    \end{equation} Plugging \eqref{eq:Darboux} in \eqref{eq:6}, we then get, with probability at least $1 - \beta$
    \begin{align}
        \sum_{k = 1}^{t}&\widetilde{\alpha}_kZ_k \leq \frac{(1-\delta)\left(f(\bx^{(1)}) - f^\star + \log(\nicefrac{1}{\beta})\right)}{a((t+2)^{1-\delta}-2^{1-\delta})} \nonumber \\ &+ \frac{a(1-\delta)LC^2(\nicefrac{1}{2} + 8LD_{\mathcal{X}})}{(2\delta-1)((t+2)^{1-\delta} - 2^{1-\delta})} + \frac{8a^3C^4L^2\sum_{k = 1}^{t}(k+1)^{2-4\delta}}{(1-\delta)((t+2)^{1-\delta} - 2^{1 - \delta})}. \label{eq:semi-final}
    \end{align} To bound the last sum, we consider different step-size schedules.
    \begin{enumerate}
        \item First, consider $\alpha_t = \frac{a}{(t + 1)^\delta}$, for $\delta \in (\nicefrac{2}{3},\nicefrac{3}{4})$. Using the lower Darboux sum, we have
    \begin{equation*}
        \sum_{k = 1}^{t}(k+1)^{2-4\delta} \leq \int_{1}^{t+1}k^{2-4\delta}dk \leq \frac{(t+1)^{3-4\delta}}{3 - 4\delta}.
    \end{equation*} Combining with \eqref{eq:semi-final}, we get
    \begin{equation}\label{eq:bound-st1}
        \sum_{k = 1}^{t}\widetilde{\alpha}_kZ_k \leq \frac{R_1}{(t+2)^{1-\delta} - 2^{1-\delta}} + \frac{R_2(t+1)^{3-4\delta}}{(t+2)^{1-\delta} - 2^{1-\delta}} \leq \frac{R_1}{(t+2)^{1-\delta} - 2^{1-\delta}} + \frac{R_2}{(t+2)^{3\delta-2} - 2^{3\delta-2}},
    \end{equation} where $R_1 \triangleq (1-\delta)\left[\frac{\left(f(\bx^{(1)}) - f^\star + \log(\nicefrac{1}{\beta})\right)}{a} + \frac{aLC^2(\nicefrac{1}{2} + 8LD_{\mathcal{X}})}{(2\delta-1)}\right]$ and $R_2 \triangleq \frac{8a^3C^4L^2}{(1-\delta)(3-4\delta)}$.

    \item Next, consider $\alpha_t = \frac{a}{(t + 1)^\delta}$, for $\delta = \nicefrac{3}{4}$. Using the lower Darboux sum, we have
    \begin{equation*}
        \sum_{k = 1}^{t}(k+1)^{2-4\delta} = \sum_{k = 1}^{t}\frac{1}{(k+1)} \leq \int_{1}^{t+1}\frac{1}{k}dk \leq \log(t+1).
    \end{equation*} Combining with \eqref{eq:semi-final}, we get
    \begin{equation}\label{eq:bound-st2}
        \sum_{k = 1}^{t}\widetilde{\alpha}_kZ_k \leq \frac{R_1  + R_3\log(t+1)}{(t+2)^{\nicefrac{1}{4}} - 2^{\nicefrac{1}{4}}},
    \end{equation} where $R_3 \triangleq 32a^3C^4L^2$.

    \item Finally, for $\alpha_t = \frac{a}{(t + 1)^\delta}$, where $\delta \in (\nicefrac{3}{4},1)$, we have
    \begin{equation*}
        \sum_{k = 1}^{t}(k+1)^{2-4\delta} \leq \int_{1}^{t+1}k^{2-4\delta}dk \leq \frac{1}{4\delta-3},
    \end{equation*} therefore, combining with \eqref{eq:semi-final}, we get
    \begin{equation}\label{eq:bound-st3}
        \sum_{k = 1}^{t}\widetilde{\alpha}_kZ_k \leq \frac{R_1 + R_4}{(t+2)^{1-\delta} - 2^{1-\delta}},
    \end{equation} where $R_4 \triangleq \frac{8a^3C^4L^2}{(1-\delta)(4\delta-3)}$.
    \end{enumerate} To obtain a bound on the quantity of interest $\min_{k \in [t]}\|\nabla f(\bxk)\|^2$, we proceed as follows. Notice that the bounds in \eqref{eq:bound-st1}-\eqref{eq:bound-st3} can be represented in a unified manner as
    \begin{equation}\label{eq:bound}
        \sum_{k = 1}^t\widetilde{\alpha}_kZ_k \leq Mt^{-\kappa},
    \end{equation} for appropriately selected constants $M,\kappa > 0$.\footnote{Note that for $\delta = \nicefrac{3}{4}$ we might have an additional factor of $\log(t)$ in the right-hand side of \eqref{eq:bound}. However, this can be easily incorporated, by allowing $M$ to depend on $t$, e.g., by defining $M_t = M\log(t)$.} Next, define $U \triangleq \{k \in [t]: \: \|\nabla f(\bxk)\| \leq \eta_1/\eta_2 \}$, with $U^c \triangleq [t] \setminus U$. From~\eqref{eq:bound}, we then have 
    \begin{align*}
        \sum_{k \in U^c}\widetilde{\alpha}_k\|\nabla f(\bxk)\| \leq M_1t^{-\kappa} \text{ and } \sum_{k \in U}\widetilde{\alpha}_k\|\nabla f(\bxk)\|^2 \leq M_2t^{-\kappa},
    \end{align*} where $M_1 = M/\eta_1$, $M_2 = M/\eta_2$. It then readily follows that
    \begin{align*}
        \min_{k \in [t]}\|\nabla f(\bxk)\| \leq \sum_{k \in U}\widetilde{\alpha}_k\|\nabla f(\bxk)\| + \sum_{k \in U^c}\widetilde{\alpha}_k\|\nabla f(\bxk)\| \leq \sum_{k = 1}^{t}\widetilde{\alpha}_kz_k + M_1t^{-\kappa},
    \end{align*} where $z_k = \|\nabla f(\bxk)\|$, for $k \in U$, otherwise $z_k = 0$. Using Jensen's inequality, we get
    \begin{align*}
        \min_{k \in [t]}\|\nabla f(\bxk)\| \leq \sqrt{\sum_{k = 1}^{t}\widetilde{\alpha}_kz_k^2} + M_1t^{-\kappa} = \sqrt{\sum_{k \in U}\widetilde{\alpha}_k\|\nabla f(\bxk)\|^2} + M_1t^{-\kappa} \leq \sqrt{M_2t^{-\kappa}} + M_1t^{-\kappa}.
    \end{align*} Squaring both sides and using $(a+b)^2 \leq 2a^2 + 2b^2$, gives the desired result. 
\end{proof}

We next prove Corollary \ref{cor:cvx}.

\begin{proof}[Proof of Corollary \ref{cor:cvx}]
    Recall the definition of the Huber loss function $H_{\lambda}: \R \mapsto [0,\infty)$, parametrized by $\lambda > 0$, e.g., \citet{huber_loss}, given by 
    \begin{equation*}
        H_{\lambda}(x) \triangleq \begin{cases}
            \frac{1}{2}x^2, & |x| \leq \lambda, \\
            \lambda|x| - \frac{\lambda^2}{2}, & |x| > \lambda.
        \end{cases}
    \end{equation*} By the definition of Huber loss, it is not hard to see that it is a convex, non-decreasing function on $[0,\infty)$. Moreover, by the definition of Huber loss, we have, for any $k \geq 1$
    \begin{equation}\label{eq:7}
    \begin{aligned}
        Z_k = \min\{\eta_1\|\nabla f(\bxk)\|,\eta_2\|\nabla f(\bxk)\|^2 \} \geq \eta_2H_{\eta_1/\eta_2}(\|\nabla f(\bxk)\|).
    \end{aligned}
    \end{equation} Next, recall that Assumption~\ref{asmpt:cvx} implies the \emph{gradient domination property}, i.e., $\|\nabla f(\bx)\|^2 \geq 2\mu(f(\bx) - f^\star)$, for any $\bx \in \R^d$, see, e.g., \citet{nesterov-lectures_on_cvxopt}. Combined with the definition of strong convexity, we have $\|\nabla f(\bx) \| \geq \mu\|\bx - \bx^\star\|$, for any $\bx \in \R^d$. Combining \eqref{eq:7} with the gradient domination property, we get 
    \begin{equation*}
        \sum_{k = 1}^{t}\widetilde{\alpha}_k Z_k \geq \eta_2\sum_{k = 1}^{t}\widetilde{\alpha}_kH_{\eta_1/\eta_2}(\mu\|\bxk - \bx^\star\|) \geq \mu^2\eta_2H_{\eta_1/(\eta_2\mu)}(\|\widehat{\bx}^{(t)} - \bx^\star\|),
    \end{equation*} where $\widehat{\bx}^{(t)} \triangleq \sum_{k = 1}^{t}\widetilde{\alpha}_k\bxk$ is the weighted average of the first $t$ iterates, the first inequality follows from~\eqref{eq:7}, the gradient domination property and the fact that $H$ is non-decreasing, while the second inequality follows from the fact that $H$ is convex and non-decreasing, applying Jensen's inequality twice and noticing that $H_{\lambda}(\mu x) = \mu^2H_{\lambda/\mu}(x)$. Using \eqref{eq:bound}, it readily follows that 
    \begin{equation}\label{eq:bound2}
        H_{\eta_1/(\eta_2\mu)}(\|\widehat{\bx}^{(t)} - \bx^\star\|) \leq \frac{M}{\eta_2^2\mu^2t^\kappa},
    \end{equation} where $M,\kappa$ depend on the step-size schedule and other problem parameters. By the definition of Huber loss and~\eqref{eq:bound2}, if $\|\widehat{\bx}^{(t)} - \bx^\star\| \leq \nicefrac{\eta_1}{\eta_2\mu}$, we have
    \begin{equation}\label{eq:8}
        \|\widehat{\bx}^{(t)} - \bx^\star\|^2 \leq \frac{2M}{\eta_2^2\mu^2t^{\kappa}}.
    \end{equation} Otherwise, if $\|\widehat{\bx}^{(t)} - \bx^\star\| > \nicefrac{\eta_1}{\eta_2\mu}$, by~\eqref{eq:bound2}, we have
    \begin{equation*}
        \frac{\eta_1\|\widehat{\bx}^{(t)} - \bx^\star\|}{2\eta_2\mu} < \frac{\eta_1\|\widehat{\bx}^{(t)} - \bx^\star\|}{\eta_2\mu} - \frac{\eta_1^2}{2\eta_2^2\mu^2} = H_{\eta_1/(\eta_2\mu)}(\|\widehat{\bx}^{(t)} - \bx^\star\|) \leq \frac{M}{\eta_2^2\mu^2t^\kappa},
    \end{equation*} implying that 
    \begin{equation}\label{eq:9}
        \|\widehat{\bx}^{(t)} - \bx^\star\|^2 \leq \frac{4M^2}{\eta_1^2\eta_2^2\mu^2t^{2\kappa}}.
    \end{equation} Combining~\eqref{eq:8} and~\eqref{eq:9}, it then follows that
    \begin{equation*}
        \|\widehat{\bx}^{(t)} - \bx^\star\|^2 \leq \max\left\{\frac{2M}{\eta_2^2\mu^2t^\kappa}, \frac{4M^2}{\eta_1^2\eta_2^2\mu^2t^{2\kappa}} \right\},
    \end{equation*} completing the proof.
\end{proof}

\subsection{Proof of Theorem~\ref{theorem:main}}\label{sec:proof-thm2}

In this section we prove Lemma~\ref{lm:key} and Theorem~\ref{theorem:main}. In order to prove Lemma~\ref{lm:key}, we first state and prove some intermediate results.

\begin{lemma}\label{lm:gradient-bound}
    Let Assumptions~\ref{asmpt:nonlin}-\ref{asmpt:noise} hold, with the step-size given by $\alpha_t = \frac{a}{(t+1)^\delta}$, for any $\delta \in (0.5,1)$ and $a > 0$. Then, for any $t \geq 1$, we have
    \begin{equation*}
        \|\nabla f(\bxt)\| \leq H_t \triangleq L\left(\|\bx^{(1)} - \bx^\star\| + aC\right)\frac{(t+1)^{1-\delta}}{1 - \delta}.
    \end{equation*}
\end{lemma}
\begin{proof} Using $L$-smoothness of $f$ and the update~\eqref{eq:nonlin-sgd}, we have
    \begin{align}
        \|\nabla f(\bxt)\| &\leq L\|\bxt - \bx^\star\| = L\|\bx^{(t-1)} - \alpha_{t-1}\bPsi^{(t-1)} - \bx^\star\| \nn \\ 
        &\leq L\left(\|\bx^{(t-1)} - \bx^\star\| + \alpha_{t-1}\|\boldsymbol{\Psi}^{(t-1)}\| \right) \nn \\ 
        &\leq L\left(\|\bx^{(t-1)} - \bx^\star\| + \alpha_{t-1}C \right). \label{eq_proof:thm:nonL_cw_1}
    \end{align}
    Unrolling the recursion in \eqref{eq_proof:thm:nonL_cw_1}, we get
    \begin{align*}
        \|\nabla f(\bxt)\| \leq L \|\bx^{(1)} - \bx^\star\| + LC\sum_{k = 1}^t\alpha_{k} \leq L\left(\|\bx^{(1)} - \bx^\star\| + aC\right)\frac{(t+1)^{1-\delta}}{1 - \delta}, 
    \end{align*} completing the proof.
\end{proof}

The next result characterizes the behaviour of the nonlinearity, when it takes the form $\bPsi(\bx) = \lbr \calN_1(x_1), \dots, \calN_1(x_d) \rbr^\top$. It follows a similar idea to Lemma~5.5 from~\citet{jakovetic2023nonlinear}, with the main difference due to allowing for potentially different marginal PDFs of each noise component. Since the proof follows the same steps, we omit it for brevity.

\begin{lemma}
\label{lemma:1.1}
    Let Assumptions~\ref{asmpt:nonlin}-\ref{asmpt:noise} hold and the nonlinearity $\bPsi$ be component-wise, i.e., of the form $\bPsi(\bx) = \lbr \calN_1(x_1), \dots, \calN_1(x_d) \rbr^\top$. Then, there exists a positive constant $\xi$ such that, for any $t \geq 1$, there holds almost surely for each $j = 1,\ldots,d$, that $|\phi^{(t)}_i| \geq |[\nabla f(\bxt)]_i|\frac{\phi_i^\prime(0)\xi}{2H_t}$, where $H_t$ is defined in Lemma~\ref{lm:gradient-bound}, while $\phi_i^\prime(0) = \frac{\partial}{\partial x_i}\E_{z_i}\calN_1(x_i + z_i)\:\big\vert_{x_i = 0}$.
\end{lemma}

The next result characterizes the behaviour of the nonlinearity, when it takes the form $\bPsi(\bx) = \bx\calN_2(\|\bx\|)$.

\begin{lemma}\label{lm:10}
    Let Assumptions~\ref{asmpt:nonlin}-\ref{asmpt:noise} hold and the nonlinearity be of the form $\bPsi(\bx) = \bx\calN_2(\|\bx\|)$. Then, for any $t \geq 1$, there holds almost surely that  
    \begin{equation*}
        \langle\nabla f(\bxt), \boldsymbol{\Phi}^{(t)}\rangle \geq \frac{\|\nabla f(\bxt)\|^2p_0\calN_2(1)}{H_t + C_0}, 
    \end{equation*} where $p_0 = P(\mathbf{0})$, $C_0 = \min\{0.5,B_0\}$ and $H_t$ is defined in Lemma~\ref{lm:gradient-bound}.
\end{lemma}
\begin{proof}
    We start from \eqref{eq:intermed}, which tells us that, for any $t \geq 1$, almost surely 
    \begin{equation*}
        \langle \bPhit, \nabla f(\bxt) \rangle \geq \|\nabla f(\bxt)\|^2p_0\calN_2(1)\min\left\{\frac{1}{\|\nabla f(\bxt)\| + C_0},1 \right\}.
    \end{equation*} Combining with Lemma \ref{lm:gradient-bound} and the fact that $H_t \geq 1$, we get almost surely
    \begin{equation*}
        \langle \bPhit, \nabla f(\bxt) \rangle \geq \frac{\|\nabla f(\bxt)\|^2p_0\calN_2(1)}{H_t + C_0},
    \end{equation*} which completes the proof.
\end{proof}

We are now ready to prove Lemma~\ref{lm:key}.

\begin{proof}[Proof of Lemma~\ref{lm:key}]
    First, consider the case when the nonlinearity is of the form $\bPsi(\bx) = \lbr \calN_1(x_1), \dots, \calN_1(x_d) \rbr^\top$. We then have 
    \begin{align*}
        \langle \bPhit, \nabla f(\bxt)\rangle &= \sum_{i = 1}^d \phi^{(t)}_i [\nabla f(\bxt)]_i \stackrel{(a)}{=} \sum_{i = 1}^d |\phi^{(t)}_i| |[\nabla f(\bxt)]_i| \\ & \stackrel{(b)}{\geq} \sum_{i = 1}^d |[\nabla f(\bxt)]_i|^2\frac{\phi_i^\prime(0)\xi}{2H_t} \stackrel{(c)}{\geq} \frac{\phi^\prime(0)\xi}{2H_t}\|\nabla f(\bxt)\|^2 = \gamma(t + 1)^{\delta - 1}\|\nabla f(\bxt)\|^2,
    \end{align*} where $\gamma = \frac{(1 - \delta)\phi^\prime(0)\xi}{2L\left(\|\bx^{(1)} - \bx^\star\| + aC\right)}$, $(a)$ follows from the oddity of $\calN_1$, $(b)$ follows from Lemma~\ref{lemma:1.1}, $(c)$ follows from $\phi^\prime(0) = \min_{i =1,\ldots,d}\phi_i^\prime(0)$. On the other hand, if the nonlinearity is of the form $\bPsi(\bx) = \bx\calN_2(\|\bx\|)$, we get
    \begin{align*}
        \langle \bPhit, \nabla f(\bxt)\rangle \geq \frac{p_0\calN_2(1)\|\nabla f(\bxt)\|^2}{H_t+ C_0} \geq \gamma(t + 1)^{\delta - 1}\|\nabla f(\bxt)\|^2,
    \end{align*} where $\gamma = \frac{(1-\delta)p_0\calN_2(1)}{L\left(\|\bx^{(1)} - \bx^\star\| + aC\right) + C_0}$, the first inequality follows from Lemma~\ref{lm:10}, while the second follows from the definition of $H_t$ and the fact that $H_t + C_0 \leq (L\left(\|\bx^{(1)} - \bx^\star\| + aC\right) + C_0)\frac{(t+1)^{1 - \delta}}{1 - \delta}$. This completes the proof.
\end{proof}

We next prove Theorem~\ref{theorem:main}.

\begin{proof}[Proof of Theorem~\ref{theorem:main}]
     Using $L$-smoothness of $f$, the update rule~\eqref{eq:nonlin-sgd2} and Lemma~\ref{lm:key}, we have
    \begin{align*}
        f (\bxtp) &\leq f(\bxt) - \alpha_t\langle \nabla f(\bxt), \bPhi^{(t)} - \bet \rangle + \mfrac{\alpha_t^2L}{2}\|\boldsymbol{\Psi}^{(t)}\|^2 \nn 
        \\ &\leq f(\bxt) - \mfrac{a\gamma\|\nabla f(\bxt)\|^2}{(t + 1)} + \mfrac{a\langle \nabla f(\bxt),\bet \rangle}{(t + 1)^\delta} + \mfrac{a^2LC^2}{2(t + 1)^{2\delta}}.
    \end{align*} Subtracting $f^\star$ from both sides of the inequality, defining $F^{(t)} = f(\bxt) - f^\star$ and using $\mu$-strong convexity of $f$, we get 
    \begin{equation}
        F^{(t+1)} \leq \left(1 - \mfrac{2\mu a\gamma}{t+1} \right)F^{(t)} + \mfrac{a\langle \nabla f(\bxt),\bet \rangle}{(t + 1)^\delta} + \mfrac{a^2LC^2}{2(t + 1)^{2\delta}}. \label{eq_proof:thm:nonL_cw_4}
    \end{equation}
    Let $\zeta = \min\left\{2\delta - 1, \nicefrac{a\gamma\mu}{2} \right\}$. Defining $Y^{(t)} \triangleq t^\zeta F^{(t)} = t^\zeta(f(\bxt) - f^\star)$, from \eqref{eq_proof:thm:nonL_cw_4} we get
    \begin{equation}
        Y^{(t + 1)} \leq a_tY^{(t)} + b_t\langle \nabla f(\bxt),\bet \rangle  + c_tV, \label{eq_proof:thm:nonL_cw_5}
    \end{equation} 
    where $a_t = \left(1 - \frac{2\mu a\gamma}{t+1} \right)\left(\frac{t + 1}{t}\right)^\zeta$, $b_t = \frac{a}{(t + 1)^{\delta - \zeta}}$, $c_t = \frac{a^2}{(t + 1)^{2\delta - \zeta}}$ and $V = \frac{LC^2}{2}$. Denote the MGF of $Y^{(t)}$ conditioned on $\mathcal{F}_t$ as $M_{t+1\vert t}(\nu) = \E\left[\exp\left(\nu Y^{(t+1)}\right)\vert \mathcal{F}_t\right]$. We then have, for any $\nu \geq 0$
    \begin{align}
        M_{t+1\vert t}(\nu) \nn &\overset{(a)}{\leq} \E\left[\exp\left( \nu (a_tY^{(t)} + b_t\langle \bet,\nabla f(\bxt) \rangle  + c_tV ) \right) \big| \: \mathcal{F}_t \right] \nn \\ 
        &\overset{(b)}{\leq} \exp (\nu a_tY^{(t)} + \nu c_tV ) \mathbb{E}\left[\exp (\nu b_t\langle \bet, \nabla f(\bxt)\rangle ) \big| \: \mathcal{F}_t \right] \nn \\ 
        & \overset{(c)}{\leq} \exp\left(\nu a_tY^{(t)} + \nu c_tV + \nu^2b^2_tN\|\nabla f(\bxt)\|^2\right) \nn \\
        & \overset{(d)}{\leq} \exp\left(\nu a_tY^{(t)} + \nu c_tV + 2\nu^2b_t^{\prime 2}LNY^{(t)}\right), \label{eq_proof:thm:nonL_cw_6}
    \end{align} 
    where $(a)$ follows from \eqref{eq_proof:thm:nonL_cw_5}, $(b)$ follows from the fact that $Y^{(t)}$ is $\mathcal{F}_t$ measurable, $(c)$ follows from Lemma~\ref{lm:error_component}, in $(d)$ we use $\|\nabla f(\bx)\|^2 \leq 2L(f(\bx) - f^\star)$ and define $b^\prime_t = a\frac{t^\frac{-\zeta}{2}}{(t+1)^{\delta - \zeta}}$, so that $b_t = t^\frac{\zeta}{2}b^\prime_t$. For the choice $0 \leq \nu \leq B$, for some $B > 0$ (to be specified later), we get
    \begin{equation*}
        M_{t+1\vert t}(\nu) \leq \exp\left(\nu(a_t + 2b_t^{\prime 2}LNB)Y^{(t)}\right)\exp\left(\nu c_tV\right).
    \end{equation*} 
    Taking the full expectation, we get
    \begin{equation}
    \label{eq:induction_step}
        M_{t+1}(\nu) \leq M_t((a_t + 2b_t^{\prime2}LNB)\nu)\exp(\nu c_tV).
    \end{equation} Similarly to the approach in~\citet{harvey2019tight}, we now want to show that $M_t(\nu) \leq e^{\frac{\nu}{B}}$, for any $0 \leq \nu \leq B$ and any $t \geq 1$. We proceed by induction. For $t=1$, we have
    \begin{equation*}
        M_1(\nu) = \exp(\nu Y^{(1)}) = \exp\left(\nu (f(\bx^{(1)}) - f^\star) \right),
    \end{equation*} 
    where we simply used the definition of $Y^{(t)}$ and the fact that it is deterministic for $t = 1$. Choosing $B \leq (f(\bx^{(1)}) - f^\star)^{-1}$ ensures that $M_1(\nu) \leq e^{\frac{\nu}{B}}$. Next, assume that for some $t \geq 2$ it holds that $M_t(\nu) \leq e^{\frac{\nu}{B}}$. We then have
    \begin{align*}
        M_{t+1}(\nu) \leq M_t((a_t + 2b_t^{\prime2}LNB)\nu)\exp(\nu c_tV) \leq \exp\left((a_t + 2b_t^{\prime2}LNB + c_tVB)\mfrac{\nu}{B} \right),
    \end{align*} 
    where we use~\eqref{eq:induction_step} in the first and the induction hypothesis in the second inequality. For our claim to hold, it suffices to show $a_t + 2b_t^{\prime2}LNB + c_tVB \leq 1$. Plugging in the values of $a_t$, $b_t^\prime$ and $c_t$, we have
    \begin{align*}
        a_t + 2b_t^{\prime2}LNB + c_tVB &= \left(1 - \mfrac{2\mu a\gamma}{t+1} \right)\left(\mfrac{t + 1}{t}\right)^{\zeta} + \mfrac{2a^2LNB}{(t + 1)^{2\delta - 2\zeta}t^{\zeta}} + \mfrac{a^2VB}{(t + 1)^{2\delta - \zeta}} \\ &\leq \left(\mfrac{t + 1}{t}\right)^{\zeta}\bigg(1 - \mfrac{2\mu a\gamma}{t+1} + \mfrac{2a^2LNB}{(t + 1)^{2\delta - \zeta}} +  \mfrac{a^2VBt^\zeta}{(t + 1)^{2\delta}}\bigg) \\ &\leq \left(\mfrac{t + 1}{t}\right)^{\zeta} \left(1 - \mfrac{2\mu a\gamma}{t+1} + \mfrac{2a^2LNB}{(t + 1)^{2\delta - \zeta}} +  \mfrac{a^2VB}{(t + 1)^{2\delta - \zeta}}\right).
    \end{align*} 
    Noticing that $2\delta - \zeta \geq 1$ and setting $B = \min\left\{\frac{1}{(f(\bx^{(1)}) - f^\star)}, \frac{\mu\gamma}{2aLN + aV} \right\}$, gives
    \begin{align*}
        a_t + 2b_t^{\prime 2}LNB + c_tVB \leq \left(\mfrac{t + 1}{t}\right)^{\zeta}\left(1 - \mfrac{\mu a\gamma}{t+1}\right) \leq \exp \lp \mfrac{\zeta}{t} - \mfrac{a\mu\gamma}{t + 1} \rp  \leq 1, 
    \end{align*} 
    where in the second inequality we use $1 + x \leq e^x$, while the third inequality follows from the choice of $\zeta$. Therefore, we have shown that $M_{t}(\nu) \leq e^\frac{\nu}{B}$, for any $t \geq 1$ and any $0 \leq \nu \leq B$. By Markov's inequality, it readily follows that
    \begin{align*}
        \mathbb{P}(f(\bxtp) - f^\star \geq \epsilon) = \mathbb{P}(Y_{t+1} \geq (t + 1)^\zeta\epsilon) \leq e^{-\nu(t + 1)^\zeta\epsilon}M_{t+1}(\nu) \leq e^{1-B(t + 1 )^\zeta\epsilon},
    \end{align*} where in the last inequality we set $\nu = B$. Finally, using strong convexity, we have 
    \begin{align*}
        \mathbb{P}(\|\bxtp - \bx^\star \|^2 \geq \epsilon) \leq \mathbb{P}\left(f(\bxtp) - f^\star \geq \mfrac{\mu}{2}\epsilon\right) \leq ee^{-B(t + 1)^\zeta\mfrac{\mu}{2}\epsilon},
    \end{align*} which implies that, for any $\beta \in (0,1)$, with probability at least $1 - \beta$,
    \begin{equation*}
        \|\bxtp - \bx^\star\|^2 \leq  \mfrac{2\log\left(\nicefrac{e}{\beta}\right)}{\mu B(t+1)^\zeta},
    \end{equation*} completing the proof. 
\end{proof}

\subsection{Proof of Theorem \ref{thm:non-sym}}\label{sec:proof-thm3}

\begin{proof}[Proof of Theorem \ref{thm:non-sym}]
    Consider the ``denoised'' nonlinearity $\bPhit \triangleq \E[\bPsi(\nabla f(\bxt) + \bzt)\: \vert \: \mathcal{F}_t]$. From Assumption \ref{asmpt:non-sym} and the linearity of expectation, it follows that $\bPhit$ can be expressed as 
    \begin{equation}\label{eq:phi-split}
        \bPhit = \lambda \bPhi_1^{(t)} + (1 - \lambda)\bPhi_2^{(t)},
    \end{equation} where $\bPhi_i^{(t)} = \E_{\bzt \sim P_i}[\bPsi(\nabla f(\bxt) + \bzt)\: \vert \: \mathcal{F}_t]$, $i \in [2]$ are the ``denoised'' nonlinearities with respect to each of the noise components. Defining the effective noise as $\bet = \bPhit - \bPsi^{(t)}$, it can be readily seen that Lemma \ref{lm:error_component} still applies. Similarly, it can be seen that Lemma \ref{lm:huber} holds for $\bPhi_1$, as this represents the effective search direction with respect to the symmetric noise component. Apply the smoothness inequality and the update rule~\eqref{eq:nonlin-sgd2}, to get
    \begin{align}
        f(\bxtp) &\leq f(\bxt) - \alpha_t\langle \nabla f(\bxt),\bPhit - \bet \rangle + \frac{\alpha_t^2L}{2}\|\bPsi^{(t)}\|^2 \nonumber \\ &\leq f(\bxt) - \alpha_t(1-\lambda)\langle \nabla f(\bxt),\bPhi_1^{(t)}\rangle - \alpha_t\lambda\langle \nabla f(\bxt),\bPhi_2^{(t)} \rangle + \alpha_t\langle \nabla f(\bxt),\bet\rangle + \frac{\alpha_t^2LC^2}{2} \nonumber \\& \leq f(\bxt) - \alpha_t(1-\lambda)Z_t - \alpha_t\lambda\langle \nabla f(\bxt),\bPhi_2^{(t)} \rangle + \alpha_t\langle \nabla f(\bxt),\bet\rangle + \frac{\alpha_t^2LC^2}{2}, \label{eq:smooth-setup}
    \end{align} where the first inequality follows from \eqref{eq:phi-split} and the boundedness of the nonlinearity, while the second inequality follows from Lemma \ref{lm:error_component}, recalling that $Z_t \triangleq \min\{\eta_1\|\nabla f(\bxt) \|,\eta_2\|\nabla f(\bxt)\|^2\}$. To bound the inner product of the gradient and the non-symmetric component, we proceed as follows. For any $\bx \in \R^d$, we have
    \begin{equation}\label{eq:phi2}
        \langle \bx, \bPhi_2(\bx) \rangle \leq \|\bx\|\|\bPhi_2(\bx) \| \leq C\|\bx\| \leq \begin{cases}
            C\|\bx\|, & \|\bx\| \geq B \\
            CB, & \|\bx\| < B 
        \end{cases}, 
    \end{equation} where $B > 0$ is an arbitrary constant, to be specified later. Note that \eqref{eq:phi2} is equivalent to
    \begin{equation}\label{eq:phi2-pt2}
        \langle \bx, \bPhi_2(\bx) \rangle \leq C\max\{\|\bx\|,B\}.
    \end{equation} Plugging \eqref{eq:phi2-pt2} in \eqref{eq:smooth-setup}, we get
    \begin{align*}
        f(\bxtp) &\leq f(\bxt) - \alpha_t(1-\lambda)Z_t + \alpha_t\lambda C\max\{\|\nabla f(\bxt)\|,B\} + \alpha_t\langle \nabla f(\bxt),\bet\rangle + \frac{\alpha_t^2LC^2}{2}
    \end{align*} Setting $B = \eta_1/\eta_2$, it can be readily seen that
    \begin{equation*}
        (1-\lambda)Z_t - \lambda C\max\{\|\nabla f(\bxt)\|,\eta_1/\eta_2\} = \min\{(\eta_1(1-\lambda) - \lambda C)\|\nabla f(\bxt)\|, \eta_2(1-\lambda)\|\nabla f(\bxt)\|^2 - \lambda C\eta_1/\eta_2 \}.
    \end{equation*} From the condition $\lambda < \frac{\eta_1}{\eta_1 + C}$, it follows that $\eta_1(1-\lambda) - \lambda C > 0$. Next, define $\widetilde{Z}_t \triangleq \min\{(\eta_1(1-\lambda) - \lambda C)\|\nabla f(\bxt)\|, \eta_2(1-\lambda)\|\nabla f(\bxt)\|^2 - \lambda C\eta_1/\eta_2 \}$. Rearranging and summing up the first $t$ terms, we get
    \begin{equation*}
        \sum_{k = 1}^{t}\alpha_k\widetilde{Z}_k \leq f(\bx^{(1)}) - f^\star + \frac{LC^2}{2}\sum_{k = 1}^{t}\alpha_k^2 + \sum_{k = 1}^{t}\alpha_k\langle \nabla f(\bxk), \bek \rangle. 
    \end{equation*} Repeating the same steps as in the proof of Theorem \ref{thm:non-conv}, we get
    \begin{align}
        \sum_{k = 1}^{t}&\widetilde{\alpha}_k\widetilde{Z}_k \leq \frac{(1-\delta)\left(f(\bx^{(1)}) - f^\star + \log(\nicefrac{1}{\beta})\right)}{a((t+2)^{1-\delta}-2^{1-\delta})} \nonumber \\ &+ \frac{a(1-\delta)LC^2(\nicefrac{1}{2} + 8LD_{\mathcal{X}})}{(2\delta-1)((t+2)^{1-\delta} - 2^{1-\delta})} + \frac{8a^3C^4L^2\sum_{k = 1}^{t}(k+1)^{2-4\delta}}{(1-\delta)((t+2)^{1-\delta} - 2^{1 - \delta})}. \label{eq:semi-final-pt2}
    \end{align} Considering the different step-size schedules, we can similarly obtain a unified representation of the form
    \begin{equation}\label{eq:bound-pt2}
        \sum_{k = 1}^t\widetilde{\alpha}_k\widetilde{Z}_k \leq Mt^{-\kappa},
    \end{equation} for appropriately selected constants $M,\kappa > 0$. Using $U \triangleq \{k \in [t]: \: \|\nabla f(\bxk)\| \leq \eta_1/\eta_2 \}$, $U^c \triangleq [t] \setminus U$ and \eqref{eq:bound-pt2}, we get 
    \begin{align*}
        \eta_2(1-\lambda)\sum_{k \in U}\widetilde{\alpha}_k\|\nabla f(\bxk)\|^2 \leq Mt^{-\kappa} + \lambda C\eta_1/\eta_2 \text{ and } (\eta_1(1-\lambda) - \lambda C)\sum_{k \in U^c}\widetilde{\alpha}_k\|\nabla f(\bxk)\| \leq Mt^{-\kappa}.
    \end{align*} It then readily follows that
    \begin{align*}
        \min_{k \in [t]}\|\nabla f(\bxk)\| \leq \sum_{k \in U}\widetilde{\alpha}_k\|\nabla f(\bxk)\| + \sum_{k \in U^c}\widetilde{\alpha}_k\|\nabla f(\bxk)\| \leq \sum_{k = 0}^{t-1}\widetilde{\alpha}_kz_k + M_2t^{-\kappa},
    \end{align*} where $M_2 = M/(\eta_1(1-\lambda)-\lambda C)$, while $z_k = \|\nabla f(\bxk)\|$, for $k \in U$, and $z_k = 0$, for $k \in U^c$. Using Jensen's inequality, we get
    \begin{align*}
        \min_{k \in [t]}\|\nabla f(\bxk)\| \leq \sqrt{\sum_{k = 1}^{t}\widetilde{\alpha}_kz_k^2} + M_2t^{-\kappa} &= \sqrt{\sum_{k \in U}\widetilde{\alpha}_k\|\nabla f(\bxk)\|^2} + M_2t^{-\kappa} \\ &\leq \sqrt{M_1t^{-\kappa} + \frac{\lambda C \eta_1}{\eta_2^2(1-\lambda)}} + M_2t^{-\kappa},
    \end{align*} where $M_1 = \frac{M}{\eta_2(1-\lambda)}$. Squaring both sides and using $(a+b)^2 \leq 2a^2 + 2b^2$, gives the desired result.
\end{proof}

\section{Rate $\zeta$}\label{app:rate}

Recalling Assumption~\ref{asmpt:nonlin} and the definition of $C$, it readily follows that $\gamma(a) = \frac{(1 - \delta)\phi^\prime(0)\xi}{2L\left(\|\bx^{(1)} - \bx^\star\| + a\sqrt{d}C_1\right)}$ for nonlinearities of the form $\bPsi(\bx) = \lbr \calN_1(x_1), \dots, \calN_1(x_d) \rbr^\top$ (i.e., component-wise), while $\gamma(a) = \frac{(1-\delta)p_0\calN_2(1)}{L\left(\|\bx^{(1)} - \bx^\star\| + aC_2\right) + C_0}$, for nonlinearities of the form $\bPsi(\bx) = \bx\calN_2(\|\bx\|)$ (i.e., joint). Combined with Theorem~\ref{theorem:main}, it follows that the rate $\zeta$ is given by
\begin{align*}
    \zeta_{joint} &= \min\left\{2\delta - 1, \frac{a\mu(1-\delta)p_0\calN_2(1)}{2L\left(\|\bx^{(1)} - \bx^\star\| + aC_2 \right) + 2C_0} \right\}, \\ \zeta_{comp} &= \min\left\{2\delta - 1, \frac{a\mu\phi^\prime(0)\xi(1-\delta)}{4L\left(\|\bx^{(1)} - \bx^\star\| + aC_1\sqrt{d} \right)} \right\}.
\end{align*} We note that $\zeta$ depends on the following problem-specific parameters:
\begin{itemize}[leftmargin=*]
    \item \emph{Initialization} - starting 
    farther from the minima results in smaller $\zeta$ (i.e., larger $\|\bx^{(1)}-\bx^\star\|$). The effect of initialization can be eliminated by choosing sufficiently large $a$.
    \item \emph{Condition number} - larger values of $\frac{L}{\mu}$ (i.e., a more difficult problem) result in smaller $\zeta$.
    \item \emph{Nonlinearity} - the dependence of $\zeta$ on the nonlinearity comes in the form of two terms: the uniform bound on the nonlinearity $C_1$ or $C_2$, and the value $\phi^\prime(0)$ or $\calN_2(1)$.
    \item \emph{Problem dimension} - for component-wise nonlinearities through $\sqrt{d}$.
    \item \emph{Noise} - in the form of $\phi^\prime(0)$, $\xi$ for component-wise and $p_0$, $C_0 = \min\{0.5,B_0\}$ for joint ones.  
    \item \emph{Step-size} - both terms in the definition of $\zeta$ depend on the step-size parameter $\delta \in (0,1)$.
\end{itemize}

\section{Derivations for Examples \ref{example:5}-\ref{example:7}}\label{app:derivations}

Recall that the size of the neighborhood and condition on $\lambda$ in Theorem \ref{thm:non-sym} are given by $\frac{\eta_1 \lambda C }{\eta_2^2(1-\lambda)}$ and $\lambda < \frac{\eta_1}{C+\eta_1}$, where $C$ is the bound on the nonlinearity, while $\eta_1,\eta_2$ are the constants from Lemma \ref{lm:huber}. From the full statement of Lemma \ref{lm:huber} in the Supplement (i.e., Lemma \ref{lm:huber-app}), we know that $\eta_1 = \nicefrac{\phi^\prime(0)\xi}{2\sqrt{d}}$, $\eta_2 = \nicefrac{\phi^\prime(0)}{2d}$ for copmponent-wise and $\eta_1 = p_0\calN_2(1)/2$, $\eta_2 = p_0\calN_2(1)$ for joint nonlinearities. From the definition of PDF in Example \ref{example:1}, it follows that $p_0 = P(\mathbf{0}) = \left[\frac{\alpha - 1}{2}\right]^d$. We now consider specific nonlinearities. 

\begin{enumerate}
    \item For sign, we have $C = \sqrt{d}$ and it can be shown that $\phi^\prime(0) \approx \alpha - 1$, $\xi \approx \frac{1}{\alpha}$, see \cite{jakovetic2023nonlinear}.

    \item For component-wise clipping with parameter $m > 1$, we have $C = m\sqrt{d}$ and it can be shown that $\phi^\prime(0) \approx 1 - (m+1)^{-\alpha}$, $\xi \approx m-1$, see \cite{jakovetic2023nonlinear}.

    \item For joint clipping with parameter $M > 0$, we have $C = M$ and $\calN_1(1) = \min\{1,M\}$.
\end{enumerate}

Plugging in the said values completes the derivations.

\section{Analytical Example}\label{app:analytical}

In this section we specialize the rates from Theorem \ref{thm:non-conv} for specific choices of nonlinearity and noise, showing analytically that our theory predicts clipping is not always the optimal choice of nonlinearity and confirms the prior findings of \cite{zhang2020adaptive}, namely that for some noise instances, component-wise clipping shows better dimension dependence than joint clipping. 

To that end, we consider the noise with PDF from Example \ref{example:1}, for some $\alpha > 2$ and choice of step-size with $\delta = 3/4$. We consider component-wise and joint clipping, with thresholds $m > 1$ and $M > 0$, respectively. As shown in the derivations from the previous section, in this case, we have $C_{cc} = m\sqrt{d}$, $\eta_{1,cc} = \frac{[1-(m+1)^{-\alpha}](m-1)}{2\sqrt{d}}$, $\eta_{2,cc} = \frac{1-(m+1)^{-\alpha}}{2d}$ for component-wise and $C_{jc} = M$, $\eta_{1,jc} = \left[\frac{\alpha - 1}{2}\right]^d\min\{1/2,M/2\}$, $\eta_{2,jc} = \left[\frac{\alpha - 1}{2}\right]^d\min\{1,M\}$ for joint clipping. For simplicity, we ignore the higher-order term in the bound of Theorem \ref{thm:non-conv} and focus on the first, dominating term, which is ok to do, as the dependence on problem parameters and $\eta_1$, $\eta_2$ in both terms is almost identical. Similarly, we will only focus on the resulting problem related constants that figure in the leading term, ignoring the rate and global constants. To that end, we have the following problem related constants figuring in the leading terms
\begin{align*}
    \text{Component clipping: } &\frac{d(f(\bx^{(1)}-f^\star + \log(\nicefrac{1}{\beta})) + a^2d^2m^2L(1 + LD_{\mathcal{X}}) + a^4d^3m^4L^2}{a[1-(m+1)^{-\alpha}]}, \\
    \text{Joint clipping: } &\frac{(f(\bx^{(1)}-f^\star + \log(\nicefrac{1}{\beta})) + a^2M^2L(1 + LD_{\mathcal{X}}) + a^4M^4L^2}{a[(\alpha-1)/2]^d\min\{1,M\}}.
\end{align*} Note that the leading term for component clip shows a polynomial dependence on problem dimension, of order $d^3$, while the leading term for the joint clip has an exponential dependence on $d$, via $[(\alpha-1)/2]^{-d}$. As $\alpha$ is an intrinsic property of the noise, whenever $\alpha \in (2,3)$, (i.e., variance is unbounded and noise is heavy-tailed), we have $[(\alpha-1)/2]^{-d} \rightarrow \infty$, as $d \rightarrow \infty$, at an exponential rate, showing a much worse dependence on problem dimension than component clip, providing a theoretical confirmation of our numerical results (recall that we use $\alpha = 2.05$ in our simulations) and underlining the benefits of component clipping over joint one for certain noises and certain regimes, as noted in \cite{zhang2020adaptive}. The polynomial dependence of component clip on dimension $d$ can be seen as a byproduct of our unified black-box analysis, wherein we provide a general bound $C$, which results in a factor $\sqrt{d}$ when specialized to component-wise nonlinearities. This polynomial dependence is unavoidable, as, even by tuning the step-size parameter $a$ and clipping threshold 
$m > 1$, we can at best remove the direct dependence on $d$ in the numerator, while resulting in the denominator of the form $[1 - (m/d^{\kappa}+1)^{-\alpha}]$, for some $\kappa > 0$, which still explodes as $d \rightarrow \infty$, again at a polynomial rate. Similarly, the exponential explosion of the bound in the joint clipping case and heavy-tailed noise (i.e., 
$\alpha \in (2,3)$) is unavoidable, even under careful tuning of $a$ and $M$. Therefore, our bounds confirm the observations from \cite{zhang2020adaptive}, that for some noise instances, component clipping shows better dimension dependence than than the joint one. Finally, we note that the same dependence on problem dimension can be shown to hold for sign and normalized gradient, further underlining the benefit of component-wise nonlinearities for some noise instances.

\section{Additional Experiments}\label{app:num}

In this section we provide additional experiments.

\paragraph{Noise Symmetry - Setup Details.} The convolutional layers have 32 and 64 filters, with $3 \times 3$ kernels, respectively. The fully connected layers are of size $9216 \times 168$ and $168 \times 10$, respectively. We apply dropout, with rates 0.25 and 0.5, respectively, applied after the max pooling layers and the first fully connected layer. We use a batch size of 64, set the learning rate to 1 and decrease it by a factor of $0.7$ every epoch. The experiments are done on MacOS 15.0 with M1 Pro processor using PyTorch 2.2.2 MPS backend.

\paragraph{Noise Symmetry - Additional Results.} In Figure \ref{fig:proj-epoch15-reps}, we independently sample 6 Gaussian random projection matrices, and for each realization we plot the per-sample gradient projections, after training for 15 epochs. We can see that the noise projection is again highly symmetric for most random projections. 

\begin{figure*}[htbp]
    \centering
    \begin{subfigure}[b]{0.3\textwidth}
        \centering
        \includegraphics[width=\textwidth]{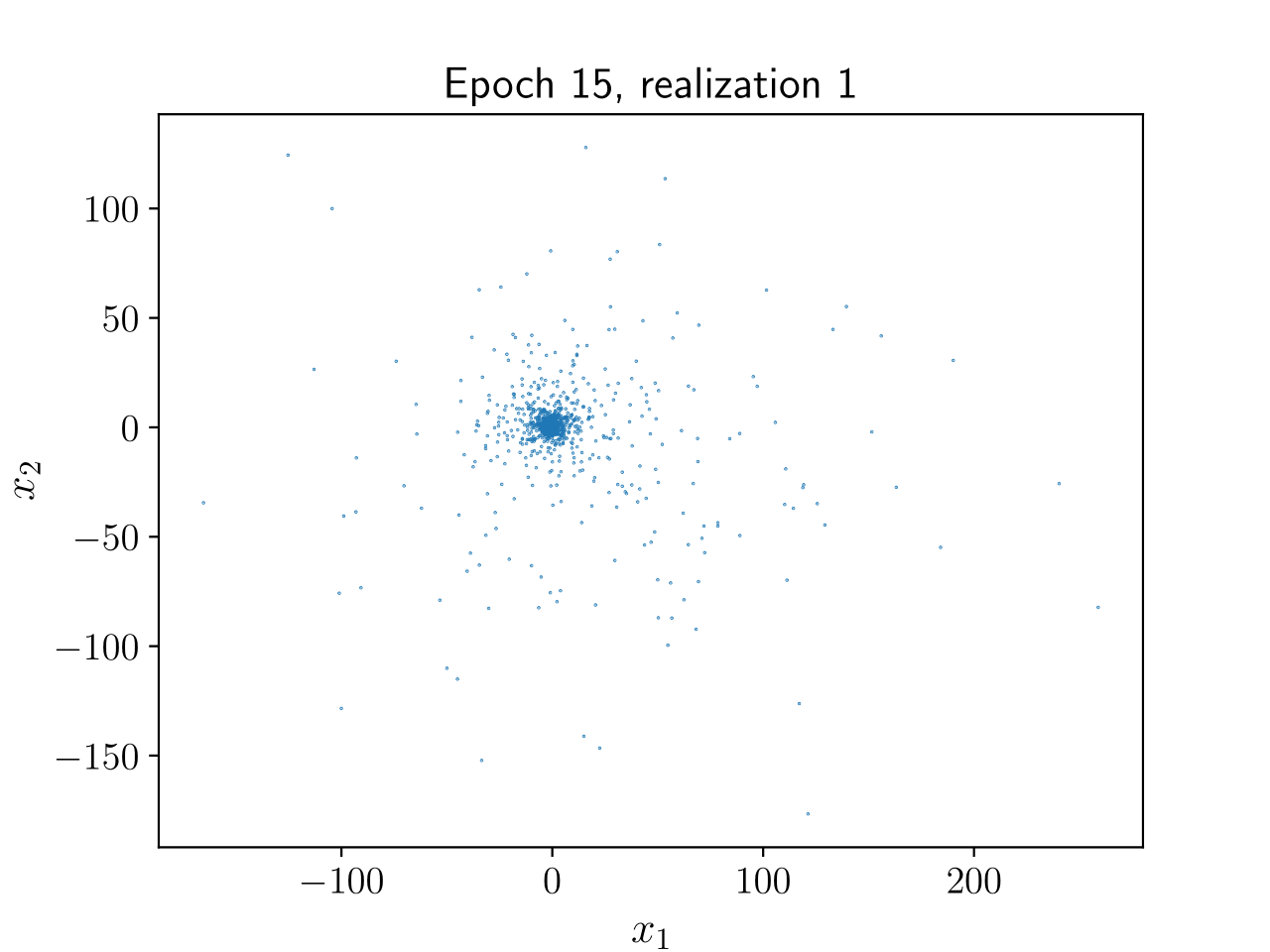} 
    \end{subfigure}
    \hfill
    \begin{subfigure}[b]{0.3\textwidth}
        \centering
        \includegraphics[width=\textwidth]{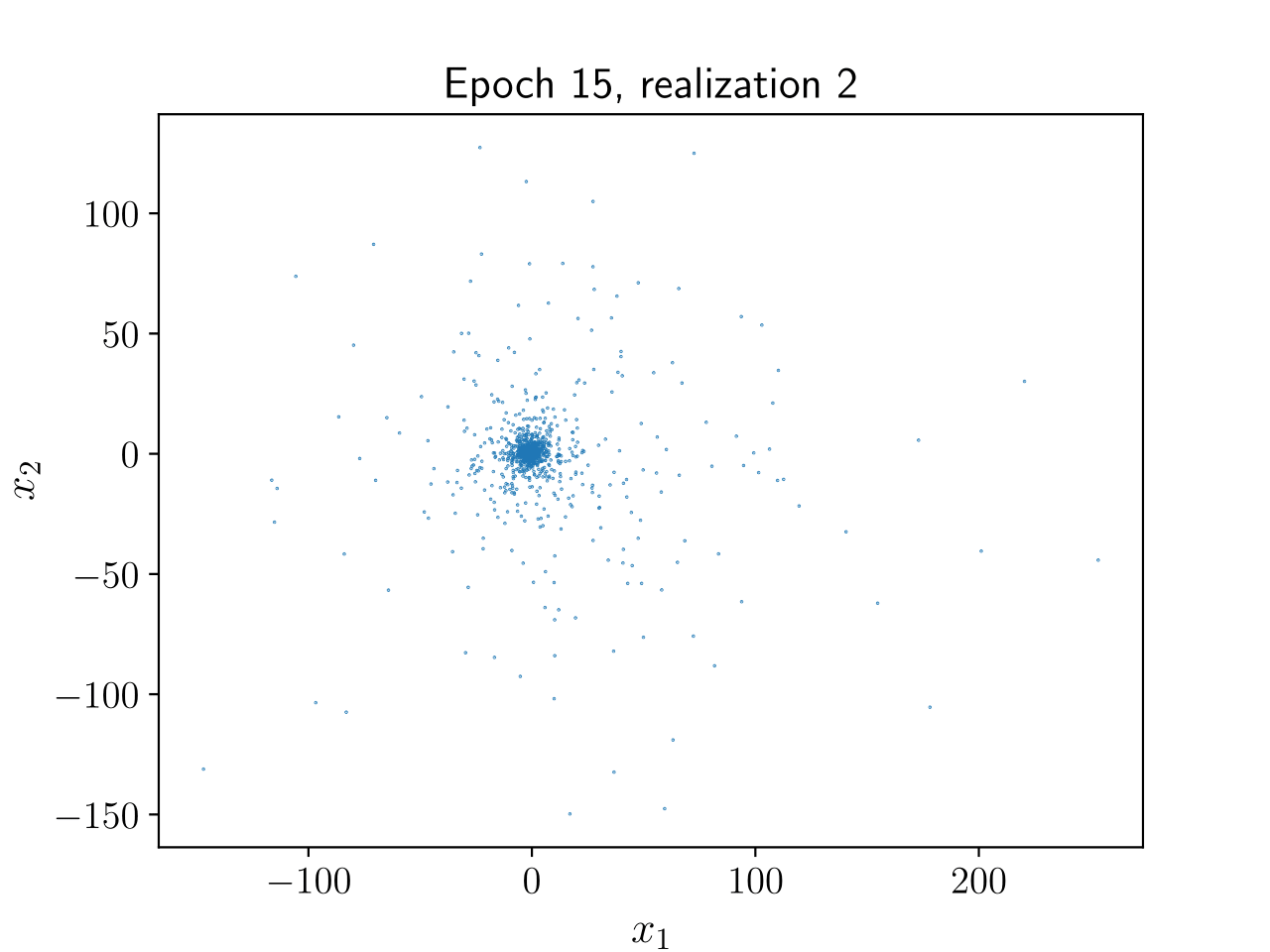} 
    \end{subfigure}
    \hfill
    \begin{subfigure}[b]{0.3\textwidth}
        \centering
        \includegraphics[width=\textwidth]{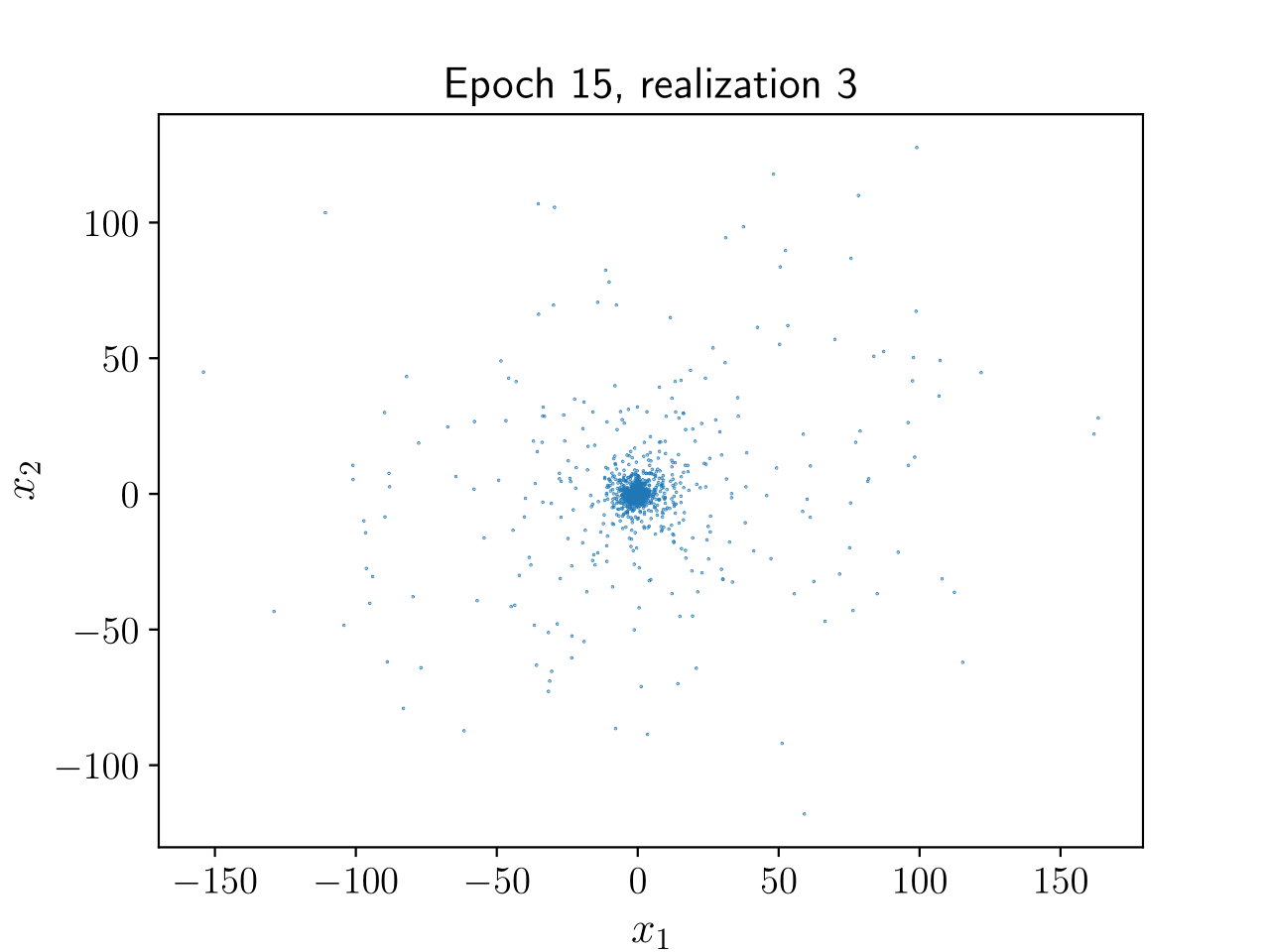} 
    \end{subfigure}
    \vskip\baselineskip
    \begin{subfigure}[b]{0.3\textwidth}
        \centering
        \includegraphics[width=\textwidth]{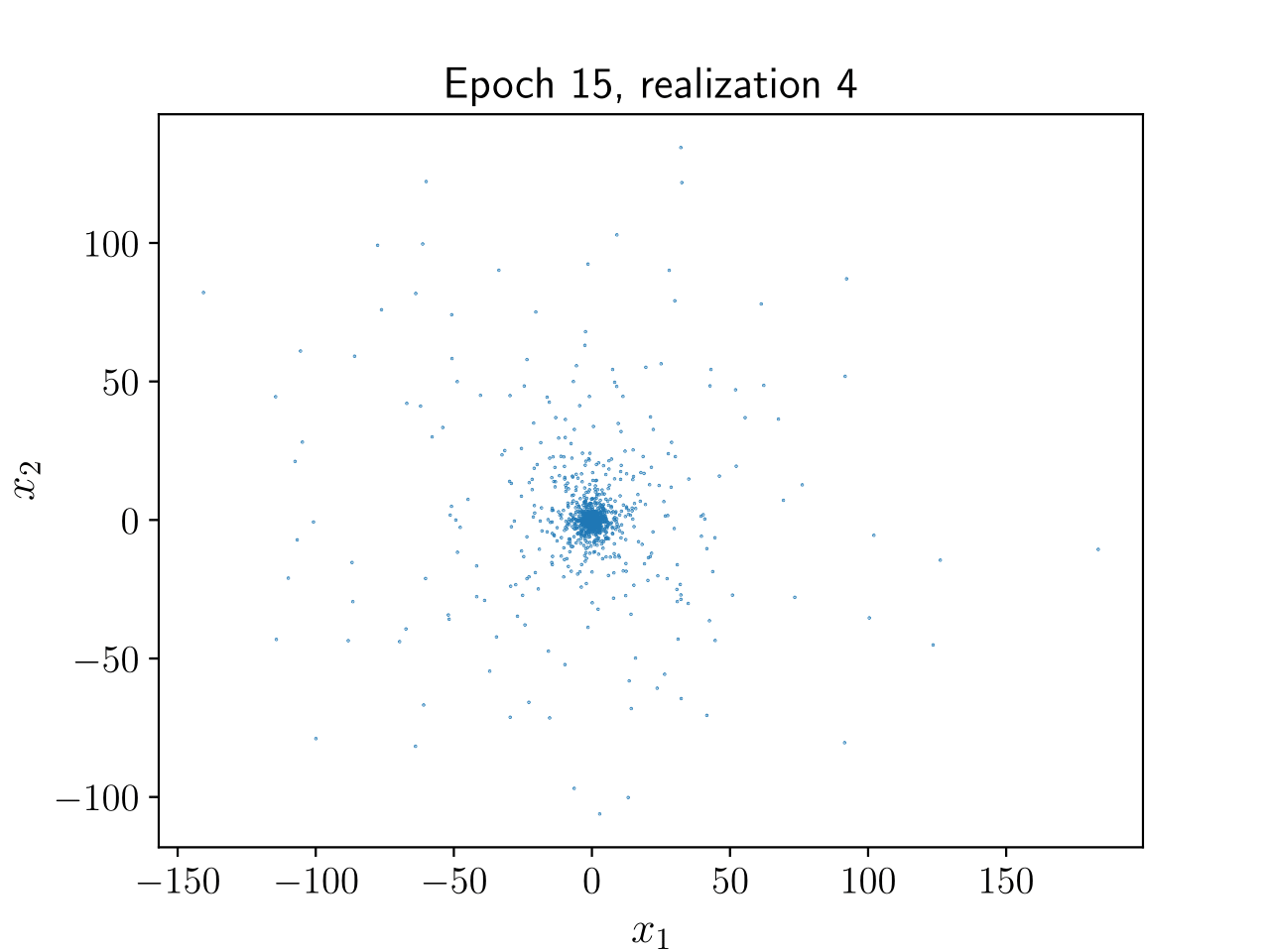} 
    \end{subfigure}
    \hfill
    \begin{subfigure}[b]{0.3\textwidth}
        \centering
        \includegraphics[width=\textwidth]{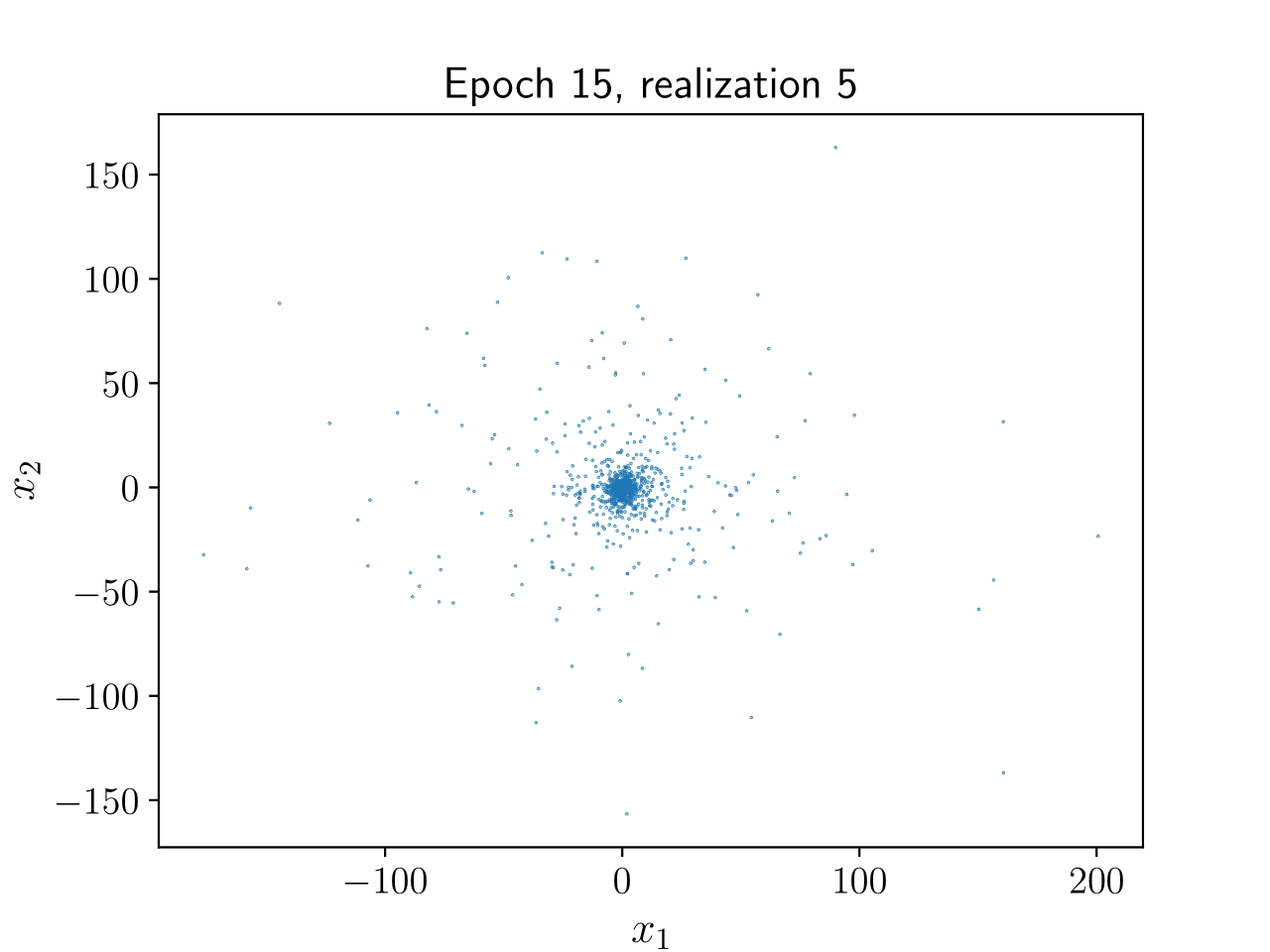}
    \end{subfigure}
    \hfill
    \begin{subfigure}[b]{0.3\textwidth}
        \centering
        \includegraphics[width=\textwidth]{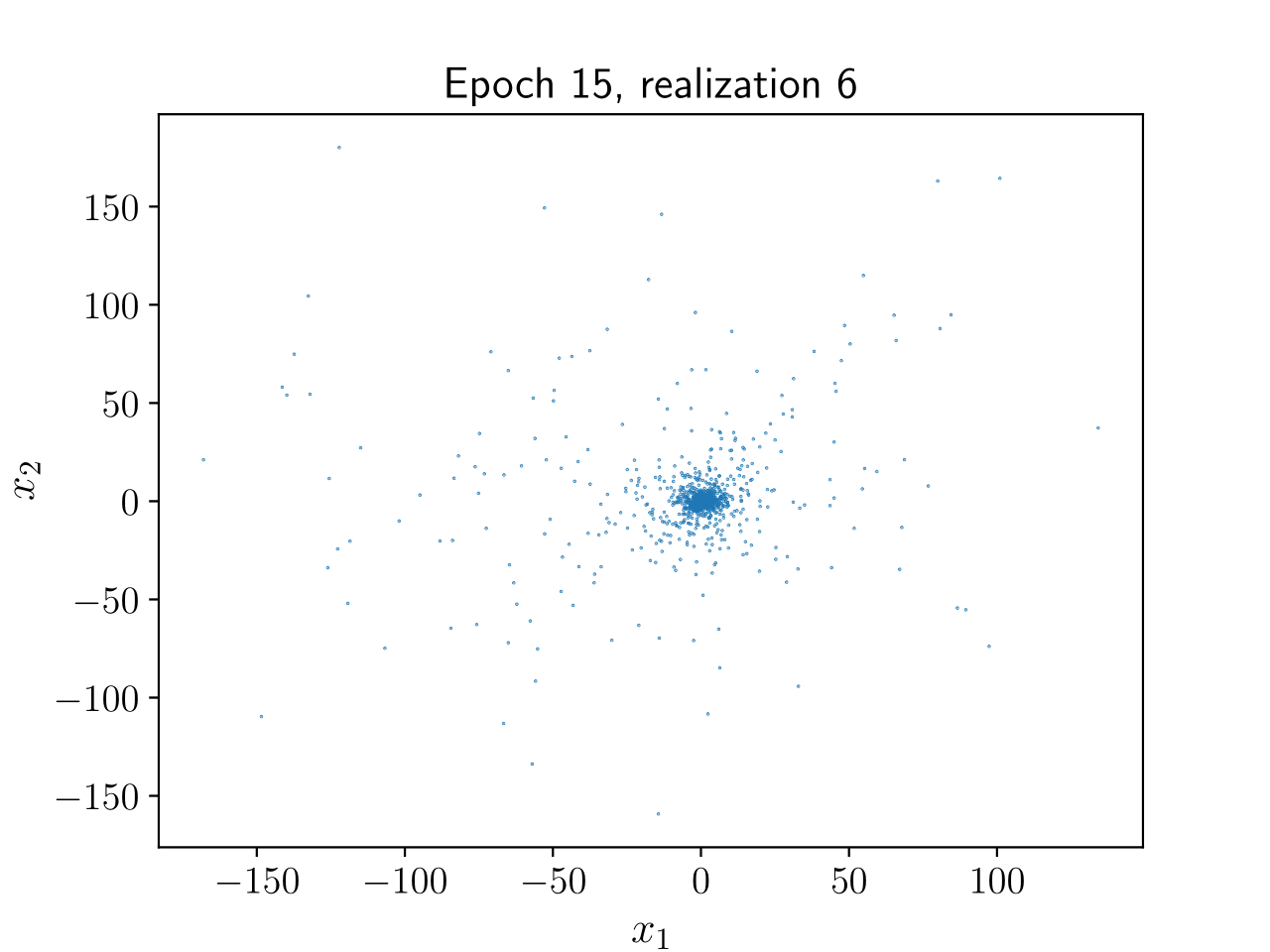} 
    \end{subfigure}
    \caption{The distribution of gradient projections after training for 15 epochs, using 6 different projection matrices.}
    \label{fig:proj-epoch15-reps}
\end{figure*}

\paragraph{Comparisons of Nonlinear SGD Methods.} We use the same MNIST dataset and CNN model as described above to test the performance of SGD with different nonlinearities under injected heavy-tailed noise. In particular, when computing mini-batch stochastic gradients, we inject random noise following a Levy stable distribution, with the stability parameter 1.5, the skewness 1, location parameter 0, and scale 1. Note that this is a non-symmetric heavy-tailed distribution. We compare the test accuracies and test losses of baseline SGD method, SGD with component-wise and joint clipping, as well as normalized SGD. All algorithms use a varying step-size schedule $\alpha_t = \frac{a}{(t + 1)^{3/4}}$, where $a$ is a hyper-parameter chosen from $\{0.001, 0.005, 0.01, 0.05, 0.1, 0.5, 1.0\}$. For component-wise and joint clipped SGD, we pick the best clipping threshold from the set $\{0.1, 0.5, 1.0\}$. For the best hyper-parameter combination for each algorithm, we run the algorithm for 5 independent runs and plot the mean value with error bars. The results are presented in Figure \ref{fig:perf_nlr}, where it can be seen that all nonlinear SGD methods (fine-tuned) perform well, while the performance of vanilla SGD is significantly affected by the presence of heavy-tailed noise.

\begin{figure}[!ht]
    \centering
    \begin{subfigure}{0.48\textwidth}  
        \centering
        \includegraphics[width=\linewidth]{{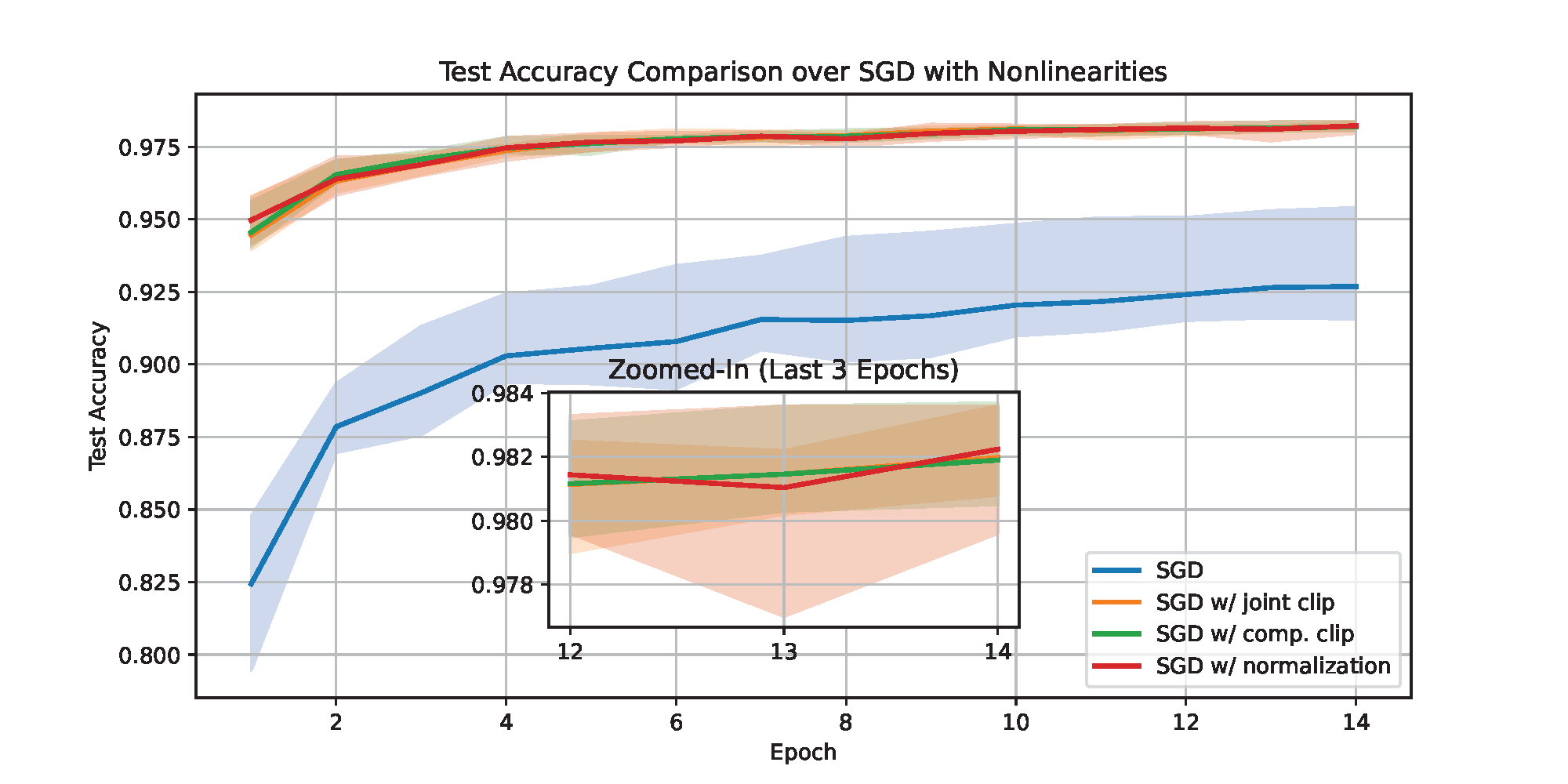}} 
    \end{subfigure}
    \hfill
    \begin{subfigure}{0.48\textwidth}  
        \centering
        \includegraphics[width=\linewidth]{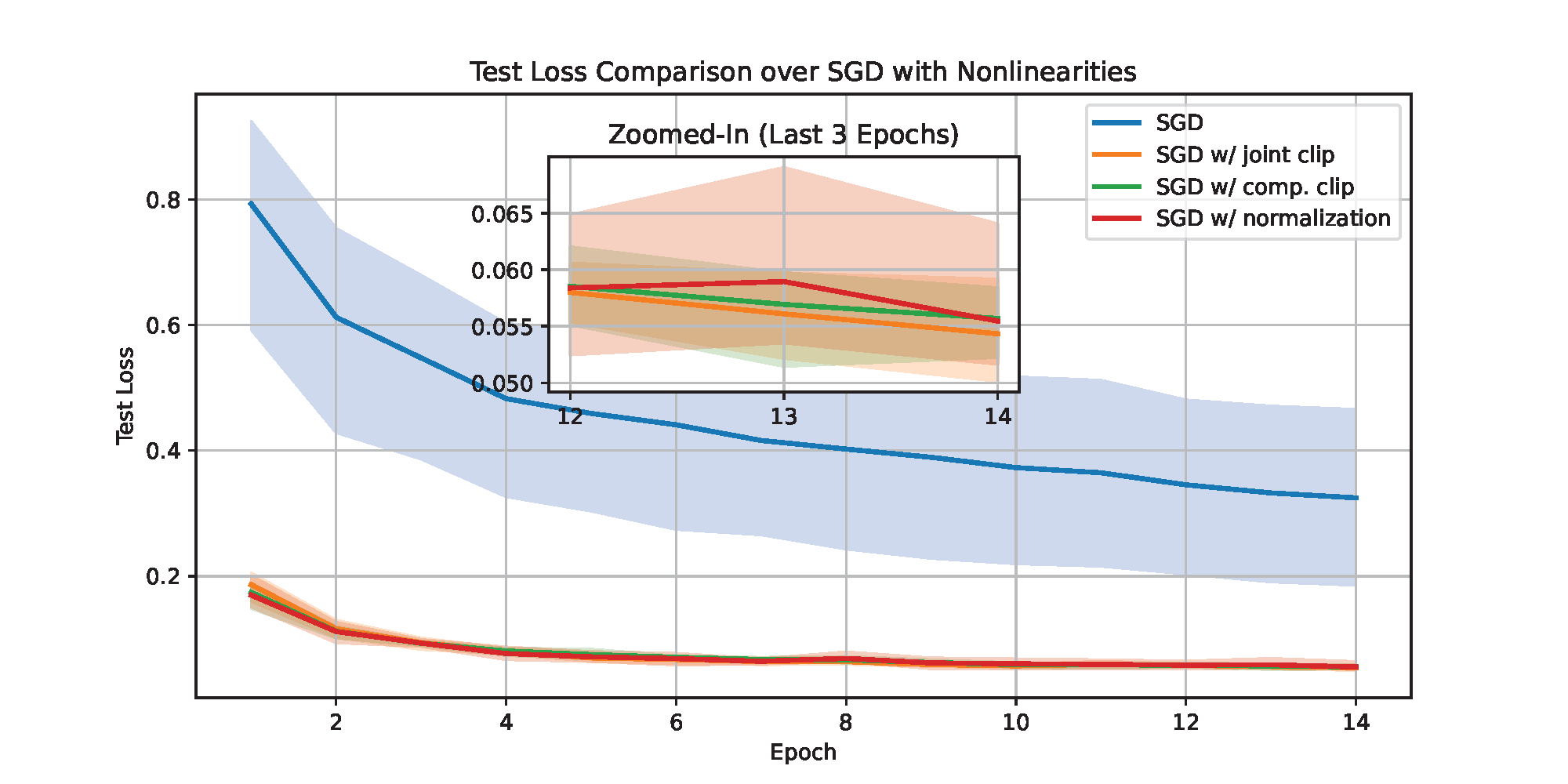}
    \end{subfigure}
    \caption{Comparisons of test accuracies and losses of SGD with different nonlinearities under Levy stable gradient noise.}
    \label{fig:perf_nlr}
\end{figure} 

\paragraph{Additional Experiments.} Here, we present the results for the same setup as used in Section \ref{sec:an-num} in the main body, for a wider range of step-sizes and tail probability thresholds. Figure \ref{fig:fig4} provides the MSE behaviour of sign, joint and component-wise clipping for step-sizes $\alpha_t = \frac{1}{(t+1)^\delta}$, with $\delta \in \{17/24,3/4,7/8\}$, while Figure \ref{fig:5} presents the tail probability for all three methods, with step-size $\delta = 3/4$ and using thresholds $\varepsilon \in \{0.05,0.1,0.5,5 \}$ . We can see that the results from Section \ref{sec:an-num} are consistent for different ranges of step-sizes, confirming that joint clipping is not always the optimal choice of nonlinearity. Moreover, we can see that all three methods achieve exponential tail decay, with joint clipping requiring a larger threshold, as it converges slower than the other two nonlinear methods, reaching a lower accuracy in the allocated number of iterations.

\begin{figure*}[!ht]
\centering
\begin{tabular}{lll}
\includegraphics[scale=0.33]{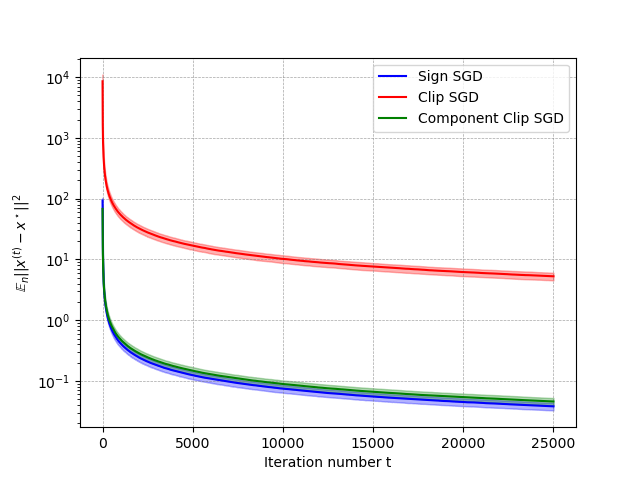}
&
\includegraphics[scale=0.33]{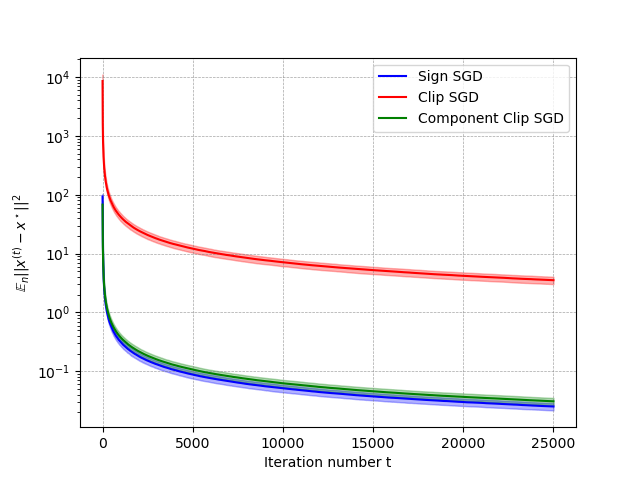}
&
\includegraphics[scale=0.33]{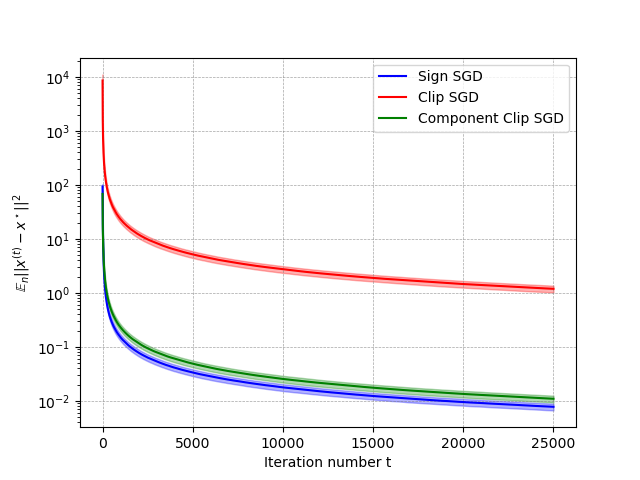}
\end{tabular}
\caption{MSE performance of nonlinear SGD methods, using step-size policy $\alpha_t = 1/(t+1)^\delta$, for different values of $\delta \in (2/3,1)$. Left to right: we choose the values $\delta \in \{17/24,3/4,7/8\}$, respectively. We can see that both component-wise nonlinearities converge faster in the MSE sense, independent of the step-size choice.}
\label{fig:fig4}
\end{figure*}

\begin{figure*}[htbp]
    \centering
    \begin{subfigure}[b]{0.45\textwidth}
        \centering
        \includegraphics[width=\textwidth]{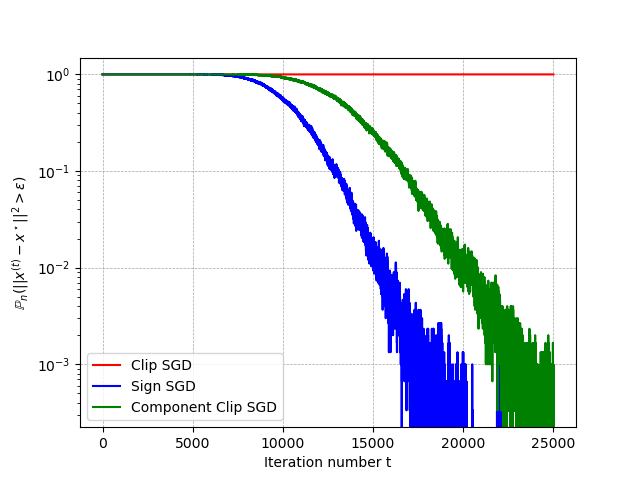} 
    \end{subfigure}
    \hfill
    \begin{subfigure}[b]{0.45\textwidth}
        \centering
        \includegraphics[width=\textwidth]{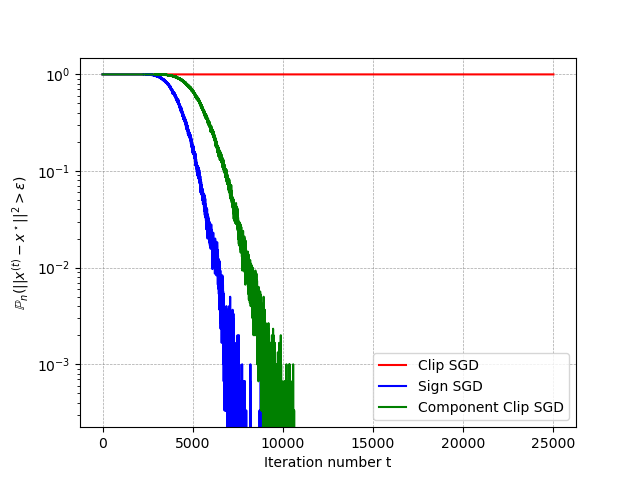} 
    \end{subfigure}
    \vskip\baselineskip
    \begin{subfigure}[b]{0.45\textwidth}
        \centering
        \includegraphics[width=\textwidth]{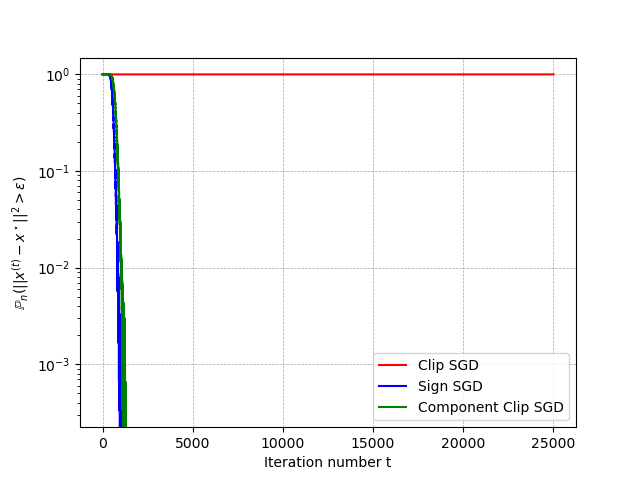} 
    \end{subfigure}
    \hfill
    \begin{subfigure}[b]{0.45\textwidth}
        \centering
        \includegraphics[width=\textwidth]{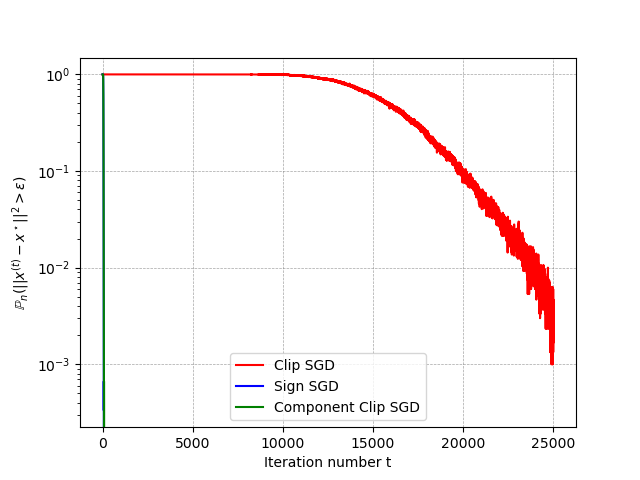}
    \end{subfigure}
    \caption{High-probability performance of nonlinear SGD methods, using step-size policy $\alpha_t = 1/(t+1)^\delta$, with $\delta = 3/4$. We use the thresholds $\varepsilon \in \{0.05,0.1,0.5,5 \}$ to compute the tail probability, left to right and top to bottom. We can see that all the methods exhibit exponential tail decay, with joint clipping needing the largest threshold to achieve exponential decay, due to slower convergence.}
    \label{fig:5}
\end{figure*}

\section{On the Noise Assumptions}\label{app:noise}

In this section we provide detailed discussions on the noise assumption used in our paper. In particular, we provide a detailed comparison with the bounded $p$-th moment assumption and discuss relaxations of the independent, identically distributed condition.

\paragraph{Comparison with Assumption \eqref{eq:bounded-moment}.} As discussed in Remark \ref{rmk:noise-comparison}, while the noise assumption in our work and in works assuming \eqref{eq:bounded-moment} are different, it is important to note that neither set of assumptions is uniformly weaker and both come with some advantages and disadvantages, as we detail next. To begin with, both set of assumptions are concerned with heavy-tailed noises, with ours requiring no moment bounds, while assumption \eqref{eq:bounded-moment} requires bounded moments of order $p \in (1,2]$, uniformly for all $x \in \mathbb{R}^d$. This is a significant relaxation on our end and allows for considering extremely heavy-tailed noises, such as Cauchy noise, for which even the mean does not exist! On the other hand, in order to guarantee exact convergence, our work requires noise with symmetric PDF, positive around zero, whereas no such requirements are needed for \eqref{eq:bounded-moment}. However, we relax the symmetry requirement, allowing for mixtures of symmetric and non-symmetric components, resulting in potentially biased noise, for which convergence to a neighbourhood of stationarity is shown (which is in general the best possible guarantee for biased SGD without corrective mechanisms like momentum or error-feedback). Contrary to this, \eqref{eq:bounded-moment} always requires the noise to be unbiased. Finally, while we require the noise vectors to be independent and identically distributed, which is not the case with \eqref{eq:bounded-moment}, this condition can be relaxed to include noises which are not identically distributed and depend on the current state (which we detail in the next paragraph), making the two sets of assumptions comparable on this point. Therefore, we can clearly see that both sets of noise assumptions come with advantages and disadvantages, with neither uniformly stronger than the other.

\paragraph{On the Independent, Identically Distributed Condition.} As discussed in Remark \ref{rmk:iid}, the independent, identically distributed condition can be significantly relaxed. First, the noise vectors need not be identically distributed. Instead, it suffices that in each iteration $t=1,2,\ldots$, the noise vector $\bzt$ has a probability density function (PDF) $P_t$, where in addition to being symmetric, we make the following requirement: there exists a $B_0 > 0$, such that $\inf_{t = 1,2,\ldots}P_t(\bz) > 0$, for each $\|\bz\| \leq B_0$. This condition can be seen as a uniform positivity in a neighbourhood of zero requirement, which is a mild condition on the behaviour of the sequence of PDFs and is satisfied, e.g., if the PDFs are drawn from a finite family $\mathcal{P}$ of symmetric PDFs, positive in a neighbourhood of zero (assuming a finite family is natural, as for our finite-time bounds, a weaker condition actually suffices, namely $\min_{t \in [T]}P_t(\bz) > 0$, for all $\|\bz\| \leq B_0$, which exactly corresponds to considering a finite family $\mathcal{P}$ of symmetric distributions, positive in a neighbourhood of zero, with $|\mathcal{P}| = T$, for any finite time horizon $T$). Therefore, defining $\phi^\prime(0) = \min_{i \in [d]}\inf_{t = 1,2,\ldots}\phi_{i,t}^\prime(0) > 0$, where $\phi_{i,t}(x_i) = \mathbb{E}_{z_i \sim P_t} \big[\mathcal{N}_1(x_i + z_i)\big]$ is the marginal expectation of the $i$-th noise component at time $t$, and $p_0 = \inf_{t = 1,2,\ldots}P_t(\mathbf{0}) > 0$, our current analysis applies and our proofs readily go through. Second, for joint nonlinearities, the noise vectors need not be independent. Instead, in each iteration $t$, the noise vector $\bzt$ is allowed to depend on the history through current state $\bxt$. This is facilitated by assuming that, for each fixed $\bx \in \mathbb{R}^d$, the noise vector $\bz = \bz(\bx)$ has a PDF $P_{\bx}(\bz) = P(\bz\vert X = \bx)$, which is symmetric for each fixed $\bx \in \mathbb{R}^d$, and that there exists a $B_0 > 0$, such that $\inf_{\bx \in \mathbb{R}^d}P_\bx(\bz) > 0$, for all $\| \bz \| \leq B_0$. The uniform positivity around zero for the conditional PDF $P_{\bx}(\bz)$ is again a generalization of the positivity around zero condition, and similar to the previous discussion, can be relaxed to a path-wise condition for our finite-time high-probability guarantees, namely, $\inf_{t \in [T]}P_{\bxt}(\bz) > 0$, for all $\|\bz\| \leq B_0$ and each fixed $T$. It can be shown, using the same steps of our proof, while replacing $P(\bz)$ with $P_{\bx}(\bz)$, that Lemma 3.2 holds for joint nonlinearities (recall the proof of Lemma S3.2, with $p_0 = P(\mathbf{0})$ now replaced by $p_0 = \inf_{\bx \in \mathbb{R}^d}P_{\bx}(\mathbf{0})$). Similarly, the proofs of Theorems 1-3, which use the conditional moment-generating function, conditioned on the entire history of the algorithm, readily go through, requiring no further modification.

\section{On the Metric}\label{app:metric}

As discussed in Remark \ref{rmk:metric}, it is possible to provide high-probability convergence guarantees of the same order as in Theorem \ref{thm:non-conv}, for the metric $\frac{1}{t}\sum_{k = 1}^t\min\{\|\nabla f(\bxk)\|,\|\nabla f(\bxk)\|^2\}$. To do so, we proceed as follows. Recall equation \eqref{eq:6} in the proof of Theorem \ref{thm:non-conv} in Section \ref{app:proofs}, namely that, for any $\beta \in (0,1)$, with probability at least $1 - \beta$, we have $G_t \leq \log(\nicefrac{1}{\beta}) + N_t$, where $G_t \triangleq \sum_{k = 1}^t\alpha_k\min\{\eta_1\|\nabla f(\bxk)\|,\eta_2\|\nabla f(\bxk)\|^2\}$ and $N_t \triangleq f(\bx^{(1)}) - f^\star + LC^2\left(\nicefrac{1}{2} + 8LD_{\mathcal{X}}\right)\sum_{k = 1}^{t}\alpha_k^2 + 8C^4L^2\sum_{k = 1}^{t}\alpha_k^2A_{k}^2$. Instead of dividing both sides of the inequality by $A_t = \sum_{k = 1}^t\alpha_k$, as was originally done in \eqref{eq:6}, we divide both sides of the inequality by $t$ and notice that the sequence of step-sizes is decreasing, to get, with probability at least $1 - \beta$
\begin{equation*}
    \frac{\alpha_t}{t}\sum_{k = 1}^t\min\{\eta_1\|\nabla f(\bxk)\|,\eta_2\|\nabla f(\bxk)\|^2\} \leq \frac{\log(\nicefrac{1}{\beta})+N_t}{t}.
\end{equation*} Dividing both sides of the above inequality by $\eta\alpha_t$, where $\eta = \min\{\eta_1,\eta_2\}$ and recalling that $\alpha_t = \frac{a}{(t+1)^\delta}$, we get
\begin{equation*}
    \frac{1}{t}\sum_{k = 1}^t\min\{\|\nabla f(\bxk)\|,\|\nabla f(\bxk)\|^2\} \leq \frac{2^\delta(\log(\nicefrac{1}{\beta}) + N_t)}{a\eta t^{1-\delta}},    
\end{equation*} with probability at least $1 - \beta$. Considering the different choices of step-size parameter $\delta \in (2/3,1)$, we can obtain the same convergence rates as in Theorem \ref{thm:non-conv}. The same trick can be used to show convergence guarantees of the exact Polyak-Ruppert average $\widetilde{\bx}^{(t)} \triangleq \frac{1}{t}\sum_{k = 1}^t\bxk$ in Corollary \ref{cor:cvx}. 

As discussed, the metric $\frac{1}{t}\sum_{k = 1}^t\min\{\|\nabla f(\bxk)\|,\|\nabla f(\bxk)\|^2\}$ is a more general quantity than $\min_{k \in [t]}\|\nabla f(\bxk)\|^2$, in the sense that in our proof of Theorem \ref{thm:non-conv}, we used the bounds on the metric $\frac{1}{t}\sum_{k = 1}^t\min\{\|\nabla f(\bxk)\|,\|\nabla f(\bxk)\|^2\}$ to show that they imply the same rates on the more standard metric $\min_{k \in [t]}\|\nabla f(\bxk)\|^2$.\footnote{Technically, we use the bounds on the metric $\sum_{k = 1}^t\frac{\alpha_k}{\sum_{s = 1}^t\alpha_s}\min\{\eta_1\|\nabla f(\bxk)\|,\eta_2\|\nabla f(\bxk)\|^2\}$, however, as we showed above, we can easily switch to the metric $\frac{1}{t}\sum_{k = 1}^t\min\{\|\nabla f(\bxk)\|,\|\nabla f(\bxk)\|^2\}$.} The metric considered in our work, $\frac{1}{t}\sum_{k = 1}^t\min\{\|\nabla f(\bxk)\|,\|\nabla f(\bxk)\|^2\}$ is directly comparable to the metric used in \cite{nguyen2023improved}, namely $\frac{1}{t}\sum_{k = 1}^t\|\nabla f(\bxk)\|^2$. Moreover, the two metrics are asymptotically equivalent, in the sense that, for some $t_0 \in \N$ sufficiently large, we have, for all $k \geq t_0$, $\min\{\|\nabla f(\bxk)\|,\|\nabla f(\bxk)\|^2\} = \|\nabla f(\bxk)\|^2 $, as the gradient norm converges to zero with high-probability, according to Theorem \ref{thm:non-conv}. The expression $\min\{\|\nabla f(\bxk)\|,\|\nabla f(\bxk)\|^2\}$ stems from our general, black-box analysis in Lemma \ref{lm:huber}, and was also previously used in works studying clipping, e.g., \cite{zhang2020adaptive,chen2020understanding}. We used the more standard metric $\min_{k \in [t]}\|\nabla f(\bxk)\|^2$, to simplify the exposition in Theorem \ref{thm:non-conv}.    

Finally, the reason why \cite{nguyen2023improved} are able to provide bounds on the quantity $\frac{1}{t}\sum_{k = 1}^t\|\nabla f(\bxk)\|^2$ stems from the fact that a large clipping threshold is used in their analysis, proportional to $t^{1/(3p-2)}$, allowing the authors to show that the norms of gradients of the sequence of iterates, i.e., $\|\nabla f(\bxk)\|$, for all $k = 1,\ldots,t$, are guaranteed to stay below the clipping threshold with high probability, i.e., that no clipping will be performed with high probability, in effect behaving like SGD with no clipping. As observed in a recent work \cite{hubler2024gradient}, this is contrary to how clipping is used in practice, where clipping is typically deployed with a small, constant threshold, see \cite{hubler2024gradient} and references therein. On the other hand, our general black-box analysis provides convergence guarantees of (joint) clipped SGD for any constant value of the clipping threshold, bridging the existing gap between theory and practice.

\end{document}